\documentclass[11pt]{article}
\usepackage{enumerate}
\usepackage[OT1]{fontenc}
\usepackage{amsmath,amssymb}
\usepackage{natbib}
\usepackage[usenames]{color}
\usepackage{dsfont}
\usepackage{smile}
\usepackage{multirow}
\usepackage{rotating,soul,tcolorbox}

\usepackage[colorlinks,
            linkcolor=red,
            anchorcolor=blue,
            citecolor=blue
            ]{hyperref}
\usepackage{algorithm}
\usepackage{algorithmic}

\def\tr{\mathop{\text{tr}}\kern.2ex}

\def\P{{\mathbb P}}
\def\E{{\mathbb E}}

\def\supp{\mathop{\text{supp}}}

\long\def\comment#1{}

\def\tr{\mathop{\text{Tr}}}

\newcommand{\bel}{\begin{eqnarray}\label}
\newcommand{\eel}{\end{eqnarray}}
\newcommand{\bes}{\begin{eqnarray*}}
\newcommand{\ees}{\end{eqnarray*}}

\def\real{{\mathbb{R}}}

\newcommand{\red}{\color{red}}

\let\emptyset\varnothing
\let\hat\widehat
\let\tilde\widetilde

\def\P{{\mathbb P}}
\def\E{{\mathbb E}}

\def\supp{\mathop{\text{supp}\kern.2ex}}
\def\argmin{\mathop{\text{\rm arg\,min}}}

\def\tr{{\rm{Tr}}}

\def\supp{\mathop{\text{supp}}}

\def\tr{\mathrm{Tr}}

\def\var{\mathop{\text{var}}}
\def\I{{\mathbf{I}}}

\theoremstyle{plain}

\newtheorem{prop}{Proposition}[section]
\newtheorem{assump}{Assumption}[section]

\usepackage{mathrsfs}

\usepackage{fullpage}

\def\sep{\;\big|\;}

\usepackage{hyperref}
\def\##1\#{\begin{align}#1\end{align}}
\def\$#1\${\begin{align*}#1\end{align*}}
\usepackage{enumitem}

\usepackage{undertilde}
\theoremstyle{plain}

\theoremstyle{mytheoremstyle}

\usepackage{setspace}
\newcommand{\transpose}{\scriptscriptstyle \sf T}
\def\trans{^{\transpose}}

\definecolor{darkred}{RGB}{150,50,50}
\definecolor{brown}{RGB}{250,100,100}
\definecolor{green}{RGB}{000,150,100}
\definecolor{purple}{RGB}{250,000,180}

\def\red{\color{darkred}}

\def\supini{^{\scriptscriptstyle \sf ini}}
\def\bSigmahat{\widehat{\bSigma}}

\def\Ghat{\widehat{G}}

\def\PMIbb{\mathbb{PMI}}
\def\PMIbbhat{\widehat{\PMIbb}}
\def\PMIhat{\widehat{\PMI}}
\def\PMI{\mbox{PMI}}
\def\SPPMIbb{\mathbb{SPPMI}}
\def\SPPMIbbhat{\widehat{\SPPMIbb}}
\def\SPPMIhat{\widehat{\SPPMI}}
\def\SPPMI{\mbox{SPPMI}}
\def\Abhat{\widehat{\Ab}}

\def\Qbhat{\widehat{\Qb}}

\def\Obhat{\widehat{\Ob}}

\begin{document}

\title{\LARGE Knowledge Graph Embedding with Electronic Health Records Data via Latent Graphical Block Model}

\author{ 
Junwei Lu\thanks{Department of Biostatistics, Harvard T.H. Chan School of Public Health; Email:
\texttt{junweilu@hsph.harvard.edu}},~
Ying Jin\thanks{Department of Statistics, Stanford University; Email:
\texttt{ying531@stanford.edu}},~and~
Tianxi Cai\thanks{Department of Biostatistics, Harvard T.H. Chan School of Public Health; Email:
\texttt{tcai@hsph.harvard.edu}}.
}

\date{}
\maketitle

\begin{abstract}

Due to the increasing adoption of electronic health records (EHR), large scale EHRs have become another rich data source for translational clinical research. The richness of the EHR data enables us to derive unbiased knowledge maps and large scale semantic embedding vectors for EHR features, which are valuable and shareable resources for translational research. Despite its potential, deriving generalizable  knowledge from EHR data remains challenging. First, EHR data are generated as part of clinical care with data elements too detailed and fragmented for research. Despite recent progress in mapping EHR data to common ontology with hierarchical structures, much development is still needed to enable automatic grouping of local EHR codes to meaningful clinical concepts at a large scale. Second, the total number of unique EHR features is large, imposing methodological challenges to derive reproducible knowledge graph, especially when interest lies in conditional dependency structure.  Third, the detailed EHR data on a very large patient cohort imposes additional computational challenge to deriving a knowledge network. To overcome these challenges, we propose to infer the conditional dependency structure among EHR features via a latent graphical block model (LGBM). The LGBM has a two layer structure with the first providing semantic embedding vector (SEV) representation for the EHR features and the second overlaying a graphical block model on the latent SEVs. The block structures on the graphical model also allows us to cluster synonymous features in EHR. We propose to learn the LGBM  efficiently, in both statistical and computational sense, based on the empirical point mutual information matrix. We establish the statistical rates of the proposed estimators and show the perfect recovery of the block structure. Numerical results from simulation studies and real EHR data analyses suggest that the proposed LGBM estimator performs well in finite sample. 
\end{abstract}

\def\Ebb{\mathbb{E}}
\def\Vbb{\mathbb{V}}
\def\ZZ{\mathbb{Z}}
\def\Vbvec{\overrightarrow{\Vb}}
\def\simiid{\stackrel{\scriptscriptstyle \sf iid}{\sim}}

\section{Introduction}

The increasing adoption of electronic health records (EHR) has led to EHR data becoming a rich data source for translational clinical research.  Detailed longitudinal clinical information on broad patient populations  have allowed researchers to derive comprehensive prediction models for translational and precision medicine research \citep[e.g.]{lipton2015learning,choi2016doctor,choi2016retain,rajkomar2018scalable,ma2017dipole}. The longitudinal co-occurrence patterns of EHR features observed on a large number of patients have also enabled researchers to derive comprehensive knowledge graphs and large scale representation learning of semantic embedding vectors to represent EHR features \citep{che2015deep,choi2016learning,choi2016multi,miotto2016deep}, which are valuable and shareable resources for translational clinical research.

There is a rich literature on machine learning approaches to deriving knowledge graphs  \citep{che2015deep,choi2016learning}. 
Some traditional approaches to learning knowledge graph, such as \cite{nickel2011three,bordes2013translating,wang2014knowledge,lin2015learning} , only consider the information on the connectivity and relationship between the nodes. The symbolic nature of the graph, while useful, makes it challenging to manipulation. To overcome these challenges, several useful knowledge graph embedding methods,  which aim to embed entities and relations
into continuous vector spaces, has found much success in recent years 
\cite[e.g.]{du2018dynamic,nickel2011three,nguyen2017novel,shang2019end,bordes2013translating}. 
Most of these existing methods require training on patient level longitudinal data, which would be both computationally prohibitive and subject to data sharing constraints due to patient privacy. 

In addition, these methods do not address another challenge arising from EHR data elements being overly detailed with many {\em near-synonymous} features that need to be grouped into a broader concept for better clinical interpretation \citep{schulam2015clustering}. For example, among COVID-19 hospitalized patients, 4 distinct LOINC (Logical Observation Identifiers Names and Codes) codes are commonly used for C-reactive protein (CRP) at Mass General Brigham (MGB), and we {\em manually} grouped them to represent CRP to track the progression of COVID patients \cite{brat2020international}. 
Valuable hierarchical ontologies, such as the ``PheCodes" from the phenome-wide association studies (PheWAS) catalog for grouping of International Classification of Diseases (ICD) codes, the clinical classification software (CCS) for grouping of the Current Procedural Terminology (CPT) codes, RxNorm (Prescription for Electronic Drug Information Exchange) hierarchy for medication, and LOINC  hierarchy for laboratory tests, have been curated via tremendous manual efforts \citep{steindel2002introduction,healthcare2017clinical,wu2019mapping}. 
However, these mappings are incomplete and only applicable to codes that have been mapped to common ontology.  

These challenges motivate us to propose a latent graphical block modeling (LGBM) approach to knowledge graph embedding with EHR data. The LGBM has a two-layer structure linked by the latent semantic embedding vectors (SEVs) that represent the EHR features. The co-occurrence patterns of the EHR features are generated from the first layer of a hidden markov model parametrized by the latent SEVs similar to those proposed in \cite{arora2016latent}. The conditional dependency structure of the SEVs is encoded by the second layer of a vector-valued block graphical model. Specifically, let $\Vb_w \in \RR^p$ be the SEV of feature $w\ \in \cV$, where $\cV$ is the corpus of all features. The vector-valued graphical model associates a feature network to the vectors by assuming that there is an edge between the codes $j$ and $k$ if  $\Vb_j$ and $\Vb_{k}$ are jointly independent conditioning on all the other vectors 
$\{\Vb_w\}_{w\neq j, k}$. The proposed LGBM has several key advantages over existing methods. First, we learn the LGBM from a co-occurrence matrix of EHR features, which only involve simple summary statistics that can be computed at scale, overcoming both computational and privacy constraints. Second, the LGBM characterizes the conditional dependence structure of the EHR features, not  marginal relationships. Third, the learned block structure also enables automatic grouping of near synonymous codes, which improves both interpretability and reproducibility. 

There is a rich statistical literature on vector-valued graphical models with high dimensional features \citep[e.g.]{Yuan2007Model,rothman2008,Friedman2008SparseIC,dAspremont2008First,fan2009,lam2009,Yu2010High,Cai2011AC,Liu2017,PMID:25632164,kolar2014graph,du2019multivariate}. The block structured graphical model has also been extensively studied by, e.g.,  \cite{bunea2020model,eisenach2020high,eisenach2019efficient}. For example, \cite{bunea2020model} proposed the $G$-block model which assumes that the weighted matrix of the graph is block-wise constant and they estimate the network by the convex relaxation optimization. \cite{eisenach2020high} considered the inference of the block-wise constant graphical model.
On the other hand, these graphical modeling methods largely require observations on the nodes of the graph and/or the random vectors on the nodes and hence are not applicable to the EHR setting where only co-occurrence patterns of the codes are observed. The latent structure of the LGBM makes it substantially more challenging to analyze the theoretical properties of the estimated network. Although a few estimators have also been proposed for latent graphical models \citep{choi2011learning,bernardo2003variational,chandrasekaran2010latent,wu2017graphical,bunea2020model,eisenach2020high,eisenach2019efficient}, they require subject-level data and are computationally intensive. In this paper, we propose a two-step approach to efficiently learn the LGBM to infer about the conditional dependence and grouping structures of EHR features based on the summary level co-occurrence matrix. We first conduct spectral decomposition on an empirical pointwise mutual information (PMI) matrix derived from the co-occurrence matrix to learn the SEVs for the features. In the second step, we obtain a block graphical model estimator on the representation vectors learned from the first step. We establish statistical rate of the learned PMI matrix and knowledge network. The remainder of the paper is organized as follows.  We detail the LGBM assumptions in Section \ref{sec:model} and present statistical methods for estimating the SEVs along with the block graphical model in Section \ref{sec:est}.  Theoretical properties of our proposed estimators are discussed in Section \ref{sec:theory} including the estimation procesion and clustering recovery. Section \ref{sec:numerical} implements our method to both synthetic simulations and a real EHR data to learn a knowledge graph based on a large number of codified EHR features. Although our methods are generally applicable to both codified and narrative features, we will use describe our methods below in the context of EHR codes  to be more concrete. 


\def\brmI{\textbf{I}}
\def\Ssc{\mathcal{S}}
\def\Asc{\mathcal{A}}
\def\Zsc{\mathcal{Z}}
\def\supth{^{\rm th}}
\def\Vsc{\mathcal{V}}

\section{Latent Graphical Block Model}\label{sec:model}


{In this section, we detail a two-layer generative process for the proposed LGBM with the two layers linked by the latent $p$-dimensional SEVs for the $d$ EHR codes, denoted by $\Vbb_{d\times p}=[\bV_{1}, ...,\bV_{d}]\trans$, where $\bV_j$ is the SEV for the $j$th code and without loss of generality, we index the $d$ codes as $\Vsc = [d] \equiv \{1, \ldots, d\}$. The first layer of the LGBM is a hidden Markov model for the code sequence given the SEVs and the second layer is a latent gaussian graphical block model that encodes the joint distribution of $\{\bV_j, j\in [d]\}$.}

\def\Vsc{\mathcal{V}}
\def\supth{^{\scriptscriptstyle \rm th}}

\def\th{^{\scriptscriptstyle \rm th}}

\subsection{Overview of the Hidden Markov Model}

We first give a high level picture of how we assume the longitudinal EHR data is generated from a hidden markov model. Consider an observed length $T$ longitudinal sequence of EHR features occurred in a patient, $w_1, w_2, \ldots w_T \in \cV$. We assume that the $t\th$ code, $w_t$, is generated from a hidden Markov model \citep{arora2016latent} and it takes the value of $w \in \cV=[d]$ with the probability
\[
  \PP(w_t = j|\bc_t) \propto{e^{\langle \Vb_j, \bc_t \rangle}},
\]
where $\Vb_j = (V_{j1}, ..., V_{jp})\trans \in \RR^p$ is the latent SEV of code $w$ with $\Vbb_{\cdot\ell}= (V_{1\ell}, ..., V_{d\ell})\trans$ following a Gaussian prior
$\Vbb_{\cdot\ell} \simiid N(0, \bSigma)$, and  $\{\bc_t\}_{t\geq 1}$ is the the discourse vector which follows a hidden markov process 
\begin{equation} \label{eq:hidden_Markov_def}
\begin{split} 
    &\bz_{2} = \sqrt{\alpha}\cdot \bz_1/\|\bz_1\|_2+ \sqrt{1 - \alpha} \cdot r_{2}; \\
    &\bz_{t+1} = \sqrt{\alpha}\cdot \bz_t + \sqrt{1 - \alpha} \cdot r_{t+1}, \\
    &\mbox{$\bc_t = \bz_{t+1}/\|\bz_{t+1}\|_2$ when $t \geq 2$},
\end{split}
\end{equation}
where $\alpha = 1 -  {\log d}/{k^2}$, $r_t \mbox{ i.i.d.} \sim N(0, \I_k/k)$ for all $t \geq 2$, and $r_{t+1}$ is independent from $\bz_1, \cdots, \bz_t$. 
 

From the hidden Markov model in \eqref{eq:hidden_Markov_def}, we generate the code sequence $w_1, \ldots, w_T \in  \cV$.
We will show that under the model in \eqref{eq:hidden_Markov_def}, the dependency structure of the latent SEVs can be connected with the code sequence via the point-wise mutual information (PMI) matrix. Specifically, we will show that under the hidden markov model with Gaussian assumption on the distribution of $\Vbb_{\cdot \ell}$, the PMI matrix is close to the covariance matrix of $\Vbb_{\cdot\ell}$:
\begin{equation}\label{eq:pmi-consist}
       \|\PMIbb - \bSigma\|_{\max} = O(\sqrt{\log d/p}),
\end{equation}
where $\PMIbb$ is the population PMI matrix defined based on the so-called co-occurrence matrix of the longitudinal word occurrence data $\{w_t\}_{t \ge 1}$. Define the context of a feature $w_t$ as the codes given within $q$-days of time $t$, denoted as $c_q(w_t) = \{w_{s}~:~|s-t|\le q \}$.  Given a pair of codes  $w,w' \in \cV$, we define the co-occurrence between $w$ and $w'$ similar to \cite{beam2020clinical} as the number of incidences the feature $w'$ occurs in the context of $w$, denoted by 
\begin{equation}\label{eq:coocur}
   \cC(w,w') = |\{(t,s):|t-s|\leq q \text{ and } w = w_t,w' = w_s   ~|~ t = 1, \ldots, T \}|.
\end{equation}
 We denote the co-occurrence matrix as $\CC = [\cC(j,j')] \in \RR^{d\times d}$.
 The PMI between codes $w, w'$ is defined as
\begin{equation}\label{eq:def_stat_pmi}
\PMIbb(w,w') = \log \frac{\mathcal{N}\cdot \mathcal{N}(w,w')}{\mathcal{N}(w,\cdot)\mathcal{N}(w',\cdot)},
\end{equation}
where $\mathcal{N}(w,w')$ is the expectation of $\mathcal{C}(w,w')$, and the marginal occurrence
\[
\mathcal{N}(w,\cdot)=\sum_{c\in V}\mathcal{N}(w,c), \mathcal{N}=\sum_{w\in V}\sum_{w'\in \Vsc}\mathcal{N}(w,w').
\]

\begin{remark}
 \cite{arora2016latent} shows a similar relationship as \eqref{eq:pmi-consist}. However, one of the major difference between the results in \cite{arora2016latent} and our model is that they assumed that all code vectors are independent from each other, while our representation vectors are generated from the above block graphical model which enable us to characterize the dependency structures between the words. Moreover, they did not characterize the specific rate on distance between the PMI and the covariance matrix.
\end{remark}

\subsection{Graphical Block Model for the Latent Embedding Vectors}\label{sub:Gaussian_graph}

In this section, we will describe the latent vector-valued block graphical model that captures the conditional dependency structure of $\{\bV_i, i \in [d]\}$ as well as the unknown block structure that encodes information on which codes are near-synonymous. Specifically, we assume that
\begin{equation} \label{eqn:cond_Gaussian}
    \Vb_{i} |\Vbb_{-i,\cdot} \sim N\left(\sum_{j \in [d]\backslash \{ i\} } B_{ij}\Vb_{\jmath},\sigma^2 \Ib_p\right) \mbox{   for all $i \in [d]$ and some $\sigma >0$},
\end{equation}
Let  $\Vbvec = (\Vbb_{\cdot 1}\trans, ..., \Vbb_{\cdot p}\trans)\trans$ denote the vector concatenating columns of $\Vbb$. We can show (see Appendix~\ref{sub:proof_v} for details) that \eqref{eqn:cond_Gaussian} can be satisfied 
if $\Vbvec$ follows a multivariate normal distribution
\begin{equation} \label{eqn:joint_Gaussian}
   \Vbvec \sim  N(0,\Ib_p \otimes \bSigma) \quad \mbox{with}\quad 
   \bSigma = \bOmega^{-1} ,
\end{equation}
where $\bOmega = \sigma^{-2}(\Ib_d - \Bb)$, where the identity matrix comes from the variance in \eqref{eqn:cond_Gaussian}, and $\Bb = [B_{ij}]$ is same as \eqref{eqn:cond_Gaussian}. Under \eqref{eqn:joint_Gaussian}, we show in Lemma \ref{lem:distribution_v} that the conditional distribution of $\Vb_{i} |\Vbb_{-i, \cdot}$ follows  (\ref{eqn:cond_Gaussian}). Therefore, \eqref{eqn:joint_Gaussian} defines a vector-valued gaussian graphical model such that the node of the graph for code $i$ is represented by the corresponding SEV $\Vb_{i}$ and the code $i$ is connected to $j$ if and only if $B_{ij} \neq 0$. We assume that $\Bb$ is sparse to enable recovery of the code dependency structure. From \eqref{eqn:cond_Gaussian} and  \eqref{eqn:joint_Gaussian}, we can see that $\Vb_{i}$ is conditionally independent to $\Vb_{j}$ conditioning on all $\{\Vb_k\}_{k \neq i,j}$ if and only if  $B_{ij} \neq 0$. Thus the dependency structure is encoded by the support of $\Bb$ which is assumed to be sparse. 



\def\Rbb{\mathbb{R}}

In order to characterize the synonymous structure, we further impose the block structure Gaussian graphical model of the SEVs such that
\begin{equation}\label{eq:word_cluster_assign}
\Vbb_{\cdot\ell} = \Ab_{d\times K} \Zb_{\ell} + \Eb_{\ell}, \quad \text{ with } \Zb_{\ell} \simiid N(0, \Qb), \quad \Eb_{\ell}\simiid N(0, \bGamma), \quad \text{for $\ell = 1, ..., p$,}
\end{equation}
where the covariance matrix $\Qb_{K\times K}$ is positive definite and  $\bGamma_{d \times d} = \diag(\bGamma_{11}, \ldots, \bGamma_{dd})$ with $\Gamma_j >0$ for $j\in [d]$. 
The clustering assignment matrix $\Ab=[\Ab_{ik}]=[I(i \in G_k)]$ essentially encodes the partition of $[d]$ as disjoint sets $G=\{G_1, \ldots, G_K\}$ with $[d] = \cup_{k=1}^K G_k$. The codes within each group $G_k$ are considered as near synonymous and can be grouped into a clinically meaningful code concept. With the additional block structure, we may express the covariance structure as
\begin{equation}\label{eq:cluster_cov_decomp}
\bSigma = \bOmega^{-1} =    \sigma^2(\Ib_d - \Bb)^{-1} = \Ab \Qb \Ab^{\transpose} + \bGamma .
\end{equation}
Under this model, embedding vectors for codes within a cluster share similar behavior. 
To see this, let $\ZZ = (\Zb_{1}, ..., \Zb_p)$ and $\Ebb = (\bE_1, ..., \bE_p)$. Then $\Vbb = \Ab\ZZ + \Ebb$ and hence $\{\Vb_i = \ZZ_{k\cdot} + \Ebb_{i\cdot}, i \in G_k\}$ share the same center $\ZZ_{k\cdot}$.  Moreover, according to the decomposition in \eqref{eq:cluster_cov_decomp}, we have 
{ $\bSigma_{ij} = \Qb_{g_ig_j}$}, meaning that the noiseless part of large covariance matrix $\bSigma$ is entrywise identical within groups. {We aim to leverage the structure of $\bSigma$ to improve the estimation of $\Bb$ and infer the dependency structure of the codes.} 


The block model in \eqref{eq:cluster_cov_decomp} is not identifiable due to the additivity of $ \Ab \Qb \Ab^{\transpose}$ and $\bGamma$. For the identifiability of the model, we introduce the following definition of feasible covariance.
 \begin{definition}
Define the cluster gap of any positive definite matrix $\bSigma$ induced by clustering $G$ as 
$$
\Delta(\bSigma,G) = \min_{g_i\neq g_j} \max_{\ell \neq i,j} |\bSigma_{i\ell}-\bSigma_{j \ell}|.
$$
A partition $G$ along with its associated assignment matrix $\Ab$ and decomposition $\bSigma=\Ab \Qb \Ab^{\transpose} + \bGamma$ is {feasible} if $\Delta(\bSigma,G)>0$ and  $\bGamma$ is diagonal with $\Gamma_{ii}=0$ for all $i$ such that $|G_{g_i}|=1$.
\end{definition}
The decomposition of $\bSigma$ by $G$ implies that items belonging to the same cluster behave in the same way. In a {\em feasible} partition and decomposition, $\Delta(\bSigma,G)>0$ means that the matrix $\bSigma$ can be accurately separated by $G$ and code in the same cluster will not be separately into distinct groups. Setting $\Gamma_{ii} =0$ for $|G_{g_i}|=1$ removes noise from singleton codes to ensure identifiability and hence allows us to remove the requirement of minimum cluster size being at least 2 needed in \cite{bunea2020model}.  Formally, we have the following result on identifiability of our model, saying that two {\em feasible} decompositions must be identical. We leave the proof of the proposition to Appendix \ref{sec:prop-pf}.

\begin{prop}[Identifiability of block model for code clusters]\label{prop:block_model_identification} Define the assignment $g_i=k$ iff $i\in G_k$.
Let the true decomposition be $\bSigma=\Ab \Qb \Ab^{\transpose} + \bGamma$ with true partition $G$, so that $G,\Ab$ and the decomposition is {\em feasible}. Then for any partition $\tilde{G}$, its assignment matrix $\tilde{\Ab}$ with decomposition $\bSigma = \tilde{\Ab}\tilde{\Qb} \tilde{\Ab}^{\transpose} + \tilde{\bGamma}$ that are also {\em feasible}, it holds that $\tilde{G}=G$, $\tilde{\Qb}=\Qb$, $\tilde{\bGamma}=\bGamma$.
\end{prop}

\section{Estimation of the Latent Code Graphical Model}
\label{sec:est}

To estimate the precision matrix for the graphical model, we first obtain estimators for $\bSigma$ and then use the CLIME estimator of \cite{cai2011constrained} to estimate the precision matrix $\bOmega$. We propose a three step estimator for $\bSigma$ and $\Ab$. In step (I), we obtain an initial estimator for $\bSigma$, $\bSigmahat\supini$, based on the co-occurrence matrix. In step (II), we perform clustering of the codes based on $\bSigmahat\supini$ to obtain $\Abhat$ and the resulting $\Ghat=\{\Ghat_1, ..., \Ghat_K\}$. Finally in step (III), we update the estimate of $\bSigma$ as $\bSigmahat$ by leveraging the estimated group structure $\Ghat$. 

\paragraph{Step I: } If $\{\Vb_i, i = 1, ..., d\}$ were observed, $\bSigma$ can be estimated empirically using the empirical covariances of $\Vb_i$ and $\Vb_j$. However, since $\Vb_i$ is latent, we instead estimate $\bSigma$ directly from the co-occurrence matrix $\CC_{d\times d} = [C_{jj'}]=[\cC(j,j')]$ with $\cC(j,j')$ calculated as (\ref{eq:coocur}) across all patients. From $\CC$, we derive the shifted and truncated empirical PMI matrix as an estimator for $\bSigma$. Specifically, we define the PMI matrix estimator as $\PMIbbhat_{d \times d} = [\PMIhat(j,j')]$ with
\begin{equation}\label{eqcal:def_emprical_pmi}
\PMIhat(j,j') = \log \frac{\cC(j,j')}{\cC(j, \cdot)\cC(j', \cdot)}
\end{equation}
and the shifted positive PMI matrix estimator as 
\begin{equation}
\SPPMIbbhat= [\SPPMIhat(j,j')] \quad \mbox{with}\quad
  \SPPMIhat(j,j') = \max \left\{\widehat{\mbox{PMI}}(j,j'), \eta \right\} ,
\end{equation}
where $\eta > -\infty$ is a threshold used in practice to prevent the values of  $\SPPMIhat$ being minus infinity. In our theoretical analysis, we will show that $\widehat{\mbox{PMI}}(j,j')$ is lower bounded by some constant with high probability under appropriate assumptions, so $\max \big(\widehat{\mbox{PMI}}(j,j'), \eta \big) $ would be closer to truth than $\widehat{\mbox{PMI}}(j,j')$ with high probability if $\eta$ is chosen properly. Our initial estimator for $\bSigma$ is then set as $\bSigmahat\supini = \SPPMIbbhat$. 


\paragraph{Step II} With $\bSigmahat\supini$, we estimate the code cluster $G = \{G_1, \ldots, G_K\}$ as $\Ghat = \{\Ghat_1, \ldots, \Ghat_K\}$ based on Algorithm \ref{al:cluster} proposed by \cite{bunea2020model} with distance between two rows of $\bSigmahat\supini$ corresponding to codes $j$ and $j'$ defined as 
\[
d(j,j') =  \max_{c \neq j, j'}\Big|\SPPMIhat(j,c) - \SPPMIhat(j',c)\Big|.
\]
\begin{algorithm}[htb]
\caption{
The COD Algorithm \citep{bunea2020model}}\label{al:cluster}
\begin{algorithmic}
\STATE Input: $\SPPMIhat$ and $\alpha >0$
\STATE Let the candidate nodes $V = [d]$ and $\ell = 0$
\WHILE{$V \neq \emptyset$}
  \STATE $\ell = \ell +1$
  \STATE If $|V| = 1$, then $\hat G_{\ell}  = V$
  \STATE If $|V| > 1$, then
\begin{enumerate}
\item $(j_{\ell}, j'_{\ell}) = \argmin_{j\neq j'} d(j,j')$
\item If $d(j_{\ell}, j'_{\ell}) > \alpha$, then $\hat G_{\ell} = \{j_{\ell}\}$\\
Else $\hat G_{\ell} = \{c \in V ~|~ d(j_{\ell}, c) \wedge d(j'_{\ell}, c) \le \alpha\}$
\end{enumerate}
\STATE Update $V  = V \backslash \hat G_{\ell}$
\ENDWHILE
\STATE Output: the cluster estimator $\hat G = \{\hat G_1, \ldots, \hat G_K\}$
\end{algorithmic}
\end{algorithm}

\paragraph{Step III} In the final step, we refine the estimator for $\bSigma$ by averaging the entries of $\bSigmahat\supini$ belonging to the same cluster since $\bSigma_{ii'} = \Qb_{kk'}$ and $\bSigma_{ii} = \Qb_{kk} + \bGamma_{ii}$ for all $i \in G_k$ and $i' \in G_{k'}\setminus \{i\}$. Thus, for $p,k' = 1,\ldots, K$, we estimate $\Qb_{kk'}$ as
$$
\arraycolsep=1.4pt\def\arraystretch{1.3}
  \Qbhat_{kk'} = \left\{
  \begin{array}{cl}
  \{|\widehat{G}_k|\cdot|\widehat{G}_{k'}|\}^{-1}\sum_{w \in \widehat{G}_k, w' \in \widehat{G}_{k'} } \SPPMIhat(w,w'), & \quad \mbox{if }k \ne k'; \\
 \{|\widehat{G}_k|\cdot (|\widehat{G}_k|-1)\}^{-1}\sum_{w, w' \in\Ghat_k} I(w \ne w')\SPPMIhat(w,w'), & \quad \mbox{if }  k = k', |\Ghat_k| > 1; \\
 \SPPMIhat(k,k), &\quad \mbox{if } k = k', |\Ghat_k| = 1.
 \end{array} \right.
$$

Finally, we estimate the $k\supth$ column (corresponding to cluster $G_k$) of the precision matrix 
{$\Ob = \Qb^{-1}$} based on $\Qbhat = [\Qbhat_{kk'}]$ for $k = 1, \ldots, K$ via the CLIME estimator \cite{cai2011constrained}:
\begin{equation}\label{eq:CLIME-1}
\Obhat_k =  \argmin_{\bbeta\in \RR^{K}} 
\|\bbeta\|_1 \qquad \text{subject to}\qquad 
\|\Qbhat \bbeta - \eb_k\|_{\infty} \le \lambda,
\end{equation}
and $\Obhat = (\Obhat_1, \ldots, \Obhat_K)$.
The latent code graph can thus be estimated from the support of $\Obhat$. By the model assumption in Eq\eqref{eq:cluster_cov_decomp}, we can estimate $\bGamma$ and $\bOmega$ by
\begin{equation}\label{eq:omg-est}
\hat \bGamma = \diag(\bSigmahat\supini - \hat \Qb) \text{ and } \hat \bOmega = \hat \bGamma^{-1} - \hat \bGamma^{-2}(\Obhat + \hat \Ab^{\transpose}  \hat \bGamma^{-1}  \hat \Ab   )^{-1}.
\end{equation}


\section{Theoretical Properties}\label{sec:theory}

In this section, we establish the estimation rates of the PMI matrix and the precision matrix estimators. We will also establish the consistency of the clustering recovery. The high level summary of our theoretical analysis has three major components: (1) the statistical error between $\SPPMIbbhat$ and the true PMI matrix; (2) the approximation error between the true $\mbox{PMI}$ and the feature SEV inner products $\Vb \Vb\trans/p$; and (3) the statistical error of precision matrix estimator and the clustering recovery.  We then integrate the three parts together to get the final rates. We first detail the key assumptions required by our theoretical analyses.

\begin{assump}\label{assump:parameter_space}
The true precision matrix $\Ob$ belongs to the parameter space 
\begin{equation}\label{eq:UM}
\begin{aligned}
 \cU_s(M,\rho) &= \Big\{\Ob \in R^{K \times K} \,\big|\, \Ob \succ 1/\rho, \|\Ob\|_{2} \le \rho,  \max_{j \in [d]} \|\Ob_{j\cdot}\|_{0} \le s, \|\Ob \|_1 \le M \Big\} 
 \mbox{.}
\end{aligned}
\end{equation}
for  some constants $\rho$ and $M$. 
\end{assump}

\begin{assump}\label{assump:block_bound}
Assume that the latent model follows the block structure so that the true decomposition $\bSigma=\Ab \Qb \Ab^{\transpose} + \bGamma$ is legal, and $\max_i \bGamma_{ii}\leq c_0$ for some constant $c_0>0$. Moreover, and the distance between clusters satisfies that $\Delta(\Qb)=\Delta(\bSigma,G)=\epsilon >0$, with $\sqrt{\log d/p} = o(\epsilon)$.
\end{assump}


\begin{assump}\label{assump:kdT}
Assume that $\log d=o(\sqrt{p})$, $s\log d/p = o(1)$, and $p(\log^2 d) \cdot \max_k |G_k| = O(d^{1-\gamma})$ for some $\gamma\in (0,1/2)$. Also assume that the corpus size $T$ satisfies $T=\Omega(p^5 d^4\log^4 d)$.
\end{assump}

\begin{remark}
The matrix class $\cU_s(M,\rho)$ is frequently considered in the literature on inverse covariance matrix estimation \citep{cai2016}. 
Assumption \ref{assump:parameter_space} of $\Ob \in \cU_s(M,\rho)$  bounds the variances of the code SEVs from both sides, thus ensuring the norm of the SEVs to concentrate around $O(k)$. The bounds on norms of $\Ob$  also prevents the  SEVs in different clusters from being too correlated. Assumption \ref{assump:block_bound} ensures that different code clusters have sufficiently large distance.
\end{remark}

\begin{remark}
Assumption \ref{assump:kdT} requires that the code SEVs are relatively compact with dimension $p = O(\sqrt{d})$ and the number of codes in the same cluster to be controlled by $\sqrt{d}/p$. We also require the corpus size $T$ to be appropriately large (polynomial in $p,d$) so that code occurrences are able to reveal desired properties of the underlying feature SEVs, with diminishing random deviations.
\end{remark}

\subsection{Statistical Rate of PMI Estimators}

The following proposition provides the approximation error rate of $\SPPMIbbhat$ for the population PMI matrix $\PMIbb$.  
The result is a consequence of the concentration of empirical code occurrences to their expectations conditional on $\{\bc_t\}$ and $\Vbb$, while the latter further concentrate to true PMI due to the good mixing properties of $\{\bc_t\}$. Detailed proof of Proposition~\ref{prop:hat_PMI_to_PMI} is given in Appendix~\ref{sec:conv_hat_PMI}. 

\begin{proposition}
Suppose Assumptions~\ref{assump:parameter_space}, \ref{assump:block_bound} and \ref{assump:kdT} hold. For a fixed window size $q\geq 2$,  
\[
\|\SPPMIbbhat- \PMIbb\|_{\max}  \leq {5}/{\sqrt{p}}  + dq/T,
\]
with probability at least $1-1/d$ and appropriately large $p$.  


\label{prop:hat_PMI_to_PMI}
\end{proposition}

Next we establish the rate of the approximation error between $\PMIbb$ and the population covariance matrix $\bSigma = \bOmega^{-1}$ in Proposition \ref{prop:stat_pmi_to_cov} with proof given in Appendix~\ref{sec:stat_PMI_to_cov_pf}.

\begin{proposition}\label{prop:pmi}
Under Assumptions~\ref{assump:parameter_space}, \ref{assump:block_bound} and \ref{assump:kdT}, with probability at least $1-3d^{-\tau}$, for sufficiently large $(k,d,T)$ and constant $C_0 = 5+(7\sqrt{2q}+48\sqrt{3(\tau+4)})\|\bSigma\|_{\max}$, we have
\begin{equation}
\|\PMIbb
- \bm{\Sigma}\|_{\max}\leq C_0\sqrt{\log d /p}.
\label{eq:stat_pmi_to_cov}
\end{equation}

\label{prop:stat_pmi_to_cov}
\end{proposition}

\begin{remark}
 \cite{arora2016latent} only established the consistency of the PMI matrix without showing the concrete statistical rate. On the contrary, we establish the exact statistical rates of the PMI estimator with respect to $d, k$ and $T$ which involves finer analysis on the hidden Markov process. Moreover, the analysis of  \cite{arora2016latent} assumes that the prior of the word vectors are independent thus they can apply concentration inequalities to the code vectors. However, in our latent graphical model, the word vectors $\Vbb$ are dependent and thus their proof can no longer be applied.  We introduced the log-Sobolev inequality and the whitening trick to show the the rate of PMI under the  the Gaussian graphical model prior.
 \end{remark}
 
\subsection{Clustering Recovery and Precision Matrix Estimation Consistency}

\def\Obhat{\hat \Ob}
We next summarize the results on the clustering recovery and the estimation consistency of {\red $\Obhat$}  for the precision matrix $\Ob = \Qb^{-1}$. First, we have the following theorem on the perfect recovery of  $\Ghat=\{\Ghat_1, ..., \Ghat_K\}$ for $G=\{G_1, ..., G_K\}$ with proof given in Appendix~\ref{sec:appendix_cluster_recovery}.

{
Recall the latent model $\Vb_{\ell \cdot} = \Ab \Zb_{\ell} + \Eb_{\ell}$ and the decomposition $\bSigma = \Ab \Qb \Ab\trans + \bGamma$. 
In order to differentiate between different groups in the LGBM, we define the distance between clusters as 
\[
  \Delta(\Qb) :=  \min_{i\neq j} \max_{k\neq i,j} \big| \bm{Q}_{ik}-\bm{Q}_{jk}\big|.
\] 
} 
Such quantity is also defined in the study of block matrix estimation \citep{bunea2020model,eisenach2020high,eisenach2019efficient}.

\begin{theorem}\label{thm:recover_cluster}
Under Assumptions~\ref{assump:parameter_space}, \ref{assump:block_bound} and \ref{assump:kdT}, with any  $\alpha \in (0, \epsilon/2]$ in Algorithm \ref{al:cluster}, if $\Delta(\Qb) \ge C_0 \sqrt{\log d/p}$ where $C_0$ is defined in Proposition \ref{prop:stat_pmi_to_cov}, $\widehat{G}=G$ with probability at least $1-\exp(-\omega(\log^2 d))-O(d^{-\tau})$ for some large constant $\tau>0$.
\end{theorem}

With the perfect recovery of clusters, we have the following corollary whose proof is given in Appendix \ref{sec:appendix_precision_after_recovery}.
\begin{corollary}\label{cor:hat_PMI}
Under assumptions of Theorem \ref{thm:recover_cluster}, with probability no less than $1 - \exp(-\omega(\log^2 d))-O(d^{-\tau})$ for some large constant $\tau>0$, we have
$$\|\Qbhat - \Qb \|_{\max} \leq O\Big(\sqrt{\frac{\log d}{p}}\Big).$$
\end{corollary}

In the following, we present the main theorem on the convergence rate of our precision estimator with proof given in Appendix \ref{sec:appendix_precision_after_recovery} as well.

\begin{theorem}\label{thm:main}
Under the settings of Theorem \ref{thm:recover_cluster}, for sufficiently large $p,d,T$, with probability no less than $1 - \exp(-\omega(\log^2 d))-O(d^{-\tau})$ for some large constant $\tau>0$, we have
\[
  \|\widehat \Ob - \Ob\|_{\max} \leq \lambda \|\Ob\|_{1} + O\Big(\sqrt{\frac{\log d}{p}} \Big), 
\|\widehat{\Ob}-\Ob\|_{1} \leq s\lambda \|\Ob\|_{1} + O\Big(s\sqrt{\frac{\log d}{p}} \Big), \]
and $\text{supp}(\widehat \Ob) = \text{supp}(\Ob)$.
If we choose $\lambda = C\log d/p$ for some sufficient large constant $C$, we also have with probability no less than $1 - \exp(-\omega(\log^2 d))-O(d^{-\tau})$ for some large constant $\tau>0$,
$
\|\widehat{\bOmega}-\bOmega\|_{1} = O\big(s\sqrt{{\log d}/{p}} \big).
$
\end{theorem}

\def\bOmegahat{\widehat{\bOmega}}
\section{Numerical Experiments for Synthetic and Real Datasets}
\label{sec:numerical}
In this subsection, we conduct simulation studies to evaluate the performance of our proposed algorithm and compare it with the GloVe method \citep{pennington2014glove}. We consider the settings  $(p,d)=(50,25)$, $(100,50), (500,1000)$ and $(1000,2000)$. 
We generate the precision matrix  $\Ob=\Qb^{-1}$ via the following two types of graphs. 
\begin{itemize}
\item \noindent{\it Independent Graph -- }
The basic model is where all nodes are independent with same variance. We set $\Ob = c\Ib_K$ for some $c$.
\item \noindent{\it Erdős–Rényi Graph -- }
This model generates a graph with the adjacency matrix $A_{ij}\stackrel{\text{i.i.d.}}{\sim} \mbox{Bern}(p)$, $i<j$ for some $p\in(0,1)$, i.e., all edges are independently added with probability $p$. In an Erdős–Rényi Graph with $K$ nodes, the expected amount of total edges is $K(K-1)p/2$. To satisfy the sparsity conditions, we typically choose small $p$, specified later on. After generating the adjacency matrix, we let $\Ob= c\Ab + (|\lambda_{\min}(c\Ab)|+c_1)\Ib $ for some $(c,c_1)$.
\end{itemize}
Under these two types of Graphs, we consider a total of six set of hyper parameters to generate $\Ob$ as detailed in Table \ref{fig:graph}.
\begin{table}[h]
\centering
\begin{tabular}{l |l }
\toprule
Name & Setting  \\
\hline
G1 & Independent Graph with $c=0.5$. \\
\hline
G2 & Independent Graph with $c=2$. \\
\hline
G3 & Erdős–Rényi Graph with $p=0.2$, $(c,c_1)=(0.3,0.2)$.\\
\hline
G4 & Erdős–Rényi Graph with $p=0.2$, $(c,c_1)=(0.5,0.3)$.\\
\hline
G5 & Erdős–Rényi Graph with $p=0.05$, $(c,c_1)=(0.3,0.2)$.\\
\hline
G6 & Erdős–Rényi Graph with $p=0.05$, $(c,c_1)=(0.5,0.3)$.\\
\bottomrule
\end{tabular}
\caption{The scenarios of the graphs generating the precision matrix.  \label{fig:graph}}
\end{table}
Finally, we construct $\Ab$ by evenly assigning the group with $G_k = \{(k-1)m+1,\dots,km\}$ for $k=1,\dots,K$ and generate the diagonal entries of $\bGamma$  from $\mbox{Unif}[0.25,0.5]$ to form the final $\bSigma$, where $m = d/K$ and we choose $K = 10, 25, 50$. 
After generating the underlying graph, we generate $\Zb$ and code vectors $\Vb$, and simulate $T$ corpus of the word sequence with discourse process  $\{\bc_t\}$ specified in \eqref{eq:hidden_Markov_def}. PMI matrix is then calculated, with window size $q=10$. The threshold $\alpha$ in Algorithm \ref{al:cluster} for estimating $\Ghat$ is set as $\alpha=c\cdot\sqrt{\log d/k}$, where $c$ is tuned over grid $\{0.1,0.2,\dots,2\}$ to have the most stable cluster assignment (quantified by Rand index specified later). { For $\Obhat$ estimated via CLIME-type estimator as in Eq. \eqref{eq:CLIME-1}, we suppose the clusters are recovered perfectly and use the true partition $G$ to calculate $\widetilde{\mbox{SPPMI}}$.} 
 We set the tuning parameter in \eqref{eq:CLIME-1} as {\ $\lambda=c\cdot \sqrt{\log d/k}$, where $c$ is chosen over grid $\{0.1,0.2,\dots,2\}$ to be most stable, with the smallest entry-wise change from the previous one.} 


\def\perr{\%\mbox{Err}}
\def\RI{\mbox{RI}}
Evaluations on the estimation of the knowledge network focused on three components: (1) accuracy in cluster recovery with cluster partition $\Ghat$; (2) average error in precision matrix estimation; and (3) support recover accuracy. We evaluate the performance of cluster recovery by Rand index averaged over all the repetitions. Specifically, for true partition $G$ and estimator $\hat G$ for nodes $\{1,\dots,d\}$, let $g(j)$ be the cluster of node $j$ in $G$ and $\hat{g}(j)$ in $\hat{G}$, then Rand index is calculated by 
\begin{equation*}
\RI = \frac{\big| \{i\neq j:g(i)=g(j),\hat{g}(i)=\hat{g}(j)\}\big| +\big|\{i\neq j:g(i)\neq g(j),\hat{g}(i)\neq \hat{g}(j)\}\big| }{d(d-1)/2}.
\end{equation*}
We evaluate the performance of our precision matrix estimator via the average relative error,
$\perr = \textrm{average}(\|\hat \Ob - \Ob\|/\|\Ob\|)$, where we let the matrix norm $\|\cdot\|$ be either $\|\cdot\|_{\max}$ or $\|\cdot\|_{F}$. We evaluate the support recovery of the true graph via F-score, which is defined based on the true positives (TP), false positives (FP) and false  negatives (FN) as
\[
\rm Precision = 
\frac{TP}{
TP + FP}
, Recall = \frac{TP}{
TP + FN}
, \text{F-score} = \frac{2 \cdot Precision \cdot Recall}{
Precision + Recall} ,
\]

Table \ref{fig:simu_cluster_recovery} summarizes the clustering accuracy of the estimator. The clustering accuracy is generally high with $\RI$ above 90\% across all settings. With a fixed $d$, the accuracy tends to be slightly higher with larger $K$ which corresponds to smaller number of codes per code concept group. We do not observe a big difference between different configurations for $\bO$. 
 Tables \ref{fig:synth1} and \ref{fig:synth2} give the $\perr$ for the estimation of $\Ob$ with $\Obhat$ and { the F-score of the support recovery for the graphical model estimator without and with clustering.} 
 In summary, the proposed procedure can identify most signals in the graph. We can also see that utilizing the block structure of the precision matrix helps the estimation of the precision matrix as long as inferring the graph structure.

\begin{table}[h]
\centering
{\footnotesize %
\renewcommand\arraystretch{1.2}
\begin{tabular}{c|c |c |c c c c c c c c c c}
\toprule
$d$ & $p$ & $K$ & G1 & G2 & G3 & G4 & G5 & G6 \\
\hline
\multirow{2}{*}{$25$} & \multirow{2}{*}{$50$} 
  & 25 & 100\%& 99.89\% &  100\%  &100\%  & 100\% & 100\%\\
 && 10 &  95.56\%& 96.10\% & 95.00\% &95.32\%  & 95.08\% & 95.18\%\\
\hline
\multirow{3}{*}{$50$} & \multirow{3}{*}{$100$} 
  & 50 &  100\%& 100\% & 100\% &100\%  & 100\% & 100\%\\
&& 25 &  97.95\%& 98.32\%  & 98.02\% & 98.26\%  & 98.02\% & 98.17\%\\
&& 10 & 92.24\%& 94.32\% & 91.93\% & 92.13\% & 91.88\%  & 91.90\%\\ 
\hline
\multirow{2}{*}{$500$} & \multirow{2}{*}{$1000$} 
& 50 &   98.52 \%& 97.45 \%    & 98.01 \% &  97.16\%  & 97.42 \% & 98.14 \%\\
&& 25 &  96.19  \%& 97.63 \%  & 98.05 \% & 96.25 \%  & 96.26 \% & 95.69 \%\\
&& 10 & 90.91 \%& 92.17 \%  & 93.62 \% & 92.42 \% & 92.17 \%  & 94.52  \%\\ 
\hline
\multirow{2}{*}{$1000$} & \multirow{2}{*}{$2000$} 
& 50 &   98.10 \%& 97.72 \%  & 98.62 \% &  98.10\%  & 98.44 \% & 98.10 \%\\
&& 25 &   96.10 \%& 95.42 \% & 97.41 \% &  96.10\%  & 98.21 \% & 98.16 \%\\

&& 10 &  90.18  \%&  91.53\%   &  91.56\% & 92.74 \%  & 91.39 \% & 92.83  \%\\
\bottomrule
\end{tabular}
}%
\caption{Averaged Rand index for cluster recovery. 
 \label{fig:simu_cluster_recovery}}
\end{table}

\begin{table}[htbp]
\centering
{\footnotesize %
\renewcommand\arraystretch{1.1}
\begin{tabular}{c c |c c c c c c c}
\toprule

& & $K$ & G1 & G2 &G3 & G4 & G5 & G6 \\
\cline{3-9}
\multirow{9}{*}{\begin{sideways}{\it $d=500$, $k=1000$}\end{sideways}} & \multirow{3}{*}{$\|\cdot\|_{\max}$}
  & 50 & 5.32\% & 16.09\% & 34.55\% & 31.99\%  & 34.31\% & 35.02\%\\
  && 25 & 3.34\%& 13.01\%  & 31.99\% &29.98\%  & 33.27\% & 31.23\%\\
& & 10 &  5.68\% & 7.12\%  & 25.88\% & 26.67 \%  & 35.24\% & 20.86\%\\
\cline{2-9}
 & \multirow{3}{*}{ $\|\cdot\|_F$ }
  & 50 & 10.66\% & 12.27\% & 49.08\% & 42.73\% & 37.76\% & 39.81\% \\
 & & 25 &  10.93\%& 10.54\% & 41.36\% & 39.42\% & 33.81\% & 32.51\%  \\
& & 10 &  9.38\% & 8.84\% & 31.14\% &  34.08\%  & 27.89 \% & 22.91\%\\
\cline{2-9}
& \multirow{3}{*}{F-score}
& 50 & 73.77 \% & 75.84 \% &  66.83\% & 66.85 \% & 68.89 \% & 69.17 \%  \\
 & & 25 &  73.69 \%& 70.82\% & 66.62\% & 61.39\% & 63.48\% & 63.45\%  \\
& & 10 &  75.68\% & 79.84\% & 70.16\%  &  69.76\%  & 67.42 \% & 62.53\% \\
\hline \hline
& & $K$ & G1 & G2 &G3 & G4 & G5 & G6 \\
\cline{3-9}
\multirow{9}{*}{\begin{sideways}{\it $d=1000$, $k=2000$}\end{sideways}} & \multirow{3}{*}{$\|\cdot\|_{\max}$}
  & 50 & 15.20 \%& 24.66\% & 34.12\% & 31.27\% & 25.22\% & 28.02\%  \\
& & 25 & 6.46 \%& 8.45\% & 22.37\%& 21.60\% & 26.63\% & 26.60\%  \\
& & 10 & 8.42 \% & 8.23\% & 23.58\%  & 22.07\%& 21.35\% & 18.95\% \\
\cmidrule{2-9}
 &\multirow{3}{*}{ $\|\cdot\|_F$ }
  & 50 &  15.00 \%& 23.69\% & 39.57\%& 38.40\% & 28.22\% & 28.32\%  \\
& & 25 &  3.20 \%& 9.26\% & 31.55\%& 30.77\% & 22.56\% & 23.41\%  \\
& & 10 &  5.52 \%& 8.78\% & 20.99\% & 28.74\% & 16.39\% & 17.72\%\\
\cmidrule{2-9}
 &\multirow{3}{*}{F-score}
  & 50 & 71.21 \%& 70.56\% & 66.99\% & 68.09\% & 70.40\% & 62.44\% \\
& & 25 & 76.46 \%& 68.46\% & 64.36\% & 64.65\% & 68.08\% & 62.78\%  \\
& & 10 & 74.43 \%& 71.23\% &71.07\% & 64.20\% & 64.07\% & 57.66\% \\
\bottomrule
\end{tabular}
}%
\caption{Averaged empirical relative error of precision matrix $\Ob$ without using the clustering method in Algorithm \ref{al:cluster}.  \label{fig:synth1}}
\end{table}

\begin{table}[htbp]
\centering
{\footnotesize %
\renewcommand\arraystretch{1.1}
\begin{tabular}{c c |c c c c c c c}
\toprule
& & $K$ & G1 & G2 &G3 & G4 & G5 & G6  \\
\cline{3-9}
\multirow{9}{*}{\begin{sideways}{\it $d=500$, $k=1000$}\end{sideways}} & \multirow{3}{*}{$\|\cdot\|_{\max}$}
& 50  &   3.47\%& 9.76 \%  & 31.83 \% &  28.61\%  & 23.33 \% & 27.73 \%\\
&  &  25 &  2.89  \%&  7.35 \%&  19.16 \% &  21.00 \%  &  28.30 \% &  25.92  \%\\
& &  10 & 1.29  \% & 4.36 \% &  21.64 \%   \%  &  22.28 \% &  15.81 \%\\
\cline{2-9}
 & \multirow{3}{*}{ $\|\cdot\|_F$ }
 & 50  &  8.62 \%& 8.14 \% & 41.35 \% & 40.61 \%   & 28.58 \% & 25.27 \%\\
 & &  25 &  7.22  \%&  6.10 \%&  23.32 \% &  27.83 \%   &  18.55 \% &  18.52   \%\\
& &  10 &  6.83 \% & 5.62 \% & 24.14  \% & 16.88 \%  & 12.70  \% & 10.29 \%\\
\cline{2-9}
& \multirow{3}{*}{F-score}
& 50  &  73.47 \%&  76.48\%  & 72.89 \% & 71.17 \%  & 67.29 \% &  71.52\%\\
 & &  25 &  72.90 \%&  73.10 \% & 66.21  \% & 69.97 \%  &  66.73 \% &  67.55 \%\\
& &  10 &  71.29  \% & 72.11 \% & 69.36  \%  & 74.12\%  & 74.42  \% & 68.35 \%\\
\hline \hline
& & $K$ & G1 & G2 &G3 & G4 & G5 & G6 \\
\cline{3-9}
\multirow{9}{*}{\begin{sideways}{\it $d=1000$, $k=2000$}\end{sideways}} & \multirow{3}{*}{$\|\cdot\|_{\max}$}
& 50  &  12.15 \%& 21.56 \%  & 34.89 \% & 35.54 \%  & 26.20 \% & 27.12 \%\\
 & &  25 &  2.16 \%& 6.28 \%  & 24.91 \% & 27.61\%  & 28.29 \% & 25.01 \%\\
& &  10 &  1.82 \% & 5.74 \% & 20.11 \% & 14.79 \%  & 19.29 \% & 14.06 \%\\
\cline{2-9}
 & \multirow{3}{*}{ $\|\cdot\|_F$ }
 & 50  & 11.70  \%&  20.42\%  &  61.14\% &  57.99\%  & 40.15 \% & 48.64 \%\\
 & &  25 &  3.02 \%&  6.76\% & 38.72 \% & 46.03\%  &  24.07\% & 29.40 \%\\
& &  10 & 3.66  \% &  5.81\%  &  18.54\% &  19.26\%  & 11.10  \% & 13.66 \%\\
\cline{2-9}
& \multirow{3}{*}{F-score}
& 50  &  73.66 \%& 75.69 \% & 73.61 \% & 76.79 \%  &  75.96\% &  64.21 \%\\
 & &  25 & 77.16  \%&  71.45\% & 70.36 \% & 72.22 \%  & 74.25 \% & 71.81 \%\\
& &  10 &  76.82 \% & 73.75 \%  & 73.87 \% & 72.95 \%  & 71.42 \% & 65.03 \%\\
\bottomrule
\end{tabular}
}%
\caption{Averaged empirical relative error of precision matrix $\Ob$ using the clustering method in Algorithm \ref{al:cluster}.  \label{fig:synth2}}
\end{table}

\subsection{Applications to Electronic Health Record Data}
\label{sec:app}
In this section, we apply the proposed LGBM inference procedure to derive a knowledge network using codified EHR data of 2.5 million patients from a large tertiary hospital system. We analyzed four categories of codes including ICD, medication prescription, laboratory tests, and CPT procedures. We mapped ICD codes to PheCodes using the ICD-to-PheCode mapping from PheWAS catalog (https://phewascatalog.org/phecodes). The CPT procedure codes are mapped into medical procedure categories according to the clinical classifications software (CCS) (https://www.hcup-us.ahrq.gov/toolssoftware/ccs\_svcsproc/ccssvcproc.jsp). The medication codes are mapped to the ingredient level RxNorm codes, which is part of the  Unified Medical Language System (UMLS) \citep{bodenreider2004unified}. The laboratory codes are mapped to LOINC codes of the UMLS. We included a total of $d= 5507$ mapped codes that have at least 1000 occurrences and calculated the co-occurrence of these codes within 30 day window across all patients. We then applied our proposed procedures to obtain estimates of group structure $G$ and the precision matrix $\Ob$. 


We grouped the 5507 codes into $K= 8$ code clusters. 
Because the network is too large to illustrate, we focus on two specific codes of interest: rheumatoid arthritis and type-II diabetes. The code clouds of the selected neighbors of rheumatoid arthritis and depression are illustrated in Figure \ref{fig:wordcloud}. We also only focus on the clustering of {the lab codes LOINC}. We choose the tuning parameter $\alpha$ in Algorithm \ref{al:cluster} in the range from each code consists of its own cluster to all codes merge to one cluster.  From Figure \ref{fig:cluster}, we visualize the clustering result via the clustering tree and we can observe that similar codes are easier to be merged together.


\begin{figure}[htbp]%
    \centering
\begin{tabular}{cc}
Rheumatoid arthritis & Type 2 diabetes\\[-20pt]
\includegraphics[width=8cm]{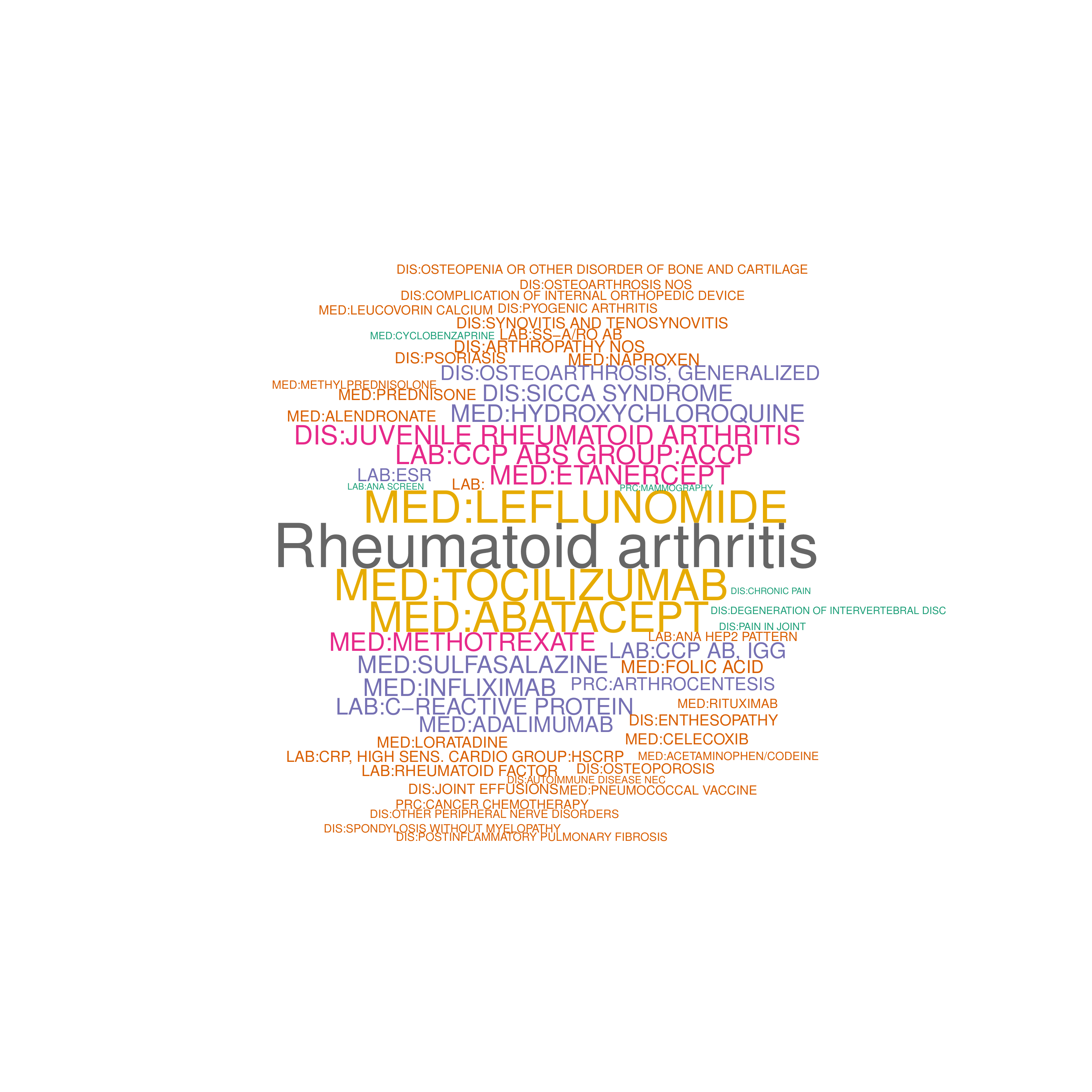} & \includegraphics[width=8cm]{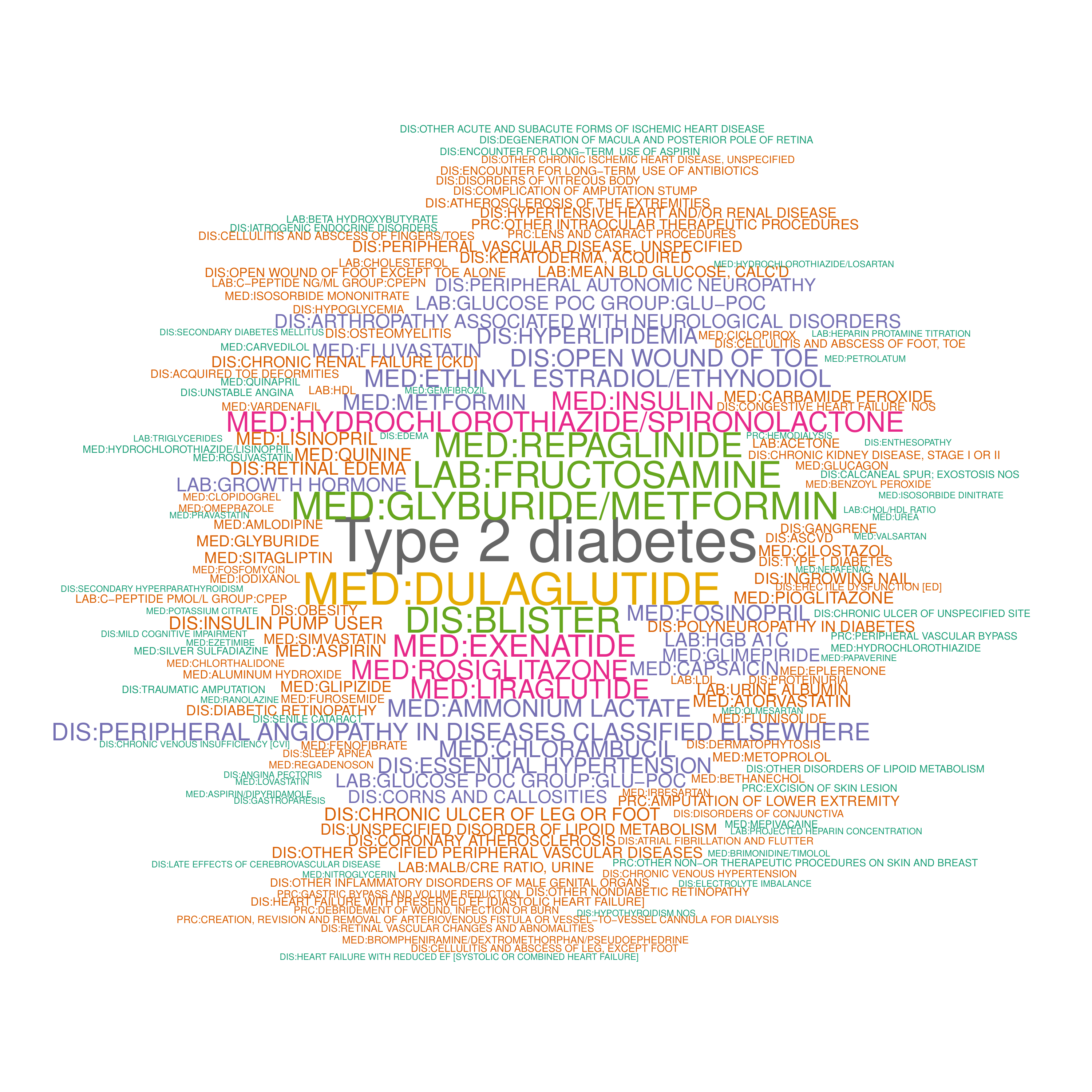}
\end{tabular}
\vskip-30pt
    \caption{Word cloud of selected features of rheumatoid arthritis and type 2 diabetes. Different colors represent different types of codes.}%
    \label{fig:wordcloud}%
\end{figure}

\begin{figure}[htbp]%
    \centering
\begin{tabular}{cc}
Rheumatoid arthritis & Type 2 diabetes\\
\includegraphics[width=7cm]{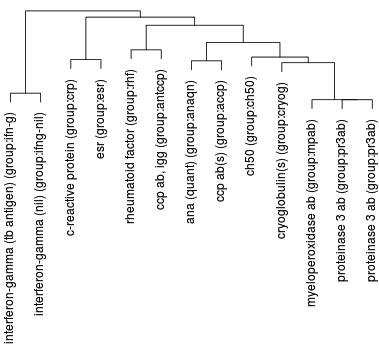} & \includegraphics[width=7cm]{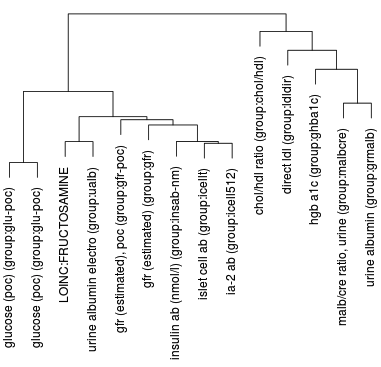}
\end{tabular}
    \caption{Clustering of laboratory codes for rheumatoid arthritis and type 2 diabetes}%
    \label{fig:cluster}%
\end{figure}

\bibliographystyle{ims}
\bibliography{ref}

\appendix


\section*{Appendix}

\section{Proofs on the Model Properties}

In this section, we prove the identifiability of the vector-valued graphical model.

\subsection{Vector-Valued Graphical Model on $\Vbb$} \label{sub:proof_v}

\begin{lemma} \label{lem:distribution_v}
Let $\Vb_j \in \RR^p$ be the code vector variable for code $j$ ($j \in [d]$) and $\Vb_{-i}$ be the set of all code vectors expect $\Vb_{i}$, i.e., $\Vb_{-i} = \{\Vb_j: j \in [d]\backslash \{i\}\}$. And let $[\Vb_j]_{\ell}$ be the $\ell$-th component of $\Vb_j$ ($\ell \in [k], j \in [d]$). 
$[\Vb]_\ell := \big( [\Vb_{1}]_\ell, [\Vb_{2}]_\ell, \cdots, [\Vb_{d}]_\ell \big)^{\transpose}\in \RR^{d}$ concatenates the $\ell$-th components of all code vectors. And $[\Vb]_{\ell}$ can be stacked in a column vector as $\Vb :=\big( [\Vb]_{1}^{\transpose}, [\Vb]_{2}^{\transpose}, \cdots, [\Vb]_{p}^{\transpose} \big)^{\transpose} \in \RR^{d\times p}$.

There exists a multivariate Gaussian distribution $N(0,\Sigma_{\Vb})$ for $\Vb$ such that 
$$    \Vb_{i} |\Vbb_{-i\cdot} = v_{-i} \sim N\left(\sum_{j \in [d]\backslash \{ i\} } B_{ij}v_j,\sigma^2 \Ib_p\right) \mbox{   for all $i \in [d]$},$$
where $\sigma^2 > 0$ is a constant.

And one such $N(0,\Sigma_{\Vb})$ is
$$ \Vb \sim N(0, \Sigma_{\Vb}) = N(0, \bOmega^{-1}) = N(0, (\Ib_p \otimes (\Ib_d - \Bb)/\sigma^2)^{-1}),$$
where $\Bb$ is a symmetric hollow matrix and $\Ib_d - \Bb$ is positive definite.

\end{lemma}

\begin{proof}[Proof of Lemma~\ref{lem:distribution_v}.]

The conditional Gaussian distribution assumption is
\begin{equation} \label{eqn:cond_Gaussian1}
    \Vb_{i} |\Vbb_{-i\cdot} = v_{-i} \sim N\left(\sum_{j = 1 }^d B_{ij}v_j,\sigma^2 \Ib_p\right) \mbox{   for all $i \in [d]$},
\end{equation}
where $\Bb$ is a hollow matrix whose diagonal entries are zeros. 

Let $Y_0$ be the $\ell$-th component of $\Vb_i$, i.e., $Y_0 := [\Vb_i]_\ell$, and $Z_0$ be the vector of all components in $\Vb$ except $Y_0$. 

Define the covariance matrix of $\begin{bmatrix}  Y_0 \\ Z_0  \end{bmatrix}$ as
$$\mathrm{Var}\left(\begin{bmatrix}  Y_0 \\ Z_0  \end{bmatrix} \right) := \begin{bmatrix}  \sigma_{Y_0Y_0} & \sigma_{Y_0Z_0} \\ \sigma_{Z_0Y_0} & \Sigma_{Z_0Z_0} \end{bmatrix}.$$

Assume that $\begin{bmatrix}  Y_0 & Z_0  \end{bmatrix}^{\transpose}$ is a multivariate Gaussian random variable, then
\begin{equation} \label{eqn:2}
\begin{split}
    Y_0| Z_0 = \bz_0 \sim N(\mu_{Y_0} + (\bz_0 - \mu_{Z_0})^{\transpose} \bSigma^{-1}_{Z_0Z_0}\sigma_{Z_0Y_0},\sigma_{Y_0Y_0} - \sigma_{Y_0Z_0} \bSigma^{-1}_{Z_0Z_0}\sigma_{Z_0Y_0}),
\end{split}
\end{equation}
where $\mu_{Y_0}$ and $\mu_{Z_0}$ are the means of $Y_0$ and $Z_0$, respectively.

Let $[a]_\ell$ be the $\ell$-th component of a vector $a$. Based on ~(\ref{eqn:cond_Gaussian1}) and ~(\ref{eqn:2}), it can be inferred that $\sum_{j = 1}^d B_{ij} [v_j]_{\ell}= \bz_0^{\transpose} \bSigma^{-1}_{Z_0 Z_0}\sigma_{Z_0 Y_0}$. Note that $[v_j]_{m}$ for $m \neq \ell$ does not show up in $\sum_{j = 1}^d B_{ij} [v_j]_{\ell}$, and thus the parameters of $[v_j]_{m}$ ($j \in [d], m \in [k] \backslash \{\ell \}$) are zeros. One way to realize this is to impose an assumption that $\Sigma_{\Vb}$ is a block diagonal matrix:
\begin{equation}\label{eqn:assumption}
    \Sigma_{\Vb} = \mbox{diag}\left( \Var([\Vb]_{1}), \Var([\Vb]_{2}), \cdots, \Var([\Vb]_{p}) \right).
\end{equation}

By the property of multivariate Gaussian distribution that zero covariance is equivalent to independence, assumption in~\ref{eqn:assumption} indicates that $[\Vb]_1, [\Vb]_2, \cdots, [\Vb]_p$ are independent. In addition, the conditional Gaussian distribution in (\ref{eqn:cond_Gaussian1}) is the same for $[\Vb]_1, [\Vb]_2, \cdots, [\Vb]_p$, and therefore $[\Vb]_1, [\Vb]_2, \cdots, [\Vb]_p$ are not only independent but also identically distributed. 

Without loss of generality, we here analyze the distribution of $[\Vb]_1$. 

Let $Y_1:=[\Vb_i]_1$ and $Z_1:=\begin{bmatrix} [\Vb_{1}]_1 \cdots[\Vb_{i-1}]_1&[\Vb_{i+1}]_1 & \cdots&  [\Vb_{d}]_1 \end{bmatrix}^{\transpose}$. Then we have
\begin{equation} \label{eqn:split}
\begin{split}
       &\Var \left( \begin{bmatrix}  Y_1 \\ Z_1  \end{bmatrix} \right)\left[\mathrm{Var}\left(\begin{bmatrix}  Y_1 \\ Z_1  \end{bmatrix} \right) \right]^{-1}:= \begin{bmatrix}  \sigma_{Y_1Y_1} & \sigma_{Y_1Z_1} \\ \sigma_{Z_1Y_1} & \Sigma_{Z_1Z_1} \end{bmatrix} \begin{bmatrix}  \theta_{Y_1Y_1} & \theta_{Y_1Z_1} \\ \theta_{Z_1Y_1} & \Omega_{Z_1Z_1} \end{bmatrix} = \begin{bmatrix} 1 & 0 \\ 0& \mathrm{I}_{d-1}  \end{bmatrix} \\ 
       \Longrightarrow 
       &\left\{ \begin{array}{lr} \sigma_{Y_1Y_1} \theta_{Y_1Y_1} + \sigma_{Y_1Z_1} \theta_{Z_1Y_1} = 1, & \\ \sigma_{Y_1Y_1} \theta_{Y_1Z_1} + \sigma_{Y_1Z_1} \Omega_{Z_1Z_1} = 0 , &\\ \sigma_{Z_1Y_1} \theta_{Y_1Y_1} + \Sigma_{Z_1Z_1} \theta_{Z_1Y_1} = 0, & \\ \sigma_{Y_1Z_1} \theta_{Y_1Z_1} + \Sigma_{Z_1Z_1} \Omega_{Z_1Z_1} = \mathrm{I}_{d-1}.& \end{array} \right.
\end{split}
\end{equation}

Based on ~(\ref{eqn:cond_Gaussian1}), the conditional Gaussian distribution of $Y_1$ is
$$Y_1 | Z_1 = \bz_1 \sim N\left(B_{i,-i} \bz_1, \sigma^2  \right) = N(\mu_{Y_1} + (\bz_1 - \mu_{Z_1})^{\transpose} \Sigma_{Z_1Z_1}^{-1}\sigma_{Z_1Y_1}, \sigma_{Y_1Y_1} - \sigma_{Y_1Z_1} \bSigma^{-1}_{Z_1Z_1}\sigma_{Z_1Y_1}),$$
where $B_{i,-i} := \begin{bmatrix}B_{i,1}&\cdots&B_{i,i-1}&B_{i,i+1}&\cdots&B_{i,d} \end{bmatrix}$ is $B_{i,:}$ without $B_{i,i}$ and $\mu_{Y_1}, \mu_{Z_1}$ are the means of $Y_1, Z_1$, respectively.

Here we assume $\mu_{Y_1} = \mu_{Z_1} = 0$. So we have
\begin{equation} \label{eqn:cond_Gaussian_V1}
    Y_1 | Z_1 = \bz_1 \sim N\left(B_{i,-i} \bz_1, \sigma^2  \right) = N( \sigma_{Y_1Z_1}\Sigma_{Z_1Z_1}^{-1}\bz_1, \sigma_{Y_1Y_1} - \sigma_{Y_1Z_1} \bSigma^{-1}_{Z_1Z_1}\sigma_{Z_1Y_1}). 
\end{equation}

According to ~(\ref{eqn:cond_Gaussian_V1}), 
\begin{equation} \label{eqn:submatrix}
\begin{split}
    &\left\{ \begin{array}{lr} B_{i,-i} = \sigma_{Y_1Z_1}\Sigma_{Z_1Z_1}^{-1} = (\Sigma_{Z_1Z_1}^{-1}\sigma_{Z_1Y_1})^{\transpose}, \\
    \sigma_{Y_1Y_1} - \sigma_{Y_1Z_1} \bSigma^{-1}_{Z_1Z_1}\sigma_{Z_1Y_1} = \sigma^2.& \end{array} \right.
\end{split}
\end{equation}

Then by ~(\ref{eqn:split}) and ~(\ref{eqn:submatrix}), we have
\begin{equation} \label{eqn:Sigma1}
\begin{split}
    & \sigma_{Z_1Y_1} \theta_{Y_1Y_1} + \Sigma_{Z_1Z_1} \theta_{Z_1Y_1} = 0\\ 
    \Longrightarrow \quad
    & \theta_{Z_1Y_1} = -\theta_{Y_1Y_1} \Sigma_{Z_1Z_1}^{-1} \sigma_{Z_1Y_1} \\ 
    \Longrightarrow \quad
    & \theta_{Z_1Y_1} = -\theta_{Y_1Y_1}  B_{i,-i}^{\transpose}.
\end{split}
\end{equation}

Furthermore, by $\sigma_{Y_1Y_1}\theta_{Y_1Y_1} + \sigma_{Y_1Z_1} \theta_{Z_1Y_1} = 1$ from ~(\ref{eqn:split}) and $\theta_{Z_1Y_1} = -\theta_{Y_1Y_1}  B_{i,-i}^{\transpose}$ from ~(\ref{eqn:Sigma1}) , we have
\begin{equation}\label{eqn:3}
    1/\theta_{Y_1Y_1} = \sigma_{Y_1Y_1} + \sigma_{Y_1Z_1} \theta_{Z_1Y_1} /\theta_{Y_1Y_1} = \sigma_{Y_1Y_1} - \sigma_{Y_1Z_1}B_{i,-i}^{\transpose}.
\end{equation}

Combining ~(\ref{eqn:submatrix}) and ~(\ref{eqn:3}), we have
\begin{equation} \label{eqn:theta_YY}
    1/\theta_{Y_1Y_1} = \sigma_{Y_1Y_1} - \sigma_{Y_1Z_1} B_{i,-i}^{\transpose} = \sigma_{Y_1Y_1} - \sigma_{Y_1Z_1} \bSigma^{-1}_{Z_1Z_1}\sigma_{Y_1Z_1}= \sigma^2.
\end{equation}

Note that ~(\ref{eqn:theta_YY}) holds for any $i \in [d]$, so the diagonal entries of $[\Var(\Vb_1)]^{-1}$ are all equal to $1/\sigma^2$. Plug $\theta_{Y_1Y_1} = 1/\sigma^2$ in $\theta_{Z_1Y_1}= - \theta_{Y_1Y_1} B_{i,-i}^{\transpose}$ from ~(\ref{eqn:Sigma1}), then we have $\theta_{Z_1Y_1}= - B_{i,-i}^{\transpose}/\sigma^2$, which holds for all $i \in [d]$. Therefore, $[\Var(\Vb_1)]^{-1} = (\Ib_d - \Bb)/\sigma^2$, where we let $\Bb$ be a symmetric hollow matrix. 

Since $[\Vb]_1, [\Vb]_2, \cdots, [\Vb]_p$ are i.i.d., the Gaussian distribution of $\Vb$ is 
$$N(0,\bOmega^{-1}) = N(0,\Ib_p \otimes [(\Ib_d - \Bb)/\sigma^2)]^{-1}) = N(0,[\Ib_p \otimes (\Ib_d - \Bb)/\sigma^2)]^{-1}) ,$$ which satisfies conditional Gaussian distribution in ~(\ref{eqn:cond_Gaussian1}). 

And the $\Bb$ here is a symmetric hollow matrix such that $\Ib_d - \Bb$ is of full rank.

\end{proof}

\subsection{Identifiability of the Model}\label{sec:prop-pf}
\begin{proof}[Proof of Propostion \ref{prop:block_model_identification}]
We first prove that $\tilde{G}=G$. Firstly, by the decompositions $\bSigma=\Ab \Qb \Ab^{\transpose} + \bGamma=\tilde{\Ab}\tilde{\Qb} \tilde{\Ab}^{\transpose} + \tilde{\bGamma}$, we know that $\max_{\ell\neq i,j}|\Sigma_{i\ell}-\Sigma_{j\ell}|=0$ for all $i,j$ such that $g(i)=g(j)$, and the same holds for all $i,j$ such that $\tilde{g}(i)=\tilde{g}(j)$. On the other side, since $\Delta(\bSigma,G)>0$, we know that for any $i\neq j$, $g(i)\neq g(j)$, there exists some $\ell\neq i,j$ such that $\Sigma_{i\ell}\neq \Sigma_{j\ell}$. And the same holds for $\tilde{G}$. This means that if $g(i)=g(j)$, then $\tilde{g}(i)=\tilde{g}(j)$ since $\max_{\ell\neq i,j}|\Sigma_{i\ell}-\Sigma_{j\ell}|=0$. And if $g(i)\neq g(j)$, then $\max_{\ell\neq i,j}|\Sigma_{i\ell}-\Sigma_{j\ell}|\neq 0$ thus it must be $\tilde{g}(i)\neq \tilde{g}(j)$ since $\Delta(\bSigma,\tilde{G})>0$. Therefore the partition $G$ and $\tilde{G}$ are the same, and $\Ab=\tilde{\Ab}$.

We then show the identifiability of $\Qb$ and $\bGamma$. For any cluster $k\neq k'$ and any $i\in G_k,j\in G_{p'}$, since $\Ab=\tilde{\Ab}$ and $\Gamma_{ij}=0$, we have $\Sigma_{ij}=\Qb_{pk'}=\tilde{\Qb}_{pk'}$. For any cluster $p$ with $|G_k|>1$, for two arbitrary members $i\neq j$, $i,j\in G_k$, by the decomposition we have $\Qb_{kk}=\Sigma_{ij}=\tilde{\Qb}_{pp}$, hence $\bGamma_{ii}=\Sigma_{ii}-\Qb_{kk} = \tilde{\bSigma}_{ii}-\tilde{\Qb}_{pp}=\tilde{\bGamma}_{ii}$ also holds for all $i\in G_k$. For any cluster $p$ with $|G_k|=1$, we have $\{i\}=G_k$ for some $i$, thus since $\tilde{G}=G$, we have $\bGamma_{ii}=\tilde{\bGamma}_{ii}=0$, hence $\Qb_{kk}=\Sigma_{ii}=\tilde{\Qb}_{pp}$. Wrapping up all above cases, we have $\Qb=\tilde{\Qb}$ and $\bGamma=\tilde{\bGamma}$.
\end{proof}

\section{Proof on the Statistical Rate of PMI}

In this section, we provide the proof of Proposition \ref{prop:hat_PMI_to_PMI}, the concentration of $\widehat{\mbox{PMI}}(w,w')$ to the truth $\mbox{PMI}(w,w')$ with high probability, and the proof of Proposition~\ref{prop:stat_pmi_to_cov} that  in $\mbox{PMI}(w,w')$ converges to $\bSigma$. 
The theoretical analysis is based on the concentration of code occurrence probabilities when the discourse variables follow their (marginal or joint) stationary distribution, which follows the discourse variables in Section~\ref{sec:discourse} and the analysis of partition functions in Section~\ref{sec:partition} . 

Recall the true PMI matrix for code $w,w'$ and window size $q$ is defined in \eqref{eq:def_stat_pmi} as
\begin{equation}
\mbox{PMI}(w,w')= \log \frac{\mathcal{N}\cdot \mathcal{N}(w,w')}{\mathcal{N}(w,\cdot)\mathcal{N}(w',\cdot)},
\end{equation}
where $\mathcal{N}_{w,w'}$ is the stationary version of expected total co-occurrence of code $w,w'$ within window size $q$, and  $\mathcal{N}(w,\cdot)=\sum_{c\in V}\mathcal{N}(w,\bc)$, $\mathcal{N}=\sum_{w,w'}\mathcal{N}(w,w')$.

We will start with formally defining how the occurrences are counted and how the empirical PMI matrix is calculated. The proofs of technical lemmas used in this section are left to Section \ref{sec:stat_PMI_to_cov}.

\subsection{Notations for code Occurrences}\label{subsec:def_word_occur}

For ease of analysis, in this section, we formally define how the empirical PMI matrix is obtained, and how these stationary versions are defined.\\

For a code $w$, let $X_w(t)$ be the indicator of the occurrence of code $w$ at time $t$, and let $X_w=\sum_{t=1}^T X_w(t)$ be the total occurrence of code $w$. Conditional on realization of discourse variables $\{\bc_t\}$, define $p_w(t):=\E[X_w(t)|\{\bc_t\}]=\frac{\exp(\langle \bv_w,\bc_t\rangle )}{\sum_{w'}\exp(\langle \bv_{w'},\bc_t\rangle )}$ and $S_w = \sum_{t=1}^T p_w(t)$ the conditional expectation of total occurrence of code $w$ at all time steps. Let $p_w=\E_{\bc\sim \mathcal{C}}[p_w(t)]$ where $\mathcal{C}$ is the stationary distribution of $\bc_t$, the uniform distribution over unit sphere. This is the stationary version of code occurrence probabilities and has nothing to do with specific $t$. Let $N_w:=\sum_{t=1}^T p_w=Tp_w$ be the (stationary) expectation of total code occurrences. 

For a pair $w,w'$ and distance $u\geq 1$, let $X_{w,w'}(t,t+u)$ be the indicator of the occurrence of code $w$ at time $t$ and $w'$ at time $t+u$, and let $X_{w,w'}^{(u)} = \sum_{t=1}^{T-u}X_{w,w'}(t,t+u)$ be the total occurrence of $(w,w')$ with distance $u$ at all time steps. 
We omit $u$ for the case $u=1$ for simplicity, and write $X_{w,w'}(t)$ for $X_{w,w'}(t,t+1)$ and $X_{w,w'}$ for $X_{w,w'}^{(1)}$. 
For realizations of discourse variables $\{\bc_t\}_{t\geq 1}$, denote the conditional expectations as $p_{w,w'}(t,t+u)=\E[X_{w,w'}(t,t+u)|\{\bc_t\}]$ and $S_{w,w'}^{(u)}=\sum_{t=1}^{T-u}p_{w,w'}(t,t+u)$. 

Similarly, let $\mathcal{D}_u$ denote the joint stationary distribution of $(\bc_t,\bc_{t+u})$, as specified in Lemma~\ref{lem:stat_joint_distr_zt}, and denote $p_{w,w'}^{(u)}=\E_{(\bc_t,\bc_{t+u})\sim \mathcal{D}_u}[p_{w,w'}(t,t+u)]$ the stationary version of code co-occurrence probabilities. Note that $p_{w,w'}^{(u)}$ is constant among all $t$. Define $N_{w,w'}^{(u)}=\sum_{t=1}^{T-u}p_{w,w'}^{(u)}=(T-u)p_{w,w'}^{(u)}$ be the sum of stationary co-occurrence probabilities. Note that here the relative position of $w,w'$ matters, i.e., the probabilitites for $w'$ appearing after $w$ with distance $u$.\\

For window size $q$, compute the total co-occurrences of $(w,w')$ within window size $q$ as
$$
X_{w,w'}^{[q]} =\sum_{u=1}^{q} \sum_{t=1}^{T-u} [X_{w,w'}(t,t+u)+X_{w',w}(t,t+u)]=\mathcal{C}(w,w')
$$
Note that this definition is equivalent to the $\mathcal{C}(w,w')$ as in ~(\ref{eqcal:def_emprical_pmi}) and ~(\ref{eq:coocur}). But for now we use the notation $X_{w,w'}^{[q]}$ to indicate the window size $q$.

Denote $S_{w,w'}^{[q]}=\E[X_{w,w'}^{[q]}|\{\bc_t\}]$ as its conditional expectation given $\{\bc_t\}_{t\geq 1}$, and $N_{w,w'}^{[q]} = \E[X_{w,w'}^{[q]}]$ be denote its expectation under stationary distribution as
\begin{equation*}
N_{w,w'}^{[q]} = \sum_{u=1}^{q} \sum_{t=1}^{T-u} [p_{w,w'}^{(u)}+ p_{w',w}^{(u)}] = \sum_{u=1}^{q}[N_{w,w'}^{(u)} + N_{w',w}^{(u)}].
\end{equation*}
Note that this definition is equivalent to the $\mathcal{N}(w,w')$ as in ~(\ref{eq:def_stat_pmi}), but for now we use the notation $N_{w,w'}^{[q]}$ to indicate the window size $q$. 

Denote $X_w^{[q]}=\sum_{w'}X_{w,w'}^{[q]}$ as the total count of co-occurrence of code $w$ with other words within window size $q$, and $X^{[q]}=\sum_{w,w'} X_{w,w'}^{[q]}$ be the total count of co-occurrences within window size $q$, with stationary version $N^{[q]}$. Note that in describing co-occurrence within a window size, whether $w$ occurs before $w'$ or vice versa does not matter. This notation coincide with previous ones that $X_{w}^{[q]}=\mathcal{C}(w,\cdot)$ $X^{[q]}=\mathcal{C}$ as in ~(\ref{eqcal:def_emprical_pmi}) and ~(\ref{eq:coocur}).

With co-occurrences computed above, the empirical PMI for code $w,w'$ is computed with and window size $q$ as $\widehat{\mbox{PMI}}_{w,w'}^{(q)} = \log \frac{X^{[q]}X_{w,w'}^{[q]}}{X_w^{[q]}X_{w'}^{[q]}}$. And the true (stationary) PMI as in ~(\ref{eq:def_stat_pmi}) for code $w,w'$ and window size $q$ is $\mbox{PMI}(w,w')= \log \frac{N^{[q]}N_{w,w'}^{[q]}}{N_w^{[q]}N_{w'}^{[q]}}$.

Furthermore, we illustrate some simple observations on the relationships in these notations for occurrences. For simplicity we use $\E_0$ to denote the expectation under stationary (joint and marginal) distributions of discourse variables.

Firstly, a simple observation is that the count of total co-occurrences within window $q$ is 
\begin{equation}
X^{[q]}=\sum_{i=1}^T \sum_{j=1}^T I{\{0<|j-i|\leq q\}}=  \sum_{t=1}^q (t-1+q) + \sum_{t=q+1}^{T-q}(2q)+ \sum_{t=T-q+1}^T (T-t+q) = 2qT-q^2-q.
\label{eq:exact_X^[q]}
\end{equation}
which is a constant. Thus $N^{[q]}\equiv X^{[q]}$. Also, by definition we know $\sum_{w'}X_{w,w'}(t,t+u) = X_w(t)$, hence by linearity of expectations, for fixed $u$,
\begin{equation*}
\sum_{w'}p_{w,w'}^{(u)} = p_w \in [0,1].
\end{equation*}
By definition,
\begin{equation*}
N_{w,w'}^{[q]} = \sum_{u=1}^{q} \sum_{t=1}^{T_u} \E_{0}[X_{w,w'}(t,t+u)+X_{w',w}(t,t+u)] = 2\sum_{u=1}^{q} \sum_{t=1}^{T-u} p_{w,w'}^{(u)}=2\sum_{u=1}^{q}(T-u)p_{w,w'}^{(u)}.
\end{equation*}
Therefore since $p_{w,w'}^{(u)}\in [0,1]$, we have
\begin{equation}
2T \sum_{u=1}^{q}p_{w,w'}^{(u)} - q(q+1)\leq N_{w,w'}^{[q]} \leq 2T \sum_{u=1}^{q}p_{w,w'}^{(u)}.
\label{eq:N_ww'^q_to_p_ww'}
\end{equation}
For a single code $w$, counting all its pairs within window size $q$ we have
\begin{equation*}
N_w^{[q]} = \sum_{w'} N_{w,w'}^{[q]} = 2\sum_{u=1}^{q}\Big((T-u)\sum_{w'} p_{w,w'}^{(u)}\Big).
\end{equation*}
Therefore as $\sum_{w'}p_{w,w'}^{(u)}=p_w \in [0,1]$ we have
\begin{equation}
2Tq p_w - q(q+1) \leq  N_w^{[q]} \leq 2qT p_w.
\label{eq:N_w^q_to_p_w}
\end{equation}

Also note that the definition here of empirical $X^{[q]}_{w,w'}$ is exactly how many times code $w,w'$ co-occur within window $q$. However, 
\begin{equation*}
X_w^{[q]} = \sum_{w'}X_{w,w'}^{[q]} = \sum_{u=1}^{q}\sum_{t=1}^{T-u} \sum_{w'}X_{w,w'}(t,t+u) + \sum_{u=1}^{q}\sum_{t=1}^{T-u} \sum_{w'}X_{w',w}(t,t+u),
\end{equation*}
where 
\begin{align*}
\sum_{u=1}^{q}\sum_{t=1}^{T-u} \sum_{w'}X_{w,w'}(t,t+u) =&
\sum_{t=1}^{T-1}\sum_{u=1}^{\min(q,T-t)}\sum_{w'}X_{w,w'}(t,t+u)  \\
=&\sum_{t=1}^{T-1}\sum_{u=1}^{\min(q,T-t)}X_{w}(t) \\
=&\sum_{t=1}^{T-1} \min(q,T-t) X_w(t) \\
=& q \sum_{t=1}^{T-1} X_w(t) + \sum_{t=1}^{q+1}(t-q)X_w(T-t)=(1+o(1))q\sum_{t=1}^{T} X_w(t).
\end{align*}
And similarly for the second term. Hence $X_w^{[q]}\approx 2 (q-1)\sum_{t=1}^{T} X_w(t)$, which is $2(q-1)$ times the total times that code $w$ appears within time $T$, since each occurrence is counted for $2(q-1)$ times within window $q$. So we have
\begin{equation*}
\begin{split}
\frac{X^{[q]}X_{w,w'}^{[q]}}{X_w^{[q]}X_{w'}^{[q]}} \approx& \frac{(2qT-q^2-q)\cdot X_{w,w'}^{[q]}}{4q^2 \sum_{t=1}^{T} X_w(t) \sum_{t=1}^{T} X_{w'}(t)} \\
=& \frac{2qT-q^2-q}{4q^2} \frac{\sum_{u=1}^{q-1}(\frac{1}{T}\sum_{t=1}^{T-u}X_{w,w'}(t,t+u)+\frac{1}{T}\sum_{t=1}^{T-u}X_{w',w}(t,t+u))}{\frac{1}{T}\sum_{t=1}^{T} X_w(t)\cdot \frac{1}{T} \sum_{t=1}^{T} X_{w'}(t)},
\end{split}
\end{equation*}
and
\begin{equation*}
\frac{N^{[q]}N_{w,w'}^{[q]}}{N_w^{[q]}N_{w'}^{[q]}}\approx \frac{wqT \cdot 2T \sum_{u=1}^{q-1}p_{w,w'}^{(u)}}{4q^2 T^2 p_w p_{w'}} = \frac{\sum_{u=1}^{q-1} (p_{w,w'}^{(u)}+p_{w',w}^{(u)}) }{2q\cdot p_w p_{w'}}.
\end{equation*}
The above two relationships will be made more specified later.

\subsection{Proof of Proposition \ref{prop:hat_PMI_to_PMI}}\label{sec:conv_hat_PMI}

Recall that  $S_{w,w'}^{[q]}=\E[X_{w,w'}^{[q]}|\{\bc_t\}]$ as its conditional expectation given $\{\bc_t\}_{t\geq 1}$, and $N_{w,w'}^{[q]} = \E[X_{w,w'}^{[q]}]$ be denote its expectation. In Lemma~\ref{lem:Sw_to_Nw} , we prove the concentration of conditional occurrence probabilities $S_{w,w'}^{(u)}$ to stationary ones $N_{w,w'}^{(u)}$ and investigate its approximate scale. In Lemma~\ref{lem:Sww'_to_Nww'}, we proceed to show the concentration of empirical occurrences $X_{w,w'}^{(u)}$ to $S_{w,w'}^{(u)}$. We refer to the proofs of these lemmas to Section \ref{sec:lm_conv_hat_PMI}.

\begin{proof}[Proof of Proposition~\ref{prop:hat_PMI_to_PMI}.]

By Lemma~\ref{lem:Sw_to_Nw} and Lemma~\ref{lem:Sww'_to_Nww'}, for fixed window size $q>0$ we have
\begin{equation*}
\P\Big( \max_{w}\Big|\frac{S_w}{N_w}-1\Big| \geq \frac{1}{\sqrt{p}} \Big) \leq \exp(-\omega(\log^2 d)) + d^{-\tau'},
\end{equation*}
\begin{equation*}
\P\Big( \max_{w,w'}\Big|\frac{S_{w,w'}^{[q]}}{N_{w,w'}^{[q]}}-1\Big| \geq \frac{1}{\sqrt{p}} \Big) \leq \exp(-\omega(\log^2 d)) + O(d^{-\tau'}),
\end{equation*}
for some large constant $\tau'>0$. By Lemma~\ref{lem:X_w_to_cond_exp} and Lemma~\ref{lem:X_ww'_to_cond_exp} we know that
\begin{equation*}
\P\Big( \max_{w}\Big|\frac{X_w}{S_w}-1\Big| \geq \frac{1}{\sqrt{p}} \Big) \leq \exp(-\omega(\log^2 d)) + d^{-\tau'},
\end{equation*}
\begin{equation*}
\P\Big( \max_{w,w'}\Big|\frac{X_{w,w'}^{[q]}}{S_{w,w'}^{[q]}}-1\Big| \geq \frac{1}{\sqrt{p}} \Big) \leq \exp(-\omega(\log^2 d)) + O(d^{-\tau'}),
\end{equation*}
for some large constant $\tau'>0$. Thus with probability at least $1-\exp(-\omega(\log^2 d)) - O(d^{-\tau'})$, it holds that
\begin{equation*}
\begin{split}
(1-\frac{1}{\sqrt{p}})^2\leq \min_{w} \frac{X_w^{[q]}}{N_w^{[q]}}  \leq& \max_{w} \frac{X_w^{[q]}}{N_w^{[q]}} \leq (1+\frac{1}{\sqrt{p}})^2,\\
(1-\frac{1}{\sqrt{p}})^2\leq \min_{w,w'} \frac{X_{w,w'}^{[q]}}{N_{w,w'}^{[q]}}  \leq& \max_{w,w'} \frac{X_{w,w'}^{[q]}}{N_{w,w'}^{[q]}}\leq (1+\frac{1}{\sqrt{p}})^2.
\end{split}
\end{equation*}
Also recall the definition 
$\widehat{\mbox{PMI}}(w,w') = \log \frac{X^{[q]}X_{w,w'}^{[q]}}{X_{w}^{[q]}X_{w'}^{[q]}}$ and $ \mbox{PMI}(w,w') =  \log \frac{N^{[q]}N_{w,w'}^{[q]}}{N_{w}^{[q]}N_{w'}^{[q]}}$, therefore with probability at least $1-\exp(-\omega(\log^2 d)) - O(d^{-\tau'})$, it holds for all $w,w'$ that 
\begin{equation*}
1 - \frac{4}{\sqrt{p}} \leq \frac{1-\frac{1}{\sqrt{p}}}{(1+\frac{1}{\sqrt{p}})^2} \leq \frac{X^{[q]}X_{w,w'}^{[q]}}{X_{w}^{[q]}X_{w'}^{[q]}} \Big/\frac{N^{[q]}N_{w,w'}^{[q]}}{N_{w}^{[q]}N_{w'}^{[q]}} =  \frac{ X_{w,w'}^{[q]}}{X_{w}^{[q]}X_{w'}^{[q]}} \Big/\frac{ N_{w,w'}^{[q]}}{N_{w}^{[q]}N_{w'}^{[q]}} \leq \frac{1+\frac{1}{\sqrt{p}}}{(1-\frac{1}{\sqrt{p}})^2} \leq 1+\frac{4}{\sqrt{p}},
\end{equation*}
for appropriately large $p$, in which case
$$
\max_{w,w'}\big|\widehat{\mbox{PMI}}(w,w') - \mbox{PMI}(w,w')\big| \leq \max\Big\{\big|\log(1-\frac{4}{\sqrt{p}})\big|,\big|\log(1+\frac{4}{\sqrt{p}})\big|\Big\} \leq \frac{5}{\sqrt{p}},
$$
for appropriately large $p$.
\end{proof}

\subsection{Proof of Proposition~\ref{prop:stat_pmi_to_cov}}\label{sec:stat_PMI_to_cov_pf}

We now provide the main result for the concentration of stationary PMI to the covariance matrix of code vectors in our graphical model. It is based on the concentration of stationary co-occurrence probabilities, stated in the following lemma.

\begin{lemma}[Concentration and boundedness of stationary probabilities] \label{lem:pw_pww'_to_cov}
Assume all $\Sigma_{ww'}$'s are bounded from above and below by some constants $\underline{c}\leq \min_{w,w'}\Sigma_{ww'}\leq \max_{w,w'}\Sigma_{ww'}\leq \overline{c}$, and Assumptions~\ref{assump:parameter_space} and~\ref{assump:kdT} hold. Then with probability at least $1-\exp(-\omega(\log^2 d))-d^{-\tau}$, it holds for any constant $\tau>0$ and constant $C_\tau=12\sqrt{3(\tau+4)}$ that
\begin{equation}
\max_{w}\Big| \log(p_w) - \big( \frac{\Sigma_{ww}}{2} - \log Z \big)\Big| \leq \frac{2+C_\tau\|\bSigma\|_{\max}\sqrt{\log d}}{\sqrt{p}},
\label{eq:pw_to_cov}
\end{equation}
\begin{equation}
\max_{w,w'}\Big| \log(p_{w,w'}) - \big( \frac{\Sigma_{ww}+\Sigma_{w'w'}+2\Sigma_{ww'}}{2} - 2\log Z \big)\Big| \leq (2C_\tau\|\bSigma\|_{\max}+6)\sqrt{\frac{\log d}{p}},
\label{eq:pww'_to_cov}
\end{equation}
\begin{equation}
\max_{w,w'}\Big| \log(p_{w,w'}^{(u)}) - \big( \frac{\Sigma_{ww}+\Sigma_{w'w'}+2\Sigma_{ww'}}{2} - 2\log Z \big)\Big| \leq (2C_\tau\|\bSigma\|_{\max}+7\sqrt{2u})\sqrt{\frac{\log d}{p}},
\label{eq:p_ww'^u_to_cov}
\end{equation}
in which case for $u\leq q$ for some fixed window size $q$, there exists some constants $\overline{p},\underline{p}$ such that 
\begin{equation}
\underline{p}/d \leq \min_w p_w \leq \max_w p_w \leq \overline{p}/d,
\label{eq:bound_pw_scale}
\end{equation}
\begin{equation}
\underline{p}/d^2 \leq \min_{w,w'} p_{w,w'}^{(u)} \leq  \max_{w,w'} p_{w,w'}^{(u)} \leq \overline{p}/d^2.
\label{eq:bound_pww'_scale}
\end{equation}
\end{lemma}

The proof of the lemma is presented in Section \ref{subsec:stat_PMI_to_cov}.

\begin{proof}[Proof of Proposition~\ref{prop:stat_pmi_to_cov}]
By Lemma \ref{lem:pw_pww'_to_cov}, with probability at least $1-\exp(-\omega(\log^2 d))-2d^{-\tau}$, the conditions in  (\ref{eq:pw_to_cov})-(\ref{eq:p_ww'^u_to_cov}) hold and the bounds in  (\ref{eq:bound_pw_scale}) and (\ref{eq:bound_pww'_scale}) hold simultaneously for all $w,w'$. In this case, from the analysis of stationary occurrence expectations in  (\ref{eq:exact_X^[q]}), (\ref{eq:N_ww'^q_to_p_ww'}) and (\ref{eq:N_w^q_to_p_w}), we have
\begin{align}
\mbox{PMI}(w,w')= \frac{N^{[q]}N_{w,w'}^{[q]}}{N_w^{[q]}N_{w'}^{[q]}} \leq& \frac{2qT \cdot 2T \sum_{u=1}^{q}p_{w,w'}^{(u)} }{q^2 (2Tp_w -q)(2Tp_{w'} -q)} \notag \\
=& \frac{1}{(1 -\frac{q}{2Tp_w})(1 -\frac{q}{2Tp_{w'}})}\cdot \frac{\frac{1}{q} \sum_{u=1}^{q}p_{w,w'}^{(u)} }{p_wp_{w'}},
\label{eq:stat_pmi_upper_bound}
\end{align}
where for each $u$, due to  (\ref{eq:pw_to_cov}) and (\ref{eq:p_ww'^u_to_cov}) we have
\begin{equation*}
\Sigma_{w,w'} - \frac{4+(4C_\tau\|\bSigma\|_{\max}+7\sqrt{2q})\log d}{\sqrt{p}} \leq \log(\frac{p_{w,w'}^{(u)}}{p_w p_{w'}}) \leq \Sigma_{w,w'} + \frac{4+(4C_\tau\|\bSigma\|_{\max}+7\sqrt{2q})\sqrt{\log d}}{\sqrt{p}}.
\end{equation*}
Taking exponential and averaging over $u=1,\dots,q-1$ we have
\begin{align*}
&\exp(\Sigma_{w,w'} - \frac{4+(4C_\tau\|\bSigma\|_{\max}+7\sqrt{2q})\sqrt{\log d}}{\sqrt{p}})
\\
&\leq \frac{\frac{1}{q-1} \sum_{u=1}^{q-1}p_{w,w'}^{(u)} }{p_wp_{w'}} \leq \exp(\Sigma_{w,w'} + \frac{4+(4C_\tau\|\bSigma\|_{\max}+7\sqrt{2q})\sqrt{\log d}}{\sqrt{p}}).
\end{align*}

Taking logarithm to both sides in  (\ref{eq:stat_pmi_upper_bound}), we have
\begin{align*}
\log \frac{N^{[q]}N_{w,w'}^{[q]}}{N_w^{[q]}N_{w'}^{[q]}} \leq& \log(\frac{p_{w,w'}^{(u)}}{p_w p_{w'}}) 
- \log(1 -\frac{q-1}{2Tp_w}) - \log(1 -\frac{q-1}{2Tp_{w'}}) \\
\leq & \Sigma_{w,w'} + \frac{4+(4C_\tau\|\bSigma\|_{\max} +7\sqrt{2u})\log d}{\sqrt{p}} 
- \log(1 -\frac{q-1}{2Tp_w}) - \log(1 -\frac{q-1}{2Tp_{w'}}).
\end{align*}
Due to the bound for $p_w$ and $p_{w'}$ in  (\ref{eq:bound_pw_scale}), as well as the fact that $d=o(T)$, we know $\sqrt{p}=o(Tp_w)$, hence since $\log(1-x)\geq -2x$ for $x\in (0,1/2)$, for appropriately large $(d,p,T)$, 
\begin{align*}
\log \frac{N^{[q]}N_{w,w'}^{[q]}}{N_w^{[q]}N_{w'}^{[q]}} \leq & \Sigma_{w,w'} + \frac{4+(4C_\tau\|\bSigma\|_{\max}+7\sqrt{2q})\log d}{\sqrt{p}}
+\frac{q-1}{Tp_w}+\frac{q-1}{Tp_{w'}}\\
\leq& \Sigma_{w,w'} + \frac{5+(4C_\tau\|\bSigma\|_{\max}+7\sqrt{2q})\log d}{\sqrt{p}}.
\end{align*}

On the other hand, from the analysis in  (\ref{eq:exact_X^[q]})-(\ref{eq:N_w^q_to_p_w}), we have
\begin{align*}
\mbox{PMI}(w,w')= \frac{N^{[q]}N_{w,w'}^{[q]}}{N_w^{[q]}N_{w'}^{[q]}} \geq& \frac{(q-1)(T-\frac{q-1}{2}) \cdot [2T  \sum_{u=1}^{q-1}p_{w,w'}^{(u)}-q(q-1)] }{4T^2(q-1)^2 p_w p_{w'}}\\
=&\frac{(1-\frac{q-1}{2T}) \cdot [  \frac{1}{q-1}\sum_{u=1}^{q-1}p_{w,w'}^{(u)}-\frac{q}{2T}] }{2p_w p_{w'}}
\end{align*}
where for appropriately large $(d,p,T)$, we have $\frac{q(q-1)}{T  \sum_{u=1}^{q-1}p_{w,w'}^{(u)}}\leq \frac{1}{2\sqrt{p}}<\frac{1}{2}$ and $\frac{q-1}{T}\leq \frac{1}{2\sqrt{p}}<\frac{1}{2}$ since $p_{w,w'}^{(u)}=\Omega(1/d^2)$. Thus since $\log(1-x)\geq -2x$ for $x\in (0,1/2)$,
\begin{align*}
\log \frac{N^{[q]}N_{w,w'}^{[q]}}{N_w^{[q]}N_{w'}^{[q]}}\geq& \log \frac{\frac{1}{q-1} \sum_{u=1}^{q-1}p_{w,w'}^{(u)} }{p_wp_{w'}} +2\log(1-\frac{1}{4\sqrt{p}}) \\
\geq& \Sigma_{w,w'} - \frac{5+(4C_\tau\|\bSigma\|_{\max}+7\sqrt{2q})\sqrt{\log d}}{\sqrt{p}}.
\end{align*}
Combining the above two directions we have that with probability at least $1-\exp(-\omega(\log^2 d))-2d^{-\tau}$ for some large constant $\tau$, it holds that 
\begin{equation*}
\Big|\log (\frac{N^{[q]}N_{w,w'}^{[q]}}{N_w^{[q]}N_{w'}^{[q]}}) - \Sigma_{ww'}\Big| \leq \frac{5+(4C_\tau\|\bSigma\|_{\max}+7\sqrt{2q})\log d}{\sqrt{p}}.
\end{equation*}

\end{proof}


\section{Word Cluster Analysis}

In this section we provide the proof for exact cluster recovery and precision matrix estimation under block model. In Section~\ref{sec:appendix_cluster_recovery}, we show that our algorithm achieves exact recovery of word clusters with high probability, and in Section~\ref{sec:appendix_precision_after_recovery}, we provide estimation accuracy of the precision matrix after recovering the block structure.

\subsection{Exact Cluster Recovery}\label{sec:appendix_cluster_recovery}

\begin{proof}[Proof of Theorem~\ref{thm:recover_cluster}]
Denote the mapping $g$ such that $g(w)=j$ if $w\in G_j$. Define the difference of word $w,w'$ as 
$$
d(w,w') = \max_{c\neq w,w'} \big| \bm{\Sigma}_{wc}-\bm{\Sigma}_{w'c} \big|.
$$
Then by the block-wise structure in Assumption \ref{assump:block_bound}, $d(w,w')=\max_{k\neq g(w),g(w')}|\Qb_{k,g(w)}-\Qb_{k,g(w')}|$. So Assumption \ref{assump:block_bound} ensures $\min_{g(w)\neq g(w')}d(w,w') = \Delta(\Qb)\geq \eta$, with $\sqrt{\frac{\log d}{p}}=o(\eta)$.

By Propositions \ref{prop:hat_PMI_to_PMI} and \ref{prop:stat_pmi_to_cov}, with probability at least $1-\exp(-\omega(\log^2 d))-O(d^{-\tau})$ for some large constant $\tau>0$, $\|\widehat{\mbox{SPPMI}}-\bm{\Sigma}\|_{\max} = O(\sqrt{\frac{\log d}{p}})$. Thus for any $w,w',c\in [d]$,
$$
\Big|\big(\widehat{\mbox{SPPMI}}(w,c)-\widehat{\mbox{SPPMI}}(w,c)|\big)- \big(\bm{\Sigma}_{wc}-\bm{\Sigma}_{w'c}\big)\Big| \leq 2\|\widehat{\mbox{SPPMI}}-\bm{\Sigma}\|_{\max}.
$$
Hence for any $w\neq w'\in [d]$,
$$
\big|d(w,w')-cd(w,w')\big|\leq  2\|\widehat{\mbox{SPPMI}}-\bm{\Sigma}\|_{\max}=O\Big( \sqrt{\frac{\log d}{p}} \Big).
$$
Thus for any $w,w'$ with $g(w)=g(w')$, by the  it holds that $d(w,w')=0$, so $d(w,w')\leq 2\|\widehat{\mbox{SPPMI}}-\bm{\Sigma}\|_{\max}$. Also, for any $w,w'$ with $g(w)\neq g(w')$, we have $d(w,w')\geq \eta-2\|\widetilde{\mbox{SPPMI}}-\bm{\Sigma}\|_{\max}$. Now since $0<\alpha\leq \eta/2$, for sufficiently large $k,d,T$ such that $2\|\widehat{\mbox{SPPMI}}-\bm{\Sigma}\|_{\max}<\alpha/2$ we have that for any $w,w'$ with $g(w)=g(w')$, $d(w,w')<\alpha$ and for any $w,w'$ with $g(w)\neq g(w')$, $d(w,w')>\alpha$. 

We prove the exact recovery in the above case, which happens with probability at least $1-\exp(-\omega(\log^2 d))-O(d^{-\tau})$, by induction on the number of steps $\ell$. Suppose it is consistent up to the $\ell$-th step, i.e. $\widehat{G}_j=G_{g(w_j)}$ for $j=1,\dots,\ell-1$. Then if $|S|=1$, it directly follows that $\widehat{G}=G$. 

Otherwise, if $d(w_\ell,w'_\ell)>\alpha$, then no $w'\in [d]$ is in the same group as $w_\ell$. By assumption the algorithm is consistent up to the $\ell$-th step, so $w_\ell$ is also not in the same group as those in $\cup_{j=1}^{\ell-1}\widehat{G}_j$, hence $w_\ell$ is a singleton, and $\widehat{G}_{\ell}=\{w_\ell\}=G_{g(a_{\ell})}$.

If $d(w_\ell,w'_\ell)\leq \alpha$, then $w_\ell,w'_\ell$ must be in the same group. Also we've seen that $d(x,y)\leq \alpha$ if and only if $g(x)=g(y)$. So for any $c\neq w,w'$, $g(c)=g(w)$ if and only if $d(w_\ell,c)\leq \alpha$ and $d(w'_\ell,c)\leq \alpha$. So $\widehat{G}_\ell = S\cap G_{g(w_\ell)}$. Since the algorithm is consistent up to step $\ell$, no members in the same group as $w_\ell$ has been included in $\cup_{j=1}^{\ell-1}\widehat{G}_j$, so $G_{g(w_\ell)}\subset S$, hence $\widehat{G}_\ell = G^\star_{g(w_\ell)}$.

So the algorithm is also consistent at the $\ell$-th step. The consistency of the Word Clustering algorithm follows by induction.
\end{proof}

\subsection{Precision Matrix Estimation after Exact Recovery}\label{sec:appendix_precision_after_recovery}
Once the partition $G$ is exactly recovered, the refined precision matrix is also sufficiently close to the true covariance matrix, which guarantees the accurate estimate of the precision matrix. In this section, we provide proof of Corollary \ref{cor:hat_PMI} about the concentration of $\widetilde{\mbox{SPPMI}}$ to $\bm{\Sigma}=\bm{\Omega}^{-1}$, and the estimation accuracy for the true precision matrix $\bm{\Omega}$.

\begin{proof}[Proof of Corollary \ref{cor:hat_PMI}]
By Propositions \ref{prop:hat_PMI_to_PMI} and \ref{prop:stat_pmi_to_cov}, with probability at least $1-\exp(-\omega(\log^2 d))-O(d^{-\tau})$ for some large constant $\tau>0$, $\|\PMIbb-\bm{\Sigma}+\log 2\|_{\max} = O(\sqrt{\frac{\log d}{p}})$. By definition of $\widehat{\mbox{SPPMI}}$, since $\eta \leq \min{\bSigma_{ij}},$ we have 
$$
\|\widehat{\mbox{SPPMI}}-\bm{\Sigma}\|_{\max} \leq \|\PMIbb-\bm{\Sigma}+\log 2\|_{\max}= O(\sqrt{\frac{\log d}{p}}).
$$
By the block structure in Assumption \ref{assump:block_bound}, for any $w\neq w'$, $\bm{\Sigma}_{ww'}=\Qb_{g(w),g(w')}$. Therefore with probability at least $1-\exp(-\omega(\log^2 d))-O(d^{-\tau})$, for any two members $w\in G_i$, $c\in G_j$, $w\neq c$, it holds that
$$
\big|\widehat{\mbox{SPPMI}}(w,c) - \bm{\Sigma}_{wc}\big|=\big|\widehat{\mbox{SPPMI}}(w,c) - \Qb_{ij}\big| \leq  c\sqrt{\frac{\log d}{p}}
$$
for some constant $c>0$. When all clusters are perfectly recovered with $\widehat{G}=G$, for all $i\neq j$, we have
$$
\Qb_{ij} - c\sqrt{\frac{\log d}{p}} \leq \widetilde{\mbox{SPPMI}}(i,j) = \frac{1}{|G_i|\cdot |G_j|}\sum_{w\in G_i,c\in G_j} \widehat{\mbox{SPPMI}}(w,c) \leq \Qb_{ij} + c\sqrt{\frac{\log d}{p}}.
$$
On the other hand, for diagonal entries of $\widetilde{\mbox{SPPMI}}$, for any $j\in [K]$ with $|G_j|>1$, note that for any $w\neq c\in G_j$, the block model implies $\bSigma_{wc}=\Qb_{jj}$, thus
$$
\Qb_{jj} - c\sqrt{\frac{\log d}{p}} \leq \widetilde{\mbox{SPPMI}}(j,j) = \frac{2}{|G_j|\cdot(|G_j|-1)}\sum_{w,c\in G_j,w\neq c} \widehat{\mbox{SPPMI}}(w,c) \leq \Qb_{jj}+ c\sqrt{\frac{\log d}{p}}.
$$
Lastly, for those $j\in[K]$ such that $|G_j|=1$, $G_j=\{w\}$ with $w$ the $i$-th word, since $\bGamma_{ii}=0$ we have $\bSigma_{ii}=\Qb_{jj}$, hence
$$
\Qb_{jj}- c\sqrt{\frac{\log d}{p}} \leq \widetilde{\mbox{SPPMI}}(j,j) = \widehat{\mbox{SPPMI}}(w,w) \leq \Qb_{jj} + c\sqrt{\frac{\log d}{p}}.
$$
Wrapping up all cases above yields
$$
\max_{w,c}\big|\widetilde{\mbox{SPPMI}}(i,j) - \Qb_{ij}\big|\leq  c\sqrt{\frac{\log d}{p}},
$$
which by union bound happens with probability at least $1-\exp(-\omega(\log^2 d))-O(d^{-\tau})$.

\end{proof}

We now provide proof for the consistency of the CLIME-type estimator.

\begin{proof}[Proof of Theorem \ref{thm:main}]

Since $\|\widetilde{\mbox{SPPMI}}-\Qb\|_{\max}=o(\lambda)$, for sufficiently large $k,d,T$, with probability at least $1-\exp(-\omega(\log^2 d))-O(d^{-\tau})$, $\bm{\beta} = \Cb_{\cdot \ell}$ is feasible for the $\ell$-th optimization problem in Eq. (\ref{eq:CLIME-1}), thus $\|\widehat{\Cb}_{\cdot \ell}\|_1\leq \|\Cb_{\cdot \ell}\|_1$ and $\|\widehat{\Cb}\|_{1}\leq \|\Cb\|_{1}$. Also by definition, $\|\widetilde{\mbox{SPPMI}}\widehat{\Cb} - \Ib\|_{\max}\leq \lambda$. Hence with probability at least $1-\exp(-\omega(\log^2 d))-O(d^{-\tau})$ for some large constant $\tau$, 
\begin{align*}
\|\widehat{\Cb}-\Cb\|_{\max} =& \|\Cb(\Qb(\widehat{\Cb}-\Cb))\|_{\max}\\
\leq& \|\Cb\|_{1} \|\Qb(\widetilde{\Cb}-\Cb))\|_{\max}\\
\leq& \|\Cb\|_{1}\big[ \|\widetilde{\mbox{SPPMI}}(\widetilde{\Cb}-\Cb))\|_{\max} + \|(\widetilde{\mbox{SPPMI}}-\Qb)(\widetilde{\Cb}-\Cb))\|_{\max}\big]\\
\leq& \|\Cb\|_{1}\big[ \|\widetilde{\mbox{SPPMI}}\cdot \widetilde{\Cb}-\Ib\|_{\max} + \|\widetilde{\mbox{SPPMI}}\cdot \Cb - \Ib\|_{\max} + \|\widetilde{\mbox{SPPMI}}-\Qb\|_{\max}\|\widetilde{\bm{\Omega}}-\Cb\|_{1}\big]\\
\leq& \|\Cb\|_{1}\big[ \lambda + 3\|\Cb\|_{1}\|\widetilde{\mbox{SPPMI}}-\Qb\|_{\max} \big]\\
\leq& \lambda \|\Cb\|_{1} + O\Big(\sqrt{\frac{\log d}{p}} \Big).
\end{align*}
Here the first equation is just $\bm{CQ}=\Ib$, and the second line is due to $\|AB\|_{\max}\leq \|A\|_{1}\|B\|_{\max}$ for two matrices $A,B$. The Third line is triangle inequality. The fourth line is triangle inequality combined with the inequality $\|AB\|_{\max}\leq \|A\|_{1}\|B\|_{\max}$. The last two lines are due to the assumptions.

Then by Assumption~\ref{assump:parameter_space}, the sparsity yields
$$
\|\widehat{\Cb}-\Cb\|_{1}\leq s\|\widehat{\Cb}-\Cb\|_{\max} \leq s\lambda \|\Cb\|_{1} + O\Big(s\sqrt{\frac{\log d}{p}} \Big).
$$
\end{proof}

\section{Properties of Discourse Variables}\label{sec:discourse}
In this section, we provide some useful properties of the hidden Markov process $\{\bz_t,\bc_t\}_{t\geq 1}$ that are related to concentration properties of stationary distributions. 
In Subsection~\ref{subsec:slow_moving_ct} we specify the (joint) stationary distributions of $\{\bz_t\}_{t\geq 1}$ and $\{\bc_t\}_{t\geq 1}$, and show that under stationary distribution, the discourse vectors $\{\bc_t\}_{t\geq 0}$ ``moves slowly'' on the unit sphere, which will be of use later in showing the convergence of true PMI matrix to covariance in our Gaussian graphical model.
We also show the mixing properties of $\{\bz_t\}_{t\geq 1}$ in Subsection~\ref{subsec:mixing_zt} and \ref{subsec:mixing_zt_zt+u}, i.e., the convergence of marginal distribution of $\{\bz_t\}_{t\geq 1}$ to its stationary version. The total variation distances enjoy exponential decay.

Recall that the hidden Markov process of the discourse variables $\{\bc_t\}_{t\geq 1}$ is specified as
\begin{equation}
\bz_2 = \sqrt{\alpha}\frac{\bz_1}{\|\bz_1\|_2} + \sqrt{1-\alpha}r_2,\quad \bz_{t+1} = \sqrt{\alpha}\bz_t + \sqrt{1-\alpha} r_{t+1},t\geq 2.
\label{eq:zt_def}
\end{equation}
Here $\bz_t$ is nonzero with probability $1$, $\alpha = 1-\frac{\log d}{p^2}$, $r_t$ i.i.d. $\sim N(0,\Ib_p/p)$. And
\begin{equation}
\bc_t = {\bz_t}/{\|\bz_t\|_2}.
\label{eq:ct_def}
\end{equation}


\subsection{Slow Moving of Stationary Disclosure}\label{subsec:slow_moving_ct}

Denote $\mathcal{D}_u$ the joint stationary distribution of $(\bc_t,\bc_{t+u})$, which does not depend on $t$. We first specify the stationary distributions of the hidden Markov process.


\begin{lemma}[Joint stationary distributions of $\bc_t$]
The distribution $\mathcal{D}_u$ is the same as $(\frac{\bz}{\|\bz\|_2},\frac{\bz'}{\|\bz'\|_2})$ where $(\bz,\bz')$ is the joint stationary distribution of $(\bz_t,\bz_{t+u})$, which is jointly Gaussian with mean zero, $\Cov(\bz)=\Cov(\bz')=\Ib_p/p$, $\E[\bz(\bz')^{\transpose}] = \alpha^u \Ib_p/p$.
\label{lem:stat_joint_distr_zt}
\end{lemma}
\begin{proof}[Proof of Lemma~\ref{lem:stat_joint_distr_zt}]
By construction of $\bz_t$, we have 
$$
\frac{\bz_{t+u}}{\alpha^{(t+u)/2}} = \frac{\bz_t}{\alpha^{t/2}} + \sum_{j=1}^u \frac{\sqrt{1-\alpha}}{\alpha^{(t+j)/2}} r_{t+j},
$$
where $r_{t+1},\dots,r_{t+u}$ are i.i.d. $N(0,\Ib_p/p)$ random variables that are independent of $\bz_t$. Hence
\begin{equation}\label{eq:zt_to_zt+u_relationship}
\bz_{t+u} = \alpha^{u/2} \bz_t + \sqrt{1-\alpha^u} r,
\end{equation}
where $r\sim N(0,\Ib_p/p)$ and is independent of $\bz_t$. 

The stationary distribution of Markov process $\{\bz_t\}_{t \geq 1}$ is $N(0, \Ib_p/p)$. One straightforward way to see this is to let $\bz_t \sim N(0,\Ib_p/p)$ for $t \geq 2$, since $r_{t+1}$ and $\bz_t$ are independent, $\bz_{t+1} \sim N(0, \left[(\sqrt{\alpha})^2 + (\sqrt{1 - \alpha})^2\right] \cdot \Ib_p/p) = N(0,\Ib_p/p)$.  This holds true for any $t \geq 2$, so $N(0,\Ib_p/p)$ is a stationary distribution of $\{\bz_t\}_{t \geq 1}$. Also, the marginal distribution of $\{\bz_t\}$ converges to $N(0,\Ib_p/p)$ as $t\to\infty$, regardless of the starting point $\bz_1$. To see this, note that given $\bz_1$, the distribution of $\bz_{t+1}$ is
$$
\bz_{t+1}|\bz_1 \sim N(\alpha^{\frac{t}{2}}\frac{\bz_1}{\|\bz_1\|_2}, (1-\alpha^t)\Ib_p/p),
$$
where the mean converges to zero and the covariance matrix converges to $\Ib_p/p$. In this case, $(\bz_t,\bz_{t+u})$ is jointly normal with mean zero, $\Cov(\bz_t)=\Cov(\bz_{t+u})=\Ib_p/p$ and $\E[\bz_{t}\bz_{t+u}^{\transpose}] = \alpha^u \Ib_p/p$. By definition, since $\bc_t = \bz_t/\|\bz_t\|_2$, the distribution $\mathcal{D}_u$ is the same as described.

\end{proof}

With the joint stationary distribution in hand, we have the following result on the ``moving step'' $\|\bc_t-\bc_{t+u}\|_2$ under stationary distributions.

\begin{lemma}[Slow moving of $\bc_t$ under stationary distribution]\label{lem:moving_step_ct_ct+u}
Suppose $(\bc,\bc')\sim \mathcal{D}_u$ for some fixed $u>0$. Then with probability at least $1-\exp(-\omega(\log^2 d))$,
\begin{equation}
\|\bc - \bc'\|_2\leq \frac{20\sqrt{2u\log d}}{3p} = \Omega(\frac{\sqrt{\log d}}{p}).
\end{equation}
Particularly, for $u=1$, with probability at least $1-\exp(-\omega(\log^2 d))$,
\begin{equation}
\|\bc - \bc'\|_2\leq \frac{5\sqrt{\log d}}{p}.
\end{equation}
\end{lemma}

\begin{proof}[Proof of Lemma\ref{lem:moving_step_ct_ct+u}]
By Lemma \ref{lem:stat_joint_distr_zt}, we have
$$
(\bc,\bc')\stackrel{\text{d}}{=}\big(\frac{z}{\|z\|_2},\frac{z'}{\|z'\|_2}\big) \stackrel{\text{d}}{=}\Big(\frac{z}{\|z\|_2},\frac{\alpha^{u/2} z + \sqrt{1-\alpha^u} r}{\|\alpha^{u/2} z + \sqrt{1-\alpha^u} r\|_2}\Big),
$$
where $z,r$ are independent $N(0,\Ib_p/p)$ random vectors. By the property of chi-squared random variables stated in Lemma~\ref{lem:z_t_r_t1}, with probability $1-\exp(-\omega(\log^2 d))$, $\frac{1}{2}\leq \|z\|/\|r\|_2\leq 2$. Also recall that $\alpha = 1-\frac{\log d}{p^2}$ where $\log d = o(\sqrt{p})$, so for fixed $u$, we safely assume $\alpha^u > 8/9$. Now we consider $\frac{z}{\|z\|_2}-\frac{z'}{\|z'\|_2}$, where $z'=\alpha^{u/2} z + \sqrt{1-\alpha^u} r$.
\begin{align*}
\bigg\|\frac{z}{\|z\|_2}-\frac{z'}{\|z'\|_2}\bigg\|_2=& \bigg\|\frac{\|z'\|_2z -\|z\|_2z'}{\|z\|_2\cdot \|z'\|_2}\bigg\|_2,
\end{align*}
where $$
\big| \alpha^{u/2}\|z\|_2- \sqrt{1-\alpha^u} \|r\|_2\big| \leq \|z'\|_2\leq \alpha^{u/2}\|z\|_2+  \sqrt{1-\alpha^u} \|r\|_2.
$$
With probability $1-\exp(\omega(-\log^2 d))$, we have 
$$
\alpha^{u/2}\|z\|_2- \sqrt{1-\alpha^u} \|r\|_2\geq (2\alpha^{u/2}-\sqrt{1-\alpha^u})\|r\|_2\geq 0.
$$
Thus with probability $1-\exp(\omega(-\log^2 d))$, we have 
\begin{align*}
\bigg\|\frac{z}{\|z\|_2}-\frac{z'}{\|z'\|_2}\bigg\|_2\leq & \max\bigg\{  \bigg\|\frac{(\alpha^{u/2}\|z\|_2+  \sqrt{1-\alpha^u} \|r\|_2)z -\|z\|z'}{\|z\|_2\cdot \|z'\|_2}\bigg\|_2, \bigg\|\frac{(\alpha^{u/2}\|z\|_2- \sqrt{1-\alpha^u} \|r\|_2)z -\|z\|_2z'}{\|z\|_2\cdot \|z'\|_2}\bigg\|_2 \bigg\},
\end{align*}
where
\begin{align*}
&\bigg\|\frac{(\alpha^{u/2}\|z\|_2+  \sqrt{1-\alpha^u} \|r\|_2)z -\|z\|z'}{\|z\|_2\cdot \|z'\|_2}\bigg\|_2\\
=& \bigg\|\frac{(\alpha^{u/2}\|z\|_2+  \sqrt{1-\alpha^u} \|r\|_2)z -\|z\|(\alpha^{u/2} z + \sqrt{1-\alpha^u} r)}{\|z\|_2\cdot \|z'\|_2}\bigg\|_2\\
=& \sqrt{1-\alpha^u} \bigg\|\frac{\|r\|_2z-\|z\|_2r}{\|z\|_2\cdot \|z'\|_2}\bigg\|_2\leq 2\sqrt{1-\alpha^u} \frac{\|r\|_2}{\|z'\|_2},
\end{align*}
\begin{align*}
&\bigg\|\frac{(\alpha^{u/2}\|z\|_2- \sqrt{1-\alpha^u} \|r\|_2)z -\|z\|z'}{\|z\|_2\cdot \|z'\|_2}\bigg\|_2\\
=& \bigg\|\frac{(\alpha^{u/2}\|z\|_2- \sqrt{1-\alpha^u} \|r\|_2)z -\|z\|(\alpha^{u/2} z + \sqrt{1-\alpha^u} r)}{\|z\|_2\cdot \|z'\|_2}\bigg\|_2\\
=&  \sqrt{1-\alpha^u} \bigg\|\frac{\|r\|_2z+\|z\|_2r}{\|z\|_2\cdot \|z'\|_2}\bigg\|_2\leq 2\sqrt{1-\alpha^u} \frac{\|r\|_2}{\|z'\|_2}.
\end{align*}
Furthermore, since $\frac{1}{2}\leq \frac{\|z\|_2}{\|r\|_2}\leq 2$ and $1-\alpha^u \leq \frac{1}{9}$ hence $\alpha^{u/2}- \sqrt{1-\alpha^u}\geq \frac{3}{5}$,
\begin{align*}
2\sqrt{1-\alpha^u} \frac{\|r\|_2}{\|z'\|_2} \leq& 2\sqrt{1-\alpha^u} \frac{\|r\|_2}{\alpha^{u/2}\|z\|_2- \sqrt{1-\alpha^u} \|r\|_2}\\
\leq& \frac{4\sqrt{1-\alpha^u}}{\alpha^{u/2}-\sqrt{1-\alpha^u} } \leq \frac{20}{3} \sqrt{1-\alpha^u} \leq \frac{20}{3} \sqrt{1-(1-\frac{\log d}{p^2})^u} \leq \frac{20\sqrt{2u\log d}}{3p}.
\end{align*}
Therefore for fixed $u$, with probability $1-\exp(-\omega(\log^2 d))$,
$\|\bc - \bc'\|_2\leq \frac{20\sqrt{2u\log d}}{3p}$. Specifically, for $u=1$,
\begin{align*}
2\sqrt{1-\alpha^u} \frac{\|r\|_2}{\|z'\|_2} 
\leq& \frac{4\sqrt{1-\alpha}}{\sqrt{\alpha}-\sqrt{1-\alpha} } \leq 5\sqrt{1-\alpha} \leq \frac{5\sqrt{\log d}}{p}
\end{align*}
with probability at least $1-\exp(-\omega^2(\log^2 d))$.

\end{proof}

\subsection{Mixing property of Random Walp}\label{subsec:mixing_zt}
In this subsection we provide the mixing property of $\{\bz_t\}$. To this end, we first provide a lower bound of density ratios of conditional distributions of $\{\bz_t\}$.

\begin{lemma}\label{lem:lambda_lower_bound}
Let $m=4p^2\log d$, $\lambda = 1-\frac{1}{e}$. Let $f_{\bz_{m+1}|\bz_1}(\cdot |\bz_1)$ be the density function of $\bz_{m+1}$ conditional on $\bz_1$, and $\pi(\cdot)$ be the density of the stationary distribution $N(0,\Ib_p/p)$. Then for $\|\bx\|_2\leq 2\sqrt{T/p}$, we have
\begin{enumerate}[label=(\roman*)]
\item 
For all values of $\bz_1\neq 0$,
$$
f_{\bz_{m+1}|\bz_1}(x|\bz_1)\geq \lambda \pi(\bx);
$$
\item 
For $t_0>1$ and $\|\bz_{t_0}\|_2\leq 2\sqrt{T/p^3}$, 
$$
f_{\bz_{m+t_0}|\bz_{t_0}}(x|\bz_{t_0})\geq \lambda \pi(\bx).
$$
\end{enumerate}
\end{lemma}
\begin{proof}
We first prove (ii). By construction, the distribution of $\bz_{m+t_0}|\bz_{t_0}$ is 
$$
\bz_{m+t_0}|\bz_{t_0} \sim N(\alpha^{\frac{m}{2}}\bz_{t_0}, (1-\alpha^m)\Ib_p/p).
$$
Thus the ratio of two densities is
\begin{align*}
\frac{f_{\bz_{m+t_0}|\bz_{t_0}}(x|\bz_{t_0})}{\pi(\bx)} =&  \frac{\exp(-\frac{1}{2(1-\alpha^m)} (x-\alpha^{\frac{m}{2}}\bz_{t_0})^{\transpose} (\Ib_p/p)^{-1} (x-\alpha^{\frac{m}{2}}\bz_{t_0}) )}{(2\pi)^{p/2}|\text{det}((1-\alpha^m)\Ib_p/p)|^{1/2}}\bigg/\frac{\exp(-\frac{1}{2}x^{\transpose} (\Ib_p/p)^{-1}x)}{(2\pi)^{p/2}|\text{det}(\Ib_p/p)|^{1/2}}\\
=& \frac{1}{(1-\alpha^m)^{p/2}}\exp(-\frac{p}{2(1-\alpha^m)}\|x-\alpha^{\frac{m}{2}}\bz_{t_0}\|_2^2 + \frac{p}{2}\|\bx\|_2^2)\\
\geq & \exp(\frac{p}{2}\alpha^m -\frac{p}{2(1-\alpha^m)}\|x-\alpha^{\frac{m}{2}}\bz_{t_0}\|_2^2 + \frac{p}{2}\|\bx\|_2^2)\\
\geq & \exp(\frac{p}{2}\alpha^m -\frac{p}{2(1-\alpha^m)}\big(\|\bx\|_2+\alpha^{\frac{m}{2}}\|\bz_{t_0}\|\big)^2 + \frac{p}{2}\|\bx\|_2^2),
\end{align*}
where since $\|\bx\|_2,\|\bz_{t_0}\|_2\leq 2\sqrt{T/p^3}$,
\begin{align*}
&\frac{1}{1-\alpha^m}\big(\|\bx\|_2+\alpha^{\frac{m}{2}}\|\bz_{t_0}\|\big)^2 - \alpha^m - \|\bx\|_2^2\\
=& \Big(\frac{(1+\alpha^{\frac{m}{2}})^2}{1-\alpha^m}-1\Big) \frac{4T}{p^3} - \alpha^m
\leq  \frac{2\alpha^m +2\alpha^{\frac{m}{2}}}{1-\alpha^m} \frac{4T}{p^3}.
\end{align*}
Also $\alpha^m = (1-\frac{\log d}{p^2})^{4p^2\log d}\leq \exp(-4\log^2(d))\leq 0.01$ for $d\geq 3$. Thus $1-\alpha^m \geq 0.99$, $\alpha^m \leq \alpha^{\frac{m}{2}}/10$, hence
$$
\frac{2\alpha^m +2\alpha^{\frac{m}{2}}}{1-\alpha^m} \leq \frac{2.2\alpha^{\frac{m}{2}}}{0.99} \leq \frac{7}{4}\alpha^{\frac{m}{2}},
$$
and still by $\alpha^m \leq \exp(-4\log^2(d))$, we have
$$
\frac{f_{\bz_{m+t_0}|\bz_{t_0}}(x|\bz_{t_0})}{\pi(\bx)} \geq \exp(-\frac{p}{2}\cdot \frac{7\alpha^{\frac{m}{2}} T}{2p^3}) \geq \exp(-\frac{7 T}{4p^2}\exp(-4\log^2(d))) = \exp(-\frac{7T}{4p^2 d^{4\log d}})\geq \lambda,
$$
since $T=\Omega(p^5d^4)$.\\

For (i), similarly the distribution of $\bz_{m+1}|\bz_1$ is 
$$
\bz_{m+1}|\bz_1 \sim N(\alpha^{\frac{m}{2}}\frac{\bz_1}{\|\bz_1\|_2}, (1-\alpha^m)\Ib_p/p).
$$
Exactly the same procedure but with $\bz_{t_0}$ replaced by $\frac{\bz_1}{\|\bz_1\|_2}$ and naturally $\|\frac{\bz_1}{\|\bz_1\|_2}\|_2=1\leq 2\sqrt{T/p^3}$ yields the lower bound $\lambda$ for the ratio of densities.

\end{proof}

With this lower bound in hand, we provide a mixing property of $\{\bz_t\}$, which is of use in the generalized Hoeffding's inequality for code occurrences.

\begin{lemma}[Exponential decay of total variation distances]\label{lem:zt_mixing}
Let $m=4p^2\log d$, $\tilde{\lambda} = 1-\frac{2}{e}$. Then under Assumption~\ref{assump:kdT}, for the hidden Markov process specified in ~(\ref{eq:hidden_Markov_def}) and $T\geq m$, it holds that $\|f_{\bz_{T+1}|\bz_1}(\cdot |\bz_1)-\pi(\cdot) \|_{\var} \leq (1-\tilde\lambda)^{\lfloor T/m\rfloor}$.

\end{lemma}

\begin{proof}[Proof of Lemma~\ref{lem:zt_mixing}]
We construct a coupling for $\{\bz_t\}$ and $\{\by_t\}$, where the ditribution of $\{\bz_t\}_{t\geq1}$ is the same as the hidden Markov model in ~(\ref{eq:hidden_Markov_def}), and each $\by_t$ follows the stationary distribution $\pi(\cdot)$. Specifically, Assume $i=1,m+1,2m+1,\dots,\lfloor T/m\rfloor m+1$. 

We consider the following coupling.
\begin{enumerate}[label=(\roman*)]
\item 
If $i=1$ or $\|\bz_i\|_2$ and $\|\by_i\|_2\leq 2\sqrt{T/p^3}$, let $\bx_i,\bu_i$ independently from $\bx_i\sim \pi(\cdot)$ and $\bu_i\sim \mbox{Unif}[0,1]$,
    \begin{enumerate}[label=(\alph*)]
    \item 
    if $\|\bx_i\|_2\leq 2\sqrt{T/p^3}$ and $\bu_i\leq \lambda$, then set $\bz_{i+m}=\by_{i+m}=i$;
    \item 
    if $\|\bx_i\|_2>2\sqrt{T/p^3}$, or $\bu_i> \lambda$, set $\by_{i+m}=\bx_i$ and choose $\bz_{i+m}\sim h(\cdot|\bz_i)$, where
    \begin{equation}\label{eq:coupling_distribution}
    h(\bx|\bz_i) = \frac{I{\{\|\bx\|_2\leq 2\sqrt{T/p^3}\}}[f_{\bz_{i+m}|\bz_i}(\bx|\bz_i)-\lambda \pi(\bx)] + I{\{\|\bx\|_2>2\sqrt{T/p^3}\}}f_{\bz_{i+m}|\bz_i}(\bx|\bz_i)}{1 - \lambda \int_{\{\|\bx\|_2\leq 2\sqrt{T/p^3}\}}\pi(\bx)d\bx}.
    \end{equation}
    \end{enumerate}
\item 
If $i>1$, $\|\bz_i\|_2$ and $\|\by_i\|_2>2\sqrt{T/p^3}$, then $\bz_{i+m}$ and $\by_{i+m}$ are chosen independently from $f_{\bz_{i+m}|\bz_i}(\cdot|\bz_i)$ and $\pi(\cdot)$.
\end{enumerate}

Firstly, the distribution specfied in ~(\ref{eq:coupling_distribution}) is well-defined, because according to Lemma~\ref{lem:lambda_lower_bound}, $f_{\bz_{i+m}|\bz_i}(x|\bz_i)-\lambda\pi(\bx) \geq 0$ on $\{\|\bx\|_2\leq 2\sqrt{T/p^3}\}$, provided that $\|\bz_i\|_2\leq 2\sqrt{T/p^3}$.

Furthermore, the coupling is well-defined. Note that marginal distribution of $\by_t$ is $\by_t\sim \pi(\cdot)$ since $\by_t$ is always chosen according to $\pi(\cdot)$. For $\bz_{i+m}$, in (i), the distribution of $\bz_{i+m}$ is a mixture of $\pi(\cdot )I{\{\|\cdot \|_2\leq 2\sqrt{T/p^3}\}}$ with probability $\lambda \int_{\{\|\bx\|_2\leq 2\sqrt{T/p^3}\}}\pi(\bx)d\bx$ and $h(\cdot|\bz_1)$ as specified in ~(\ref{eq:coupling_distribution}), with probability $1-\lambda \int_{\{\|\bx\|_2\leq 2\sqrt{T/p^3}\}}\pi(\bx)d\bx$. This mixture is exactly $f_{\bz_{i+m}|\bz_i}(\cdot|\bz_i)$. So in both cases (i) and (ii), the marginal distribution is $\bz_{i+m}\sim f_{\bz_{i+m}|\bz_i}(\cdot|\bz_i)$, thus the distribution of $\{\bz_t\}$ is exactly the target.

Let $\tilde{T}$ be the coupling time of $\{\bz_t\}$ and $\{\by_t\}$, $\tilde{T}=\inf \{t: \bz_t=\by_t\}$. Let $M = \{m+1,2m+1,\dots,\lfloor T/m\rfloor m+1\}$. Then by coupling inequality, 
\begin{equation*}
\|_2f_{\bz_{t+1}|\bz_1}(\cdot |\bz_1) - \pi(\cdot)\|_{\var} \leq \P(\tilde{T}>T+1).
\label{eq:coupling_ineq}
\end{equation*}
By construction of the coupling, for any given nonzero vector $\bz_1$ we have
\begin{align}
\nonumber\P(\tilde{T}>T+1|\bz_1)&\leq \P\Big(\exists i\in M, \|\bz_i\|\text{ or }\|\by_i\|_2>2\sqrt{T/p^3} \Big|\bz_1 \Big) \\
& \qquad + \P\Big(\tilde{T}>T+1,\|\bz_i\|_2,\|\by_i\|_2\leq 2\sqrt{T/p^3},\forall i\in M\Big| \bz_1\Big).
\label{eq:coupling_time_bound}
\end{align}
By the bounds for chi-squared random variables stated in Lemma~\ref{lem:norm_x}, for $t>1$, $\P(\|\bz_t\|\geq 2\sqrt{T/p^3}|\bz_1]\leq \exp(-\frac{T}{8p^2})$ and $\P(\|\by_t\|\geq 2\sqrt{T/p^3}]\leq \exp(-\frac{T}{8p^2})$. Hence the first term in  (\ref{eq:coupling_time_bound}) is bounded with
\begin{equation}
\P\Big(\exists i\in M, \|\bz_i\|\text{ or }\|\by_i\|_2>2\sqrt{T/p^3} \Big|\bz_1 \Big) \leq 2\lfloor \frac{T}{m}\rfloor \exp(-\frac{T}{8p^2}).
\label{eq:cp_time_bound_1}
\end{equation}
The event in the second term in ~(\ref{eq:coupling_time_bound}) corresponds to the subcase (i)(b) with $\bu_i>\lambda$ for all steps $i\in M$ and $\|\bx_i\|_2=\|\by_{i+m}\|_2\leq 2\sqrt{T/p^3}$, which happens with probability at most $1-\lambda$ in each step. Therefore by independence of $\bu_i$, we have
\begin{equation}
\P\Big(\tilde{T}>T+1,\|\bz_i\|_2,\|\by_i\|_2\leq 2\sqrt{T/p^3},\forall i\in M\Big| \bz_1\Big)\leq \P\Big(\bigcap_{i=1}^{|M|}\{\bu_i >\lambda\}\Big)= (1-\lambda)^{|M|}\leq (1-\lambda)^{\lfloor T/m \rfloor}.
\label{eq:cp_time_bound_2}
\end{equation}
Comparing the two probabilities in ~(\ref{eq:cp_time_bound_1}) and (\ref{eq:cp_time_bound_2}), 
\begin{align*}
\frac{2\lfloor \frac{T}{m}\rfloor \exp(-\frac{T}{8p^2})}{(1-\lambda)^{\lfloor T/m \rfloor}} \leq &  \exp(  -\frac{T}{8p^2} + \log(\frac{2T}{m}) - (\frac{T}{m}-1)\log(1-\lambda)       )\\
\leq&  \exp(  -\frac{T}{8p^2} + \log(\frac{2T}{m}) - \frac{T}{2m}\log(1-\lambda)       )\\
\leq& \exp( -\frac{T}{8p^2}\log (e(1-\lambda)) + \log(\frac{T}{2p^2\log d}))\\
=&  \exp( -\frac{T}{8p^2}\log 2 + \log(\frac{T}{2p^2\log d})) = \exp(-\omega(\frac{T}{p^2}))=o(1).
\end{align*}
Here we utilize the fact that $\log(T/(2p^2\log d))=o(T/p^2)$ obtained from Assumption~\ref{assump:kdT}. This shows that the bound in  (\ref{eq:cp_time_bound_1}) is dominated by the one in  (\ref{eq:cp_time_bound_2}). So there exists some $\tilde{\lambda}\in (0,1)$, for which we may let $\tilde{\lambda} = 1-\frac{2}{e}$, such that for appropriately large $k,d,T$, 
\begin{equation}
\|_2f_{\bz_{t+1}|\bz_1}(\cdot |\bz_1) - \pi(\cdot)\|_{\var} \leq \P(\tilde{T}>T+1)\leq (1-\tilde{\lambda})^{\lfloor T/m\rfloor}.
\label{eq:tot_var_bound}
\end{equation}
\end{proof}

\subsection{Mixing Property of Joint Random Walks} \label{subsec:mixing_zt_zt+u}

For constant gap $u>0$, mixing properties of joint $(\bz_t,\bz_{t+u})$ can be obtained in a similar way. We first provide a counterpart of Lemma~\ref{lem:lambda_lower_bound} for the joint distributions. Let $f_{\bz_t,\bz_{t+u}}(\cdot,\cdot|\bz_1)$ denote the density of joint distribution of $(\bz_t,\bz_{t+u})$ given $\bz_1$. Let $\pi_u(\cdot,\cdot)$ be density of the stationary joint distribution of $(\bz_t,\bz_{t+u})$, which is the joint Gaussian distribution specified in Lemma~\ref{lem:stat_joint_distr_zt}. 

\begin{lemma}\label{lem:lambda_joint_lower_bound}
Let $m=4p^2\log d$, $\lambda = 1-\frac{1}{e}$. Then for $\|\bx\|_2\leq 2\sqrt{T/p}$, we have
\begin{enumerate}[label=(\roman*)]
\item 
For all values of $\bz_1\neq 0$,
$$
f_{\bz_{m+1},\bz_{m+1+u}}(\bx,\by|\bz_1)\geq \lambda \pi_u(\bx,\by);
$$
\item 
For $t_0>1$ and $\|\bz_{t_0}\|_2\leq 2\sqrt{T/p^3}$, 
$$
f_{\bz_{m+t_0},\bz_{m+t_0+u}|\bz_{t_0}}(\bx,\by|\bz_{t_0})\geq \lambda \pi_u(\bx,\by).
$$
\end{enumerate}
\end{lemma}

\begin{proof}[Proof of Lemma~\ref{lem:lambda_joint_lower_bound}]
Firstly, for any fixed $t$ and $u>0$, according to ~(\ref{eq:zt_to_zt+u_relationship}), $(\bz_t,\bz_{t+u})$ is actually a linear transformation of $(\bz_t,r)$ where $\bz_t,r$ are independent Gaussian random vectors. By change-of-variable formula for density functions, let $g_{t,t+u}(\cdot,\cdot|\bz_1)$ be the density of joint distribution of $\bz_t,r$ given $\bz_1$, we know 
\begin{equation*}
f_{t,t+u}(\bz_t,\bz_{t+u}|\bz_1) = f_{t,t+u}(\bz_t,\alpha^{u/2}\bz_t + \sqrt{1-\alpha^u}r|\bz_1) = (1-\alpha^u)^{-k/2} g_{t,t+u}(\bz_t,r|\bz_1)  = (1-\alpha^u)^{-k/2} g_{t}(\bz_t|\bz_1) g(r),
\end{equation*}
where $\bz_{t+u}=\alpha^{u/2}\bz_t + \sqrt{1-\alpha^u}r$, $g_t(\cdot|\bz_1)$ is the density of $\bz_t$ given $\bz_1$, and $g(r)$ is the density for $r\sim N(0,\Ib_p/p)$. And the last equality is due to the independence of $\bz_t$ and $r$. Similarly, the stationary distribution of $(\bz_t,\bz_{t+u})$ can be decomposed in the same way with $\bz_t,r$ independent $N(0,\Ib_p/p)$. Again by change-of-variable formula,
\begin{equation*}
\pi_u(\bz_t,\bz_{t+u}) = (1-\alpha^u)^{-k/2} \pi_{t,t+u}(\bz_t,r)  = (1-\alpha^u)^{-k/2}\pi(\bz_t) g(r),
\end{equation*}
where $\bz_{t+u}=\alpha^{u/2}\bz_t + \sqrt{1-\alpha^u}r$. Thus the ratio of densities are simply
\begin{equation*}
\frac{f_{t,t+u}(\bz_t,\bz_{t+u}|\bz_1)}{\pi_u(\bz_t,\bz_{t+u})} = \frac{g_{t}(\bz_t|\bz_1)}{\pi(\bz_t)},
\end{equation*}
which is exactly the same as in Lemma~\ref{lem:lambda_lower_bound}. Thus we have the same result as Lemma~\ref{lem:lambda_lower_bound}.
\end{proof}

The mixing property of $(\bz_t,\bz_{t+u})$ can be obtained from a similar coupling method, as follows.

\begin{lemma}[Exponential decay of total variation distances]\label{lem:mixing_property_zt_zt+u}
Let $m=4p^2\log d$, $\tilde{\lambda} = 1-\frac{2}{e}$. Then under Assumption~\ref{assump:kdT}, for the hidden Markov process specified in ~(\ref{eq:hidden_Markov_def}), for any $T\geq m$, it holds that $\|f_{\bz_{T+1},\bz_{T+u+1}|\bz_1}(\cdot,\cdot |\bz_1)-\pi_u(\cdot,\cdot) \|_{\var} \leq (1-\tilde\lambda)^{\lfloor T/m\rfloor}$.
\end{lemma}

\begin{proof}[Proof of Lemma~\ref{lem:mixing_property_zt_zt+u}]

For constant gap $u>0$, a similar coupling for joint $(\bz_{i},\bz_{i+u})$ and $(\by_i,\by_{i+u})$ can be constructed for $i=1,m+1,\dots,\lfloor \frac{T}{m}\rfloor$, where $(\bz_{i},\bz_{i+u})$ follows the distribution in our Markov process and $(\by_i,\by_{i+u})$ follows the stationary distribution for $(\bz_i,\bz_{i+u})$, which is multivariate Gaussian.

Consider the following coupling of $(\bz_{i},\bz_{i+u})$ and $(\by_i,\by_{i+u})$ for $i\in M$.

\begin{enumerate}[label=(\roman*)]
\item 
If $i=1$ or $\|\bz_i\|_2$, $\|\by_i\|_2\leq 2\sqrt{T/p^3}$, choose $(\bx_i,\bx_{i+u})$ and $\bv_i$ independently, where $(\bx_i,\bx_{i+u})\sim \pi_u(\cdot,\cdot)$ and $\bv_i \sim \mbox{Unif}[0,1]$;
    \begin{enumerate}[label=(\alph*)]
    \item 
    if $\|\bx_i\|_2\leq 2\sqrt{T/p^3}$ and $\bv_i\leq \lambda$, then set $(\bz_{i+m},\bz_{i+m+u})=(\by_{i+m},\by_{i+m+u})=(\bx_i,\bx_{i+u})$;
    \item 
    if $\|\bx_i\|_2>2\sqrt{T/p^3}$ or $\bv_i>\lambda$, set $(\by_{i+m},\by_{i+m+u})=(\bx_i,\bx_{i+u})$ and independently choose $(\bz_{i+m},\bz_{i+m+u})\sim h(\bx,\by|\bz_i)$, where
    \begin{align}\label{eq:coupling_bivariate_distribution}
  \nonumber  h(\bx,\by|\bz_i) = & \frac{I{\{\|\bx\|_2\leq 2\sqrt{T/p^3}\}}[f_{\bz_{i+m},\bz_{i+m+u}|\bz_i}(\bx,\by|\bz_i)-\lambda \pi_u(\bx)]}{1 - \lambda \int_{\{\|\bx\|_2\leq 2\sqrt{T/p^3}\}}\pi_u(\bx,\by)d\bx d\by}\\
    & + \frac{I{\{\|\bx\|_2>2\sqrt{T/p^3}\}}f_{\bz_{i+m},\bz_{i+m+u}|\bz_i}(\bx,\by|\bz_i)}{1 - \lambda \int_{\{\|\bx\|_2\leq 2\sqrt{T/p^3}\}}\pi_u(\bx,\by)d\bx d\by}
    \end{align}
    \end{enumerate}
\item 
If $i>1$, $\|\bz_i\|_2$ and $\|\by_i\|_2>2\sqrt{T/p^3}$, then choose $(\bz_{i+m},\bz_{i+m+u}\sim f_{\bz_{i+m},\bz_{i+m+u}|\bz_i}(\cdot,\cdot|\bz_i)$ and $(\by_{i+m},\by_{i+m+u})\sim \pi_u(\cdot,\cdot)$ independently.
\end{enumerate}

Firstly, the distribution in  (\ref{eq:coupling_bivariate_distribution}) is well-defined, since according to Lemma~\ref{lem:lambda_joint_lower_bound}, 
\[
f_{\bz_{i+m},\bz_{i+m+u}|\bz_i}(\bx,\by|\bz_i)-\lambda \pi_u(\bx)\geq 0
\]
 on $\{\|\bx\|_2\leq 2\sqrt{T/p^3}\}$ given $\|\bz_i\|_2\leq 2\sqrt{T/p^3}$.

Furthermore, the coupling is well-defined. Note that the marginal distribution of $(\by_i,\by_{i+u})$ is $(\by_i,\by_{i+u})\sim \pi_u(\cdot, \cdot)$ since it is always drawn from that. The marginal distribution of $\{(\bz_i,\bz_{i+u}\}_{i\in M}$ is the same as in ~(\ref{eq:hidden_Markov_def}), because the conditional distribution $(\bz_{i+m},\bz_{i+m+u})|\bz_i$ is the same as in the hidden Markov model in ~(\ref{eq:hidden_Markov_def}). To see this, note that simialr to the coupling in the proof of Lemma~\ref{lem:zt_mixing}, in case (i) the marginal distribution of $(\bz_{i+m},\bz_{i+m+u})$ is a mixture of $\pi_u(\bx,\by)I{\{\|\bx\|_2\leq 2\sqrt{T/p^3}\}}$ and $h(\bx,\by|\bz_i)$, with weights $\lambda \int_{\{\|\bx\|_2\leq 2\sqrt{T/p^3}\}}\pi_u(\bx,\by)d\bx d\by$ and $1-\lambda \int_{\{\|\bx\|_2\leq 2\sqrt{T/p^3}\}}\pi_u(\bx,\by)d\bx d\by$, respectively. Thus both in case (i) and (ii), $(\bz_{i+m},\bz_{i+m+u})$ follows $f_{\bz_{i+m},\bz_{i+m+u}|\bz_i}(\bx,\by|\bz_i)$, and so is their marginal joint distributions. 

We again employ coupling inquality to bound the total variation distance of $(\bz_i,\bz_{i+u})|\bz_1$ and $\pi_u(\cdot,\cdot)$. Let $\tilde{T}$ be the coupling time of $\{(\bz_t,\bz_{t+u})\}$ and $\{(\by_t,y_{t+u})\}$, $\tilde{T}=\inf\{t:(\bz_{t},\bz_{t+u})=(\by_t,y_{t+u})$. Let $M=\{m+1,2m+1,\dots,\lfloor T/m\rfloor m+1\}$. Then by coupling inequality,
\begin{equation*}
\|f_{\bz_{T+1},\bz_{T+1+u}|\bz_1}(\cdot,\cdot|\bz_1) - \pi_u(\cdot,\cdot) \|_{\var} \leq \P(\tilde{T}>T+1|\bz_1).
\end{equation*}
We have the decomposition as
\begin{equation*}\label{eq:biv_coupling_time_decomp}
\P(\tilde{T}>T+1|\bz_1) \leq \P\big(\exists i\in M, \|\bz_i\|\mbox{ or }\|\by_i\|_2>2\sqrt{T/p^3}|\bz_1\big) + \P\big( \tilde{T}>T+1, \|\bz_i\|_2,\|\by_i\|_2\leq 2\sqrt{T/p^3}\leq 2\sqrt{T/p^3}\big).
\end{equation*}
Here the first term is similarly bounded by 
\begin{equation*}\label{eq:biv_coupling_time_decomp1}
\P\big(\exists i\in M, \|\bz_i\|\mbox{ or }\|\by_i\|_2>2\sqrt{T/p^3}|\bz_1\big) \leq 2\lfloor \frac{T}{m}\rfloor \exp(-\frac{T}{8p^2}),
\end{equation*}
and the second term implies that $\|\bx_i\|_2\leq 2\sqrt{T/p^3}$ but $\bv_i >\lambda$ for all steps $i\in M$. Thus
\begin{equation*}\label{eq:biv_coupling_time_decomp2}
\P\big( \tilde{T}>T+1, \|\bz_i\|_2,\|\by_i\|_2\leq 2\sqrt{T/p^3}\leq 2\sqrt{T/p^3}\big) \leq (1-\lambda)^{|M|}\leq (1-\lambda)^{\lfloor T/m\rfloor}.
\end{equation*}
The bound in ~(\ref{eq:biv_coupling_time_decomp1}) is dominated by ~(\ref{eq:biv_coupling_time_decomp2}), as stated in the proof of Lemma~\ref{lem:zt_mixing}. Hence with same arguments we have an upper bound $2(1-\lambda)^{\lfloor T/m\rfloor}\leq (1-\tilde{\lambda})^{\lfloor T/m\rfloor}$ for $T/m\geq 1$.

\end{proof}

\begin{lemma} \label{lem:z_t_r_t1}
Under Assumption~\ref{assump:parameter_space}, if $\bz_t, r_{t+1} \sim N(0,\Ib_p/p)$ and they are independent, then 
$$\PP\left[ \frac{\|\bz_t\|_2}{\|r_{t+1}\|_2} \leq \frac{1}{2}\right] = \exp(-\omega(\log^2 d)).$$
\end{lemma}

\begin{proof}[Proof of Lemma~\ref{lem:z_t_r_t1}.]

$k \|\bz_t\|_2^2 \sim \chi^2_p$, and by Lemma~\ref{lem:concentra_chi}, we have
\begin{equation}
\begin{split}
    &\PP\left[k\|\bz_t\|_2^2 \geq 2k \right] \leq \exp\left(- \frac{p}{4(1+c_1')} \right), \\
    &\PP\left[k\|\bz_t\|_2^2 \leq k/2 \right] \leq \exp\left(- \frac{p}{16(1-c_2')} \right),
\end{split}
\end{equation}
where $c_1' \geq 1$ and $c_2' \leq 1/2$. 

Let $c_1' = 1$ and $c_2' = 0$, then
$$\PP\left[\|\bz_t\|_2^2 \geq 2 \right] = \exp(-O(p)),$$
$$\PP\left[\|\bz_t\|_2^2 \leq 1/2 \right] = \exp(-O(p)).$$

This holds true for $r_{t+1}$ as well.

So we have
\begin{equation}
\begin{split}
    \PP\left[ \frac{\|\bz_t\|_2}{\|r_{t+1}\|_2} \leq \frac{1}{2}\right] 
    &\leq \PP\left[\|\bz_t\|_2^2 \leq 1/2 \bigcup \|r_{t+1}\|_2^2 \geq 2 \right] \\
    &\leq \PP\left[\|\bz_t\|_2^2 \leq 1/2\right]+\PP\left[\|r_{t+1}\|_2^2 \geq 2 \right] \\
    &= \exp(-\Omega(p)) \\
    & = \exp(-\omega(\log^2 d)).
\end{split}
\end{equation}
\end{proof}

\begin{lemma}\label{lem:norm_x}
Under Assumption~\ref{assump:parameter_space}, the $\ell_2$ norm of the following two vectors can be bounded by $2\sqrt{T/p^3}$ with high probability.
\begin{enumerate}
    \item For $z \sim N(0,\Ib_p/p)$,
     $\PP[\|z\|_2\geq 2\sqrt{T/p^3} ] \leq \exp \left( - \frac{T}{8 p^2} \right).$
     \item For $\bz_t|\bz_1$ ($t \geq 2$) in $\{\bz_t\}_{t\geq 1}$ defined in  (\ref{eq:hidden_Markov_def}), we have
     $\PP\left[\left.\|\bz_t \|_2\geq 2\sqrt{T/p^3}  \right|\bz_1 \right] 
    \leq \exp \left( - \frac{T}{8 p^2} \right).$
\end{enumerate}
\end{lemma}
\begin{proof}[Proof of Lemma~\ref{lem:norm_x}.]

For $z \sim N(0, \Ib_p/p)$, $k \|z\|_2^2  \sim \chi^2_{p}$, then by Lemma~\ref{lem:concentra_chi}, $\|z\|_2^2$ can be bounded with high probability as
\begin{equation} \label{eqn:bound_unifrom_z}
\begin{split}
    &\PP[\|z\|_2^2 \geq 1 + T/p^3] 
    \leq \exp \left( - \frac{T^2/p^5}{4(1+T/p^3)} \right) \leq \exp \left( - \frac{T}{8 p^2} \right) \\
    \Longrightarrow &\PP [\|z\|_2\geq 2\sqrt{T/p^3}] \leq \exp \left( - \frac{T}{8 p^2} \right),
\end{split}
\end{equation}
where $T \gg p^3$ under Assumption~\ref{assump:parameter_space}.
\end{proof}

\section{Properties of code vectors} \label{sec:property_word_vector}

In this section, we prove several properties of code vectors, which are frequently used in Section~\ref{sec:partition} and Section~\ref{sec:stat_PMI_to_cov}. Specifically, we show the boundedness of variance of code vectors in Lemma~\ref{lem:upper_bound_var}, and then prove the w.h.p. boundedness of code vectors in Corollary~\ref{cor:length_bound}.


\begin{lemma}[Bounds of variance of code vectors]\label{lem:upper_bound_var}
Suppose Assumptions \ref{assump:parameter_space} and \ref{assump:block_bound} hold, then the variance for any code $w$, denoted as $\sigma_{w}^2$, satisfies that
$$\frac{1}{\rho} \leq \min_w \sigma_{w}^2 \leq  \max_w \sigma_{w}^2 \leq \rho+c_0.$$
Hereafter we denote $\sigma_{\min}^2=1/\rho$, $\sigma_{\max}^2 = \rho+c_0$ as the lower and upper bounds of variances.
\end{lemma}

\begin{proof}[Proof of Lemma~\ref{lem:upper_bound_var}]
For the $i$-th code $w$, $i\in [d]$, $w\in V$, we do not distinguish between saying `word $i$' or `word $w$', and denote $\sigma_i^2:=\sigma_w^2$ in this proof. From the decomposition $\bSigma = \Ab \Qb \Ab^{\transpose} + \bGamma$, we see that $\sigma_{i}^2 = \Qb_{kk}$ if $i\in G_k$ and $|G_k|=1$, and $\sigma_{i}^2 = \Qb_{kk}+\bGamma_{ii}$ if $i \in G_k$ and $|G_k|>1$. Therefore
$$
\min_{k} \Qb_{kk} \leq \min_{i} (\sigma_i^2) \leq \max_i (\sigma_i^2) \leq \max_{k}\Qb_{kk} + \max_{i} \bGamma_{ii}.
$$
By Assumption \ref{assump:parameter_space}, we know 
$$
\min_k \Qb_{kk} \geq \min_{\|\bx\|_2=1} x^{\transpose} \Qb x = \lambda_{\min}(\Qb) = \|\Cb\|_{L_2}^{-1} \geq 1/\rho,
$$
where $\lambda_{\min}(\cdot)$ denotes the smallest eigenvalue of a matrix. On the other hand, as $\Delta(\Qb) \ge C_0 \sqrt{\log d/p}$,  we know 
$$
\max_{k}\Qb_{kk} + \max_i \bGamma_{ii} \leq c_0 + \lambda_{\max}(\Qb) = c_0 + \lambda_{\min}(\Cb)^{-1} \leq c_0 + \rho,
$$
where $\lambda_{\max}(\cdot)$ denotes the largest eigenvalue of a matrix.

\end{proof}

\begin{corollary}[W.h.p. boundedness of code vectors]\label{cor:length_bound} Suppose Assumptions~\ref{assump:parameter_space}, \ref{assump:block_bound} and~\ref{assump:kdT} hold. Then with probability at least $1-\exp(\omega(\log^2 d))$, the code vectors $\Vb_w$ generated from our Gaussian graphical model satisfy that
$$  \frac{\sigma_{\min}}{\sqrt{2}}\sqrt{p} \leq \min_w \|\Vb_w\|_2 \leq \max_w \|\Vb_w\|_2 \leq \sigma_{\max}\sqrt{2p}.$$
\end{corollary}

\begin{proof}[Proof of Corollary~\ref{cor:length_bound}.]

For code $w$, $[\Vb_w]_\ell$ i.i.d. $\sim N(0, \sigma^2_w)$ for $1 \leq \ell \leq p$. Then $\sum_{\ell = 1}^p [\Vb_w]_\ell^2/\sigma_w^2 = \|\Vb_w\|_2^2/\sigma_w^2 \sim \chi^2_p$. 

We first analyze $\PP\left(\max_w \|\Vb_w\|_2 \geq 2\sigma\sqrt{p}\right)$. By Lemma~\ref{lem:concentra_chi}, we have the tail probability bound
\begin{equation*}
    \P\big(\|\Vb_w\|_2^2/\sigma_w^2 \geq 2k\big) \leq \exp\left(-\frac{p}{4(1+\bc'_1)} \right) = \exp\left(-k/8 \right) = \exp(-\omega(p)),
\end{equation*}
where $\bc'_1=1 \geq \sqrt{p}/\sqrt{p}$ satisfies the condition in Lemma~\ref{lem:concentra_chi}. By Lemma~\ref{lem:upper_bound_var}, $\sigma^2_w \leq \sigma_{\max}^2 $, so
\begin{equation*}
    \P\big(\|\Vb_w\|_2^2 \geq 2\sigma_{\max}^2 k\big) =  \P\big(\|\Vb_w\|_2 \geq 2 \sigma_{\max} \sqrt{2p} \big) =  \exp\left(-\omega(p) \right).
\end{equation*}
With union bound, we can bound the $\ell_2$ norm of all code vectors with high probability as
\begin{equation*}
    \P\big(\max_w \|\Vb_w\|_2 \geq \sigma_{\max}\sqrt{2p} \big)= d  \exp\left(-\omega(p) \right) = \exp(-\omega(p)) = \exp(-\omega(\log^2 d)),
\end{equation*}
since $\log d = o(\sqrt{p})$ by Assumption~\ref{assump:kdT}.

Similar procedure can be used to analyze $\PP\left(\min_w \|\Vb_w\|_2 \leq \frac{\sigma_{\min}}{\sqrt{2}}\cdot \sqrt{p}\right)$. By Lemma~\ref{lem:concentra_chi}, we have
\begin{equation*}
    \P\big(\|\Vb_w\|_2^2/\sigma_w^2 \leq k/2\big) \leq \exp\left(-\frac{p}{16(1-c'_2)} \right),
\end{equation*}

where $c_2' \leq \frac{\sqrt{p}/2}{\sqrt{p} }= \frac{1}{2}$. Let $c_2' = 0$ and by $\sigma_w^2 \geq \min_w \sigma_w^2 \geq \sigma_{\min}^2$,

\begin{equation*}
    \P\big(\|\Vb_w\|_2^2 \leq \sigma_{\min}^2 k/2\big) = \P\big(\|\Vb_w\|_2 \leq \frac{\sigma_{\min}}{\sqrt{2}}\cdot \sqrt{p}\big) \leq \exp\left(-\frac{p}{16} \right)= \exp\left(-\omega(p) \right).
\end{equation*}

With union bound, we have
\begin{equation*}
    \P\left(\min_w \|\Vb_w\|_2 \leq \frac{\sigma_{\min}}{\sqrt{2}}\cdot \sqrt{p}\right) = d \exp\left(-\omega(p) \right)= \exp(-\omega(p)) = \exp(-\omega(\log^2 d)),
\end{equation*}
where $\log d = o(\sqrt{p})$ by Assumption~\ref{assump:kdT}.
\end{proof}

\begin{lemma}[Concentration of Chi-square random variable] \label{lem:concentra_chi}
$Q \sim \chi^2_{p}$ concentrates around its mean with high probability. Specifically,
$$\PP[Q \geq p + t_1 \sqrt{p}] \leq \exp\left(-\frac{t_1^2}{4(1+ c_1')}\right),$$
where $t_1 > 0$ and constant $c_1'$ satisfies $c_1' \geq t_1/\sqrt{p}$.

And
$$\PP[Q \leq p - t_2\sqrt{p}] \leq \exp\left(-\frac{t_2^2}{4(1-c_2')}\right),$$
where $0 < t_2 < \sqrt{p}$ and $c_2' \leq t_2/\sqrt{p}$ is a constant.
\end{lemma}

\begin{proof}[Proof of Lemma~\ref{lem:concentra_chi}.]

We first bound $\PP(Q \geq p + t_1\sqrt{p})$. By Chernoff bound, for some $s < 1/2$, we have
\begin{equation*}
    \P\big(Q \geq p + t_1\sqrt{p}\big) \leq \frac{\exp(sQ)}{\exp(s(p+t_1\sqrt{p}))} = \frac{(1-2s)^{-p/2}}{\exp(s(p+t_1\sqrt{p}))}.
\end{equation*}

Let $s = \frac{t_1}{2(\sqrt{p}+t_1)} < \frac{1}{2}$, we have
\begin{equation*}
    \P\big(Q \geq p + t_1\sqrt{p}\big) \leq  \frac{(1-\frac{t_1}{\sqrt{p}+t_1})^{-p/2}}{\exp(t_1\sqrt{p}/2)}.
\end{equation*}

Let $g_1(t_1):= \ln \left(\frac{(1-\frac{t_1}{\sqrt{p}+t_1})^{-p/2}}{\exp(t_1\sqrt{p}/2)} \right)$, then we will show that $g_1(t_1) \leq -\frac{t_1^2}{4(1+c_1')}$ for some constant $c_1' \geq t_1/\sqrt{p}$. Define $G_1(t_1)$ as $G_1(t_1):= g_1(t_1) +\frac{t_1^2}{4(1+c_1')}$, then
\begin{equation*}
    \frac{dG_1(t_1)}{dt_1} = \frac{p}{2(\sqrt{p}+t_1)} - \frac{\sqrt{p}}{2} + \frac{t_1}{2(1+c_1')} = -\frac{t_1}{2(1+t_1/\sqrt{p})} + \frac{t_1}{2(1+c_1')} \leq 0.
\end{equation*}

Note that $G_1(\cdot)$ is a continuous function in $[0,+\infty]$, so we have $G_1(t_1) \leq G_1(0) = 0$, which gives us 
$$g_1(t_1) \leq -\frac{t_1^2}{4(1+c_1')}.$$

And thus 
\begin{equation*}
    \PP(Q \geq p + t_1\sqrt{p}) \leq \exp(g_1(t_1)) \leq \exp\left(-\frac{t_1^2}{4(1+c_1')}\right),
\end{equation*}
where $c_1' \geq t_1/\sqrt{p}$ is a constant.

We next bound $\P(Q \leq p - t_2\sqrt{p})$ using similar method.

\begin{equation*}
     \PP\big(Q \leq p - t_2\sqrt{p} \big) 
     \leq \frac{\EE\left[\exp \{-sQ\} \right]}{\exp \{-s (p - t_2\sqrt{p})\}} 
     = \frac{(1+2s)^{-p/2}}{\exp \{-s(p - t_2 \sqrt{p})\}},
\end{equation*}
where $s > -1/2$ and $0<t_2<\sqrt{p}$. Let $s = \frac{t_2}{2(\sqrt{p}-t_2)} > -\frac{1}{2}$, then we have

\begin{equation*}
\begin{split}
     \P\big(Q \leq p - t_2\sqrt{p} \big)
     \leq \frac{(1+\frac{t_2}{\sqrt{p}-t_2})^{-p/2}}{\exp \{-t_2\sqrt{p}/2\}}.
\end{split}
\end{equation*}

Let $g_2(t_2):= \ln \left(\frac{(1-\frac{t_2}{\sqrt{p}-t_2})^{-p/2}}{\exp(-t_2\sqrt{p}/2)} \right)$, for some constant $ 0\leq c_2' \leq \frac{t_2}{\sqrt{p}}$, and $G_2(t_2):= g_2(t_2) +\frac{t_2^2}{4(1-c_2')}$. Then
\begin{equation*}
    \frac{dG_2(t_2)}{dt_2} = \frac{p}{2(\sqrt{p}-t_2)} - \frac{\sqrt{p}}{2} + \frac{t_2}{2(1-c_2')} = -\frac{t_2}{2(1-t_2/\sqrt{p})} + \frac{t_2}{2(1-c_2')} \leq 0.
\end{equation*}

Here $G_2(\cdot)$ is a continuous function in $[0,+\infty]$, so we have $G_2(t_2) \leq G_2(0) = 0$, and thus
$$g_2(t_2) \leq -\frac{t_2^2}{4(1-c_2')}.$$

Lastly, we have
\begin{equation*}
    \PP(Q \leq p - t_2\sqrt{p}) \leq \exp(g_2(t_2)) \leq \exp\left(-\frac{t_2^2}{4(1-c_2')}\right),
\end{equation*}
where $0 \leq c_2' \leq t_2/\sqrt{p}$ is a constant.

\end{proof}

\section{Concentration of Partition Function} \label{sec:partition}

Based on results of properties of code vectors developed in Section~\ref{sec:property_word_vector}, we prove the main result in this section, Lemma \ref{lem:part_fct}. It mainly says that the partition function $Z(\Vb,\bc) = \sum_w \exp(\langle \Vb_w,\bc\rangle)$ is close to its expectation with high probability, where the randomness comes from both code vectors $\Vb_w$ and discourse vectors $c$.

Before stating our concentration result, we first state the following lemma about concentration of functions of i.i.d. $N(0,1)$ variables, which plays a key role in our proof of Lemma~\ref{lem:part_fct}. It is from Corollary 2.5 in \cite{SPS_1999__33__120_0} applied to local gradient operator $\Gamma(f)=\|\nabla f\|_2^2$.

\begin{lemma}[Gaussian concentration inequality]\label{lem:gaussian_concentration}
Let $X_1,\dots,X_n$ be independent Gaussian random variables with zero mean and unit variance. Then 
\begin{equation*}
\P(f(X_1,\dots,X_n) - \E[f(X_1,\dots,X_n)]\geq t) \leq \exp(-\frac{t^2}{2\sigma^2}),
\end{equation*}
holds for all $t\geq 0$, where $\sigma^2 = \|_2\|\nabla f\|_2^2\|_{\infty}$.
\end{lemma}

\begin{proof}
See Corollary 2.5 in \cite{SPS_1999__33__120_0}, or Theorem 3.25 in \cite {van2016apc}.
\end{proof}


Below is the main result of this section.

\begin{lemma}\label{lem:part_fct}
Suppose the code vectors $\Vb_w$ are generated from the specified Gaussian graphical model, and $\bv_w$ are their realizations. Then under Assumptions \ref{assump:parameter_space}, \ref{assump:block_bound} and ~\ref{assump:kdT}, for some large constant $\tau'$, with probability at least $1-\exp(-\omega(\log^2 d))$, (with respect to the randomness in generating code vectors), the realized partition function $Z(V,\bc) = \sum_{w}\exp(\langle \bv_w,\bc\rangle)$ satisfies
\begin{equation}
\P_{\bc\sim \mathcal{C}} \Big( (1-\epsilon_z)Z\leq Z(V,\bc) \leq (1+\epsilon_z) Z \Big| \Vb_w \Big)=1-\exp(-\omega(\log^2 d)).
\label{eq:bound_partition_fct}
\end{equation}
where $Z=\sum_{i=1}^d\exp(\sigma_{i,i}^0/2)$, $\epsilon_z = \sqrt{\frac{\log d}{p}}$.
If we count in the randomness of $\Vb_w$, we still have
\begin{equation}
\P_{\bc\sim \mathcal{C},\Vb_w} \Big( (1-\epsilon_z)Z\leq Z(V,\bc) \leq (1+\epsilon_z) Z \Big) \leq \exp(-\omega(\log^2 d)).
\label{eq:bound_partition_fct_whole}
\end{equation}
\end{lemma}


Though Lemma~\ref{lem:part_fct} is similar to Lemma 2.1 in \cite{arora2016latent}, their proof is based on a Bayesian prior of code vectors, which assumes code vectors are i.i.d. produced by $v = s \cdot \hat v$, where $s$ is a scalar r.v. and $\hat v$ comes from a spherical Gaussian distribution. In our lemma, however, this Bayesian prior assumption is relaxed. And instead we assume the code vectors are generated from our Gaussian graphical model in Section~\ref{sub:Gaussian_graph}. This proof is essentially harder than that in \cite{arora2016latent}, because code vectors are correlated rather than i.i.d..

\begin{proof}[Proof of Lemma~\ref{lem:part_fct}]
When both discourse variable $c$ and code vectors are random, the partition function $Z_c = Z_c(\Vb_w, c)$ is a function of both $\Vb_w$ and $c$, which makes the analysis complicated. We prove the lemma in two steps. Firstly we analyze the concentration of $Z(\Vb_w,\bc)$ to $Z$ with $c$ fixed. Then we switch the randomness in $\Vb_w$ and $c$ to obtain the final results. In the first step, we first truncate the norm of vector $\langle \Vb_i,c\rangle$ and prove concentration for truncated version. The other parts have sufficiently small probabilities.

\paragraph{Analysis of $Z(\Vb_w,\bc)$ with $c$ fixed.} Firstly, as is shown in Lemma~\ref{lem:e_var_Z_c}, for fixed $c$, the vector $\mathbf{Y}=(\langle \Vb_1,c\rangle,\cdots,\langle \Vb_d,\bc\rangle)^{\transpose}$ follows a multivariate Gaussian distribution with mean zero and covariance matrix $\bSigma$. Let $\bSigma=LL^{\transpose}$ be the Cholesky decomposition of $\bSigma$, where $L$ is a lower triangular matrix with positive diagonal entries. Let $\mathbf{Y}=L\mathbf{X}$ or $X=L^{-1}Y$, we know $\mathbf{X}\sim N(0,I_d)$.

Denote the event $F=\{\|Y_i\|_2\leq (\gamma/2) \log d ,\forall i=1,\dots,d\}$, where $\gamma>0$ is the constant such that $k \log^2 d= O(d^{1-\gamma})$ in Assumption~\ref{assump:kdT}. Since $Y_i\sim N(0,\Sigma_{ii})$, we have 
\begin{equation*}
\P\Big(\|Y_i\|_2\leq \log d/4,\forall i=1,\dots,d\Big) \leq \sum_{i=1}^d \P(\|Y_i\|_2\leq (\gamma/2)\log d) \leq \sum_{i=1}^d \exp(-\frac{\gamma^2 \log ^2 d}{8\Sigma_{ii}}) = \exp(-\omega(\log^2 d)).
\end{equation*}
Also note that $F=\{\|L\mathbf{X}\|_{\infty}\leq (\gamma/2) \log d \}$ in terms of $\mathbf{X}$. Now define our target
$$
f(\mathbf{X})=f(X_1,\dots,X_d)= \sum_{i=1}^d \exp(Y_i)I\{F\} = \sum_{i=1}^d \exp(\sum_{j=1}^i L_{ij}X_j)I\{F\}.
$$
Since the discontinuity points (also non-differentiability points) of $I\{F\}(\mathbf{X})$ is of measure zero under both Lebesgue measure and the probability measure of $\mathbf{X}$, we have the gradient of $f$ (almost surely)
\begin{equation*}
(\nabla f)_i(X_1,\dots,X_d) = \sum_{j=i+1}^d L_{ji} e^{\sum_{p=1}^j X_p}I\{F\} = \sum_{j=i+1}^d L_{ji} e^{Y_j}I\{F\}.
\end{equation*}
And thus $\nabla f = L^{\transpose} \tilde{\mathbf{Y}}I\{F\}$ (almost surely). Hence
\begin{equation*}
\|\nabla f\|_2^2 = \tilde{\mathbf{Y}}^{\transpose} LL^{\transpose} \tilde{\mathbf{Y}}I\{F\} =  \tilde{\mathbf{Y}}^{\transpose} \bSigma \tilde{\mathbf{Y}}I\{F\} \leq \|\bSigma\|_2 \|\tilde{\mathbf{Y}}\|_2^2I\{F\},
\end{equation*}
where 
\begin{equation*}
\|\tilde{\mathbf{Y}}\|_2^2I\{F\} = \sum_{i=1}^d e^{2Y_i}I\{F\} \leq d^{1+\gamma}.
\end{equation*}
We now proceed to bound $\|\bSigma\|_2$ from our model assumptions. By the decomposition $\bSigma=\Ab \Qb \Ab^{\transpose} + \bGamma$, we have $\|\bSigma\|_2 \leq \|\Ab \Qb \Ab^{\transpose}\|_2 + \|\bGamma\|_2$, hwere $\|\bGamma\|_2 = \max_i \bGamma_{ii}\leq c_0$ since it's diagonal and satisfies Assumption \ref{assump:block_bound}. The $L_2$-norm of $\Ab \Qb \Ab^{\transpose}$ is
$$
\|\Ab \Qb \Ab^{\transpose}\|_2 = \max_{\|\bx\|_2=1} x^{\transpose} \Ab \Qb \Ab^{\transpose} x \leq  \max_{\|\bx\|_2=1}\|\Ab^{\transpose} x\|_2^2 \|\Qb\|_2,
$$
where for $\|\bx\|_2=1$, 
$$
\|\Ab^{\transpose} x\|_2^2 = \sum_{j=1}^K \big(\sum_{i \in G_j} \bx_i\big)^2 \leq \sum_{j=1}^K |G_j| \sum_{i\in G_j} \bx_i^2 \leq \max_p |G_k| \|\bx\|_2^2 = \max_{j} |G_j|.
$$
Thus we have the bound of $\|\bSigma\|_2$ as 
$$
\|\bSigma\|_2 \leq \|\Qb\|_2 \max_j |G_j| + c_0 \leq \rho\max_j |G_j| + c_0,
$$
where $\rho,c_0$ are constants from Assumptions \ref{assump:parameter_space} and \ref{assump:block_bound}. Therefore by Lemma~\ref{lem:gaussian_concentration}, with $\sigma^2 = d^{1+\gamma}$, for $t\geq 0$ we have
\begin{equation*}
\P\Big( f(X_1,\dots,X_d) - \E[f(X_1,\dots,X_d)] \geq t \Big) \leq \exp(-\frac{t^2}{2d^{1+\gamma}}).
\end{equation*}
With the same argument applied to $-f$ and using union bound, we know that for any $\delta\geq 0$,
\begin{equation}\label{eq:bound_partition_fct_delta}
\begin{split}
&\P\Big( \Big|\frac{Z(\Vb_w,\bc)I\{F\} - \E[Z(\Vb_w,\bc)I\{F\}]}{\E[Z(\Vb_w,\bc)]}\Big| \geq \delta \Big)\\
 \leq& 2\exp(-\frac{\delta^2 (\E[Z(\Vb_w,\bc)])^2 }{2d^{1+\gamma}}) 
 \leq 2\exp(-\frac{d^{1-\gamma}\delta^2}{2(\rho\max_j |G_j| + c_0)}) = \exp(-\omega(\log^2 d)),
\end{split}
\end{equation}
since $d^{1-\gamma}/(k\max_{j}|G_j|)=\omega(\log^2 d)$ by Assumption \ref{assump:kdT}, and $\E[Z(\Vb_w,\bc)]\geq d$ proved in Lemma~\ref{lem:e_var_Z_c}. Moreover, by Cauchy-Schartz inequality we have
\begin{equation*}
\begin{split}
0\leq& \E[Z(\Vb_w,\bc)] - \E[Z(\Vb_w,\bc)I\{F\}] \\
=&\E[Z(\Vb_w,\bc)I\{F^c\}] \\
\leq& (\E[Z(\Vb_w,\bc)^2])^{1/2} (\E[I\{F^c\}])^{1/2} \\
\leq& (d\sum_{i=1}^d \E[e^{2Y_i}]) (\E[I\{F^c\}])^{1/2} \\
\leq& d^2 e^{2\sigma_{\max}^2} \cdot \exp(-\omega(\log^2 d)) = \exp(-\omega(\log^2 d)).
\end{split}
\end{equation*}
Therefore letting $\delta=\frac{1}{2\sqrt{p}}$ in ~(\ref{eq:bound_partition_fct_delta}), we have 
\begin{equation*}
\begin{split}
& \P\Big( \Big|\frac{Z(\Vb_w,\bc) - \E[Z(\Vb_w,\bc)]}{\E[Z(\Vb_w,\bc)]}\Big| \geq \frac{1}{\sqrt{p}} \Big) \\
\leq &\P\big(Z(\Vb_w,\bc)\neq Z(\Vb_w,\bc)I\{F\} \big) + \P\Big(Z(\Vb_w,\bc)= Z(\Vb_w,\bc)I\{F\}, \Big|\frac{Z(\Vb_w,\bc)I\{F\} - \E[Z(\Vb_w,\bc)]}{\E[Z(\Vb_w,\bc)]}\Big| \geq \frac{1}{\sqrt{p}}  \Big)\\
\leq & \P(F^c) + \P\Big(\Big|\frac{Z(\Vb_w,\bc)I\{F\} - \E[Z(\Vb_w,\bc)]}{\E[Z(\Vb_w,\bc)]}\Big| \geq \frac{1}{\sqrt{p}} -  \Big|\frac{\E[Z(\Vb_w,\bc)I\{F\}] - \E[Z(\Vb_w,\bc)]}{\E[Z(\Vb_w,\bc)]}\Big|\Big)\\
\leq & \exp(-\omega(\log^2 d)) + \P\Big(\Big|\frac{Z(\Vb_w,\bc)I\{F\} - \E[Z(\Vb_w,\bc)]}{\E[Z(\Vb_w,\bc)]}\Big| \geq \frac{1}{\sqrt{p}} - \exp(-\omega(\log^2 d)) \Big)\\
\leq& \exp(-\omega(\log^2 d)) + \P\Big(\Big|\frac{Z(\Vb_w,\bc)I\{F\} - \E[Z(\Vb_w,\bc)]}{\E[Z(\Vb_w,\bc)]}\Big| \geq \frac{1}{2\sqrt{p}}  \Big)\\
\leq& \exp(-\omega(\log^2 d)) + 2\exp(-\frac{d^{1-\gamma}}{8p}) = \exp(-\omega(\log^2 d))
\end{split}
\end{equation*}
for appropriately large $(k,d,T)$, under the condition that $k\log^2 d = O(d^{1-\gamma})$ in Assumption~\ref{assump:kdT} so that $d^{1-\gamma}/k = \omega(\log^2 d)$. Note that all these analysis are carried out with $c$ fixed. Thus taking expectation over all $c$, we have our last assetion that
$$
\P_{\bc\sim \mathcal{C},\Vb_w} \Big( (1-\epsilon_z)Z\leq Z(V,\bc) \leq (1+\epsilon_z) Z \Big) \leq \exp(-\omega(\log^2 d)).
$$

\paragraph{Switch the randomness in $\Vb_w$ and $c$.}

We then proceed to prove the assertion with respect to $\Vb_w$ by switching the randomness in $\Vb_w$ and $c$. By applying Lemma~\ref{lem:exchange_prob} to $\Vb_w$, $c$ and $F=\{(1-\epsilon_z)Z\leq Z(\Vb_w,\bc)\leq (1+\epsilon_z) Z)\}$, and some $\epsilon = \exp(-\omega(\log^2 d))$, we have that for $\kappa = \sqrt{\epsilon}$, it holds that
\begin{equation*}
\P\Big(  \P_{\bc\sim \mathcal{C}}\Big(  (1-\epsilon_z)Z\leq Z(\Vb_w,\bc)\leq (1+\epsilon_z) Z) \Big| \Vb_w \Big) >  \kappa \cdot \exp(-\omega(\log^2 d)) \Big) \geq 1-\frac{1}{\kappa},
\end{equation*}
which actuallly reads
\begin{equation*}
\P\Big(  \P_{\bc\sim \mathcal{C}}\Big(  (1-\epsilon_z)Z\leq Z(\Vb_w,\bc)\leq (1+\epsilon_z) Z) \Big| \Vb_w \Big) >  \exp(-\omega(\log^2 d)) \Big) \geq 1- \exp(-\omega(\log^2 d)).
\end{equation*}
\end{proof}

Next we provide bounds on the mean of $Z(\Vb_w,\bc)$, which supports Lemma~\ref{lem:part_fct}.

\begin{lemma}[Distribution and mean of $Z_c$]\label{lem:e_var_Z_c} 

Under Assumption~\ref{assump:parameter_space}, for any fixed unit vector $c \in \RR^p$, and random code vector $\Vb_w$ from our Gaussian graphical model, the vector
$$
\mathbf{Y}=(\langle \Vb_1,\bc\rangle,\cdots,\langle \Vb_d,\bc\rangle)^{\transpose} \sim N(0,\bSigma),
$$
where $\bSigma$ is the covariance matrix in our Gaussian graphical model. Also we have the bounds on the mean of $Z(\Vb_w,\bc)$ that
$$d\leq d \exp(\sigma_{\min}^2/2) \leq \EE[Z(\Vb_w,\bc)] \leq d \exp(\sigma_{\max}^2/2)
$$ 
\end{lemma}

\begin{proof}[Proof of Lemma~\ref{lem:e_var_Z_c}]

Throughout this proof, we let $c$ be any fixed unit vector in $\RR^p$ and code vectors be random variables following our Gaussian graphical model. Recall that partition function $Z_c(\Vb_w, \bc)$ is
$$Z(\Vb_w,\bc) = \sum_w\exp(\langle \Vb_w, \bc \rangle).$$

Based on our Gaussian graphical model, $p$ components in a code vector random variable $\Vb_w$ are i.i.d. Gaussian random variables. Since $c$ is a given unit vector, by the property that linear combination of jointly Gaussian random variables is still a Gaussian r.v., then the vector $(\langle \Vb_1,c\rangle,\cdots,\langle \Vb_d,\bc\rangle)^{\transpose}$ follows a multivariate Gaussian distribution. Clearly the mean is zero, and the pairwise covariance is
$$
\Cov(\langle \Vb_i,\bc\rangle,\langle \Vb_j,\bc\rangle) = \E[\bc^{\transpose} \Vb_i \Vb_j^{\transpose} \bc] = \bc^{\transpose}\E[ \Vb_i \Vb_j^{\transpose}] \bc.
$$
According to Lemma~\ref{lem:distribution_v}, $\Vb_i$, $\Vb_j$ has entries such that $\big[ [\Vb_i]_\ell, [\Vb_j]_\ell \big]^{\transpose}$ are i.i.d. with covariance $\Sigma_{ij}$. Therefore $\E[ \Vb_i \Vb_j^{\transpose}] = \Sigma_{ij} \Ib_p$. Thus we have $\Cov(\langle \Vb_i,\bc\rangle,\langle \Vb_j,\bc\rangle)=\Sigma_{ij}$, where $\Sigma_{ij}$ is the $(i,j)$-th entry of $\bSigma$. Therefore, the vector $(\langle \Vb_1,c\rangle,\cdots,\langle \Vb_d,\bc\rangle)^{\transpose}\sim N(0,\bSigma)$, and $\mathbf{Z}=(\exp(\langle \Vb_1,\bc\rangle),\cdots,\exp(\langle \Vb_d,\bc\rangle) )^{\transpose}\sim \mbox{lognormal}(0,\bSigma)$.

According to the distribution of $\Zb$, the mean of $Z(\Vb_w,\bc)$ is
\begin{equation}
    \EE[Z(\Vb_w,\bc)] = \sum_{i = 1}^d \EE [\Zb_i] = \sum_{i = 1}^d \exp(\Sigma_{ii}/2).
\end{equation}

Applying Lemma~\ref{lem:upper_bound_var} and Assumptions~\ref{assump:parameter_space} and \ref{assump:block_bound} to $\Sigma_{ii}$ the variance of the $i$-th word, $\EE[Z(\Vb_w,\bc)]$ is bounded as 
\begin{equation}
 d\leq  d \exp(\sigma_{\min}^2/2) \leq \EE[Z(\Vb_w,\bc)] \leq d\exp\big(\max_w \sigma_w^2/2 \big) \leq d\exp(\sigma_{\max}^2/2),
\end{equation}
where the left-most bound is straightforward since $\sigma_{\min}^2\geq 0$.

\end{proof}

\begin{lemma} \label{lem:exchange_prob}
Let $F_2$ be an event about random variables $X_1, X_2$, formally, $F_2\in \sigma(X_1,X_2)$ and $I\{F_2\}$ be the indicator function of $F_2$. Assume that conditioning on $X_2$, it satisfies that $\EE[I\{F_2\}|X_2] =\PP(F_2|X_2) \geq 1 - \epsilon$ for some $0< \epsilon < 1$, then
$$
\PP_{X_1}\left[\PP_{X_2}(F_2|X_1) > 1 - \kappa \epsilon\right] \geq 1 - \frac{1}{\kappa},
$$
where $\kappa > 0$, and $\P_{X_1}$ indicates that the probability is with repsect to $X_1$.
\end{lemma}
\begin{proof}[Proof of Lemma~\ref{lem:exchange_prob}.]
We prove the lemma by contradiction. Assume that $\P_{X_1}\left(\PP_{X_2}(F_2|X_1) > 1 - \kappa \epsilon\right) = 1 - \frac{1}{\kappa} - \delta$, where $\delta > 0$. Then $\EE[I\{F_2\}]$ satisfies that
\begin{equation} \label{eqn:contra_1}
\begin{split}
    \EE[I\{F_2\}] 
    =& \EE \big[I\{F_2\}I{\{\EE[I\{F_2\}|X_1] > 1 - \kappa \epsilon\}}\big] + \E\big[I\{F_2\}I{\{\EE[I\{F_2\}|x_1] \leq  1 - \kappa \epsilon\}}\big]\\ 
    \leq& \E\big[I{\{\EE[I\{F_2\}|X_1] > 1 - \kappa \epsilon\}}\big] + \E\big[\E[I\{F_2\}I{\{\EE[I\{F_2\}|X_1] \leq  1 - \kappa \epsilon\}}\big| X_1]\big]\\
    \leq& \PP(\EE[I\{F_2\}|x_1] > 1 - \kappa \epsilon) + \E\big[\E[I\{F_2\}\big| X_1]I{\{\EE[I\{F_2\}|X_1] \leq  1 - \kappa \epsilon\}}\big] \\
    \leq&\PP(\EE[I\{F_2\}|x_1] > 1 - \kappa \epsilon) + (1-\kappa \epsilon)\E\big[I{\{\EE[I\{F_2\}|X_1] \leq  1 - \kappa \epsilon\}}\big]\\
    =& \PP_{X_1}\left(\PP_{X_2}(F_2|X_1) > 1 - \kappa \epsilon\right) + (1 - \kappa \epsilon)\PP_{x_1}\left[\PP_{x_2}(F_2|x_1) \leq 1 - \kappa \epsilon\right]  \\
    =& 1 - \frac{1}{\kappa} - \delta + (1 - \kappa \epsilon) \left(\frac{1}{\kappa} + \delta\right)=1 - \epsilon - \kappa \epsilon \delta < 1 - \epsilon.
\end{split}
\end{equation}
Here the first inequlity is due to monotonicity of expectation and tower property and conditional expectation, and the second inequality is because $I{\{\EE[I\{F_2\}|X_1] \leq  1 - \kappa \epsilon\}}$ is $\sigma(X_1)$-measurable. 

However, by condition that $\EE_{X_1}[I\{F_2\}|X_2] =\PP_{X_1}[F_2|X_2]  \geq 1 - \epsilon$, it follows that
\begin{equation} \label{eqn:contra_2}
    \EE[I\{F_2\}] = \EE_{X_2}[\EE_{X_1}[I\{F_2\}|X_2]] = \EE_{X_2}[\PP_{X_1}[F_1|X_2]] \geq 1 - \epsilon,
\end{equation}
a contradiction to (\ref{eqn:contra_1}). Thus $\PP_{x_1}\left[\PP_{x_2}(F_2|x_1) > 1 - \kappa \epsilon\right] \leq 1 - \frac{1}{\kappa} - \delta$ for any $\delta > 0$. Therefore, it must hold that
$$\PP_{x_1}\left[\PP_{x_2}(F_2|x_1) > 1 - \kappa \epsilon\right] \geq 1 - \frac{1}{\kappa}.$$

\end{proof}


\section{Technical Lemmas on the Markov Processes}\label{sec:lm_conv_hat_PMI}

\subsection{Concentration of Conditional Occurrence Expectations}\label{subsec:S_w,w'_to_stationary}

Note that $p_w(t)$ is function of $\bc_t$, which is also functions of $\bz_t$. Also $p_w(t)\in [0,1]$. We employ a generalized Hoeffding's inequality to provide concentration properties of expectations of word occurrences.

\begin{lemma} \label{lem:Hoeffding_glynn}
Let $S$ be the state space of a Markov process $\{X_t\}_{t \geq 0}$. Assume that for each $x \in S$, there exists a probability measure $\varphi$ on $S$, $\lambda > 0$ and an integer $m \geq 1$ such that $\PP[X_m \in \cdot|X_0 = x] \geq \lambda \varphi(\cdot)$. For a function $f: S \mapsto \RR$, let $Y_i := f(X_i)$. If $\|f\| := \sup \{ |f(x)|:x \in S\} < \infty$, then we can apply a generalized Heoffding's inequality to $S_n:=\sum_{i = 0}^{n-1}Y_i$ as
$$\PP[S_n - \EE S_n] \geq n\varepsilon|X_0 = x] \leq \exp\left(-\frac{\lambda^2(n\varepsilon - 2\|f\|m/\lambda)^2}{2n\|f\|^2m^2} \right),$$
for $n > 2\|f\|m/(\lambda \varepsilon)$.
\end{lemma}
Lemma~\ref{lem:Hoeffding_glynn} here is a generalized Heoffding's inequality proven in \cite{glynn2002hoeffding}. With minor modification, it can be presented as Lemma~\ref{lem:Heoffding_revised}. We will use Heoffding's inequality in Lemma~\ref{lem:Heoffding_revised} to prove the concentration of $\sum_{t = 1}^T f_2 \circ f_1(\bz_t)$.

\begin{lemma}\label{lem:Heoffding_revised}
Let $S$ be the continuous state space of a Markov process $\{X_t\}_{t\geq 1}$. Let $\pi(\cdot)$ be a stationary distribution of $\{X_t\}_{t\geq1}$. Assume that there exists an integer $m\geq 1$ and $\lambda>0$ such that $\|f_{X_{n+1}|X_1}(\cdot|X_1=x_1)-\pi(\cdot)\|_{\var} \leq (1-\lambda)^{\lfloor n/m\rfloor}$. Let $f:S\to \real$ be a function on $S$, and $S_n:=\sum_{i=1}^n f(X_i)$. Let $\alpha=\E_\pi[f(X_i)]=\int_S f(x)d\pi(x)$ be the expectation of $f(X_i)$ when $X_i\sim \pi(\cdot)$. Then, if $\|f\|_S<\infty$,
$$
\P\Big( S_n - n\alpha  \geq n\epsilon |X_1=x_1 \Big) \leq \exp\Big(-\frac{\lambda^2(n\epsilon - \frac{2m\|f\|_S}{\lambda} )^2}{2nm^2\|f\|_S^2 }\Big)
$$
for $n\epsilon \geq 2m\|f\|_S/\lambda$.
\end{lemma}

\begin{proof}[Proof of Lemma~\ref{lem:Heoffding_revised}.]

We prove Lemma~\ref{lem:Heoffding_revised} by proving the condition here is the same as Lemma~\ref{lem:Hoeffding_glynn}.

Lemma~\ref{lem:Hoeffding_glynn} assumes $\PP[X_m \in \cdot|X_0 = x] \geq \lambda \varphi(\cdot)$ for any $x \in S$. In the proof in \cite{glynn2002hoeffding}, this assumption is \emph{only} used to prove
\begin{equation}\label{eqn:result_assump_A1}
    \left|E[f(X_n)|X_0 = x] - \int_S f(x)\pi (dx) \right| \leq \|f\|\cdot (1 - \lambda)^{\lfloor n/m \rfloor}.
\end{equation}

However, Lemma~\ref{lem:Heoffding_revised} does not assume $\PP[X_m \in \cdot|X_0 = x] \geq \lambda \varphi(\cdot)$, but assumes $\|f_{X_{n+1}|x_1}(\cdot |X_1 = x_1) - \pi(\cdot)\|_{\mathrm{var}}  \leq (1 -\lambda)^{\lfloor n/m \rfloor}$. Also, note that the Markov process in Lemma~\ref{lem:Heoffding_revised} starts at $t = 1$.

We here show that condition $\|f_{X_{n+1}|x_1}(\cdot |X_1 = x_1) - \pi(\cdot)\|_{\mathrm{var}}  \leq (1 -\lambda)^{\lfloor n/m \rfloor}$ can also derive Eq.~\ref{eqn:result_assump_A1}, only with different subscripts:
\begin{equation} \label{eqn:assumption_replace}
\begin{split}
    &\|f_{X_{n+1}|x_1}(\cdot |X_1 = x_1) - \pi(\cdot)\|_{\mathrm{var}}  \leq (1 -\lambda)^{\lfloor n/m \rfloor} \\
    \Longrightarrow & |f(\cdot)|\times |f_{X_{n+1}|x_1}(\cdot |X_1 = x_1) - \pi(\cdot)| \leq \|f\| (1 - \lambda)^{\lfloor n/m \rfloor} \\
    \Longrightarrow &\left| \EE[f(X_{n+1})|X_1 = x_1] - \int_S f(x)\pi (dx) \right| \leq \|f\| (1 -\lambda)^{\lfloor n/m \rfloor}.
\end{split}
\end{equation}

With this, Lemma~\ref{lem:Heoffding_revised} can be proven is the same way as Lemma~\ref{lem:Hoeffding_glynn}.
\end{proof}

Recall our definition in Section~\ref{subsec:def_word_occur} of conditional occurrence probabilities $p_w(t)=\E[X_{t}(w)|\{\bc_t\}_{t>0}]$, and the conditional total occurrence $S_w=\sum_{t=1}^T p_w(t)$ of word $w$. Also recall that $N_w = Tp_w = \sum_{t=1}^{T}\E_{c\sim \mathcal{D}}[p_w(t)]$ is the stationary expectation of $S_w$. Based on Lemma~\ref{lem:Heoffding_revised}, we show the relative concentration of $S_w$ to $N_w$ provided that $p_w$ is of order $\Omega(1/d)$, which is decided by $\Vb_w$. Note that the stationary probabilities $p_w$ and hence $N_w$ only depend on word vectors $\Vb_w$. 



\begin{lemma}\label{lem:Sw_to_Nw}
Suppose Assumptions~\ref{assump:parameter_space}, \ref{assump:block_bound} and~\ref{assump:kdT} hold, and let $\overline{p},\underline{p}$ be the constant in Lemma~\ref{lem:pw_pww'_to_cov}. Denote the set of word vectors $\mathcal{V}^*=\{V_w: \underline{p}/d\leq \min_w p_w \leq \max_w p_w\leq \overline{p}/d\}$. Then for $\Vb_w\in \mathcal{V}^*$, it holds that
$$
\P\Big(\Big|\frac{S_w}{N_w} - 1\Big| \geq \frac{1}{\sqrt{p}}\Big| \Vb_w\Big) = \exp(-\omega(\log^2(d))),
$$
If we remove the above condition on $\Vb_w$, then $\P(\Vb_w\notin \mathcal{V}^*) \leq \exp(-\omega(\log^2 d) + d^{-\tau} $ and
$$
\P\Big(\Big|\frac{S_w}{N_w} - 1\Big| \geq \frac{1}{\sqrt{p}}\Big) = \exp(-\omega(\log^2(d)))+d^{-\tau},
$$
for some large constant $\tau>0$ as in Lemma~\ref{lem:pw_pww'_to_cov}.
\label{lemma:cond_sum_pw_stat}
\end{lemma}

\begin{proof}
Adopt the choice of $m=4p^2\log d$ and $\lambda=1-\frac{2}{e}$ in Lemma \ref{lem:zt_mixing}, then conditions in Lemma~\ref{lem:Heoffding_revised} are satisfied. By Lemma \ref{lem:Heoffding_revised} applied to $f(\bz_t)=p_w(t)$, with $\|f\|_S\leq 1$, where $S$ is the continuous state space of $\bz_t$, we have that for $T\epsilon \geq 2m/\lambda$,
\begin{equation}
\P\Big( \big|\sum_{t=1}^T p_w(t) - T p_w\big|\geq T\epsilon \Big| \Vb_w \Big)\leq 2\exp(-\frac{\lambda^2(T\epsilon - \frac{2m}{\lambda})^2 }{2T m^2}).
\label{eq:Sw_to_stat_bound}
\end{equation}
Here we choose $\epsilon = \delta p_w$, where $\delta = \frac{1}{\sqrt{p}}$. Then as $\lambda\geq 1/4$ we have
$$
\frac{T\epsilon}{2m/\lambda} = \frac{\lambda Tp_w}{8p^2\log d} \geq \frac{Tp_w}{32p^2 \log d} \geq 2,
$$
since $p_w=\Omega(1/d)$ for $\Vb_w\in \mathcal{V}^*$. Thus $T\epsilon\geq 2m/\lambda$ and
\begin{equation}
\P\Big( \Big|\frac{1}{Tp_w}\sum_{t=1}^T p_w(t) - 1\Big|\geq \delta \Big| \Vb_w\Big)\leq 2\exp(-\frac{\lambda^2 T^2\delta^2 p_w^2 }{8T m^2})= \exp(-\omega(\frac{T}{p^5 d^2 \log^2 d}))=\exp(-\omega(\log^2(d))).
\label{eq:Sw_to_stat_bound2}
\end{equation}
Furthermore, by the results about the scale of $p_{w,w'}^{(u)}$ in Lemma~\ref{lem:pw_pww'_to_cov}, $\P(\Vb_w\notin \mathcal{V}^*)\leq \exp(-\omega(\log^2 d))+d^{-\tau}$ for the large constant $\tau>0$ as in Lemma~\ref{lem:pw_pww'_to_cov}. Thus, incorporating the randomness in $\Vb_w$,
\begin{equation}
\begin{split}
&\P\Big( \Big|\frac{1}{Tp_w}\sum_{t=1}^T p_w(t) - 1\Big| \geq \frac{1}{\sqrt{p}}\Big) \\
=&\P\Big( \big\{ \Vb_w\in \mathcal{V}^* \big\}\cap \Big\{\Big|\frac{1}{Tp_w}\sum_{t=1}^T p_w(t) - 1\Big| \geq \frac{1}{\sqrt{p}}\Big\}\Big) + \P\Big(\big\{ \Vb_w\in \mathcal{V}^* \big\}^c \cap \Big\{\Big|\frac{1}{Tp_w}\sum_{t=1}^T p_w(t) - 1\Big| \geq \frac{1}{\sqrt{p}}\Big\}\Big)\\
\leq& \E\Big[ I{\{ \Vb_w\in \mathcal{V}^*\}} \P\Big( \Big|\frac{1}{Tp_w}\sum_{t=1}^T p_w(t) - 1\Big|\geq \delta \Big| \Vb_w\Big)  \Big] + \P(\big\{ \Vb_w\in \mathcal{V}^* \big\}^c)\\
\leq& \exp(-\omega(\log^2 d))\E\big[I{\{ \Vb_w\in \mathcal{V}^*\}}\big] +\P(\big\{ \Vb_w\in \mathcal{V}^* \big\}^c) = \exp(-\omega(\log^2 d)) + d^{-\tau}.
\end{split}
\label{eq:sw_to_nw_whole_prob}
\end{equation}
Here the first inequality is due to tower property of conditional expectation (conditional on $\sigma(\Vb_w)$) and the fact that $I{\{ \Vb_w\in \mathcal{V}^*\}}$ is $\sigma(\Vb_w)$-measurable, and the second inequality is due to Eq.(\ref{eq:Sw_to_stat_bound2}) and the fact that $\P(F^c)\leq\exp(-\omega(\log^2 d))+d^{-\tau})$, as proved in Lemma~\ref{lem:pw_pww'_to_cov}.
\end{proof}

We analyze $\sum_{t=1}^{T-1}p_{w,w'}^{(u)}(t)$ in the same way as in previous notes combined with the results for $\sum_{t=1}^{T}p_w$, where $u$ is a constant distance. Recall that $S_{w,w'}^{(u)}=\sum_{t=1}^{T-u}p_{w,w'}(t,t+u)$ is the conditionally expected co-occurrence counts, and $N_{w,w'}^{(u)}$ is its stationary version, which is a function of $\Vb_w$. According to the coupling for joint $(\bz_t,\bz_{t+u})$ and the fact that $p_{w,w'}^{(u)}(t)$ is a function of $(\bz_t,\bz_{t+u})$ taking value in $[0,1]$, we similarly have the following result.

\begin{lemma}\label{lem:Sww'_to_Nww'}
Suppose Assumptions~\ref{assump:parameter_space}, \ref{assump:block_bound} and~\ref{assump:kdT} hold, and let $\overline{p},\underline{p}$ be the constant in Lemma~\ref{lem:pw_pww'_to_cov}. Denote the set of word vectors $\mathcal{V}^{**}=\{V_w: \underline{p}/d^2\leq \min_{w,w'} p^{(u)}_{w,w'} \leq \max_{w,w'} p^{(u)}_{w,w'}\leq \overline{p}/d^2\}$.  Denote the set of word vectors $\mathcal{V}^{**}=\{V_w: \underline{p}/d^2\leq \min_{w,w'} p^{(u)}_{w,w'} \leq \max_{w,w'} p^{(u)}_{w,w'}\leq \overline{p}/d^2\}$ with constants $\underline{p},\overline{p}$ as in Lemma~\ref{lem:pw_pww'_to_cov}. Then when $\Vb_w\in \mathcal{V}^{**}$, it holds that
$$
\P\Big(\Big|\frac{ S_{w,w'}^{(u)}}{N_{w,w'}^{(u)}} - 1\Big| \geq \frac{1}{\sqrt{p}}\Big| \Vb_w \Big) = \exp(-\omega(\log^2(d))).
$$
Furthermore, if we remove the condition that $\Vb_w\in \mathcal{V}^{**}$, then $\P(\Vb_w\notin \mathcal{V}^{**}) \leq \exp(-\omega(\log^2 d) + d^{-\tau} $ and
$$
\P\Big(\Big|\frac{ S_{w,w'}^{(u)}}{N_{w,w'}^{(u)}} - 1\Big| \geq \frac{1}{\sqrt{p}} \Big) = \exp(-\omega(\log^2(d)))+d^{-\tau},
$$
for some large constant $\tau>0$ as in Lemma~\ref{lem:pw_pww'_to_cov}.
\end{lemma}

\begin{proof}[Proof of Lemma~\ref{lem:Sww'_to_Nww'}]
This proof is essentially the same as the proof for Lemma~\ref{lemma:cond_sum_pw_stat}. Note that $f(\bz_t,\bz_{t+u}) = p_{w,w'}(t,t+u)$ is a function of $(\bz_t,\bz_{t+u})$ with $\|f\|_S\leq 1$. Then with the choice of $m=4p^2\log d $ and $\lambda = 1-\frac{2}{e}$ in Lemma~\ref{lem:mixing_property_zt_zt+u}, conditions in Lemma~\ref{lem:Heoffding_revised} are satified. Therefore for $T\epsilon\geq 2m/\lambda$ and all values of $\Vb_w$,
\begin{equation*}
\P\Big( \Big| \sum_{t=1}^{T-u} p_{w,w'}(t,t+u) - (T-u)p_{w,w'}^{(u)}   \Big|\geq \frac{(T-u)p_{w,w'}^{(u)}}{\sqrt{p}}\Big| \Vb_w \Big) \leq 2\exp(-\frac{\lambda^2 ((T-u)\epsilon - \frac{2m}{\lambda})^2}{2(T-u)m^2} ),
\end{equation*}
where $\epsilon = p_{w,w'}^{(u)}/\sqrt{p}$, and as $\lambda\geq 1/4$ we have
$$
\frac{(T-u)\epsilon}{2m/\lambda} = \frac{\lambda(T-u)p_{w,w'}^{(u)}}{8p^2\log d} \geq \frac{(T-u)p_{w,w'}^{(u)}}{32p^2\log d} \geq 2,
$$
since $p_{w,w'}^{(u)}=\Omega(1/d^2)$ when $\Vb_w\in \mathcal{V}^{**}$, and $T=\omega(p^5d^4\log^2 d)$. Therefore $(T-u)\epsilon \geq 2m/\lambda$ and 
\begin{equation*}
\begin{split}
&\P\Big( \Big|\frac{ \sum_{t=1}^{T-u} p_{w,w'}(t,t+u)}{(T-u)p_{w,w'}^{(u)}} -  1  \Big|\geq \frac{1}{\sqrt{p}}\Big| \Vb_w \Big) \\
\leq&  2\exp(-\frac{\lambda^2 (T-u) (p_{w,w'}^{(u)})^2 }{8\cdot 16p^5 \log^d} ) = \exp(-\omega(\frac{T}{p^5 d^4\log^2 d})) = \exp(-\omega(\log^2 d)),
\end{split}
\end{equation*}
since $T=\omega(p^5 d^4 \log^4 d)$ by Assumption~\ref{assump:kdT}, and $p_{w,w'}^{(u)}=\Omega(1/d^2)$. Furthermore, incorporating the randomness in $\Vb_w$, similar to the reasoning in Eq.(\ref{eq:sw_to_nw_whole_prob}) we have 
\begin{align*}
\P\Big( \frac{\Big| \sum_{t=1}^{T-u} p_{w,w'}(t,t+u)}{(T-u)p_{w,w'}^{(u)}} -  1  \Big|\geq \frac{1}{\sqrt{p}} \Big) &\leq \exp(-\omega(\log^2 d)) \P(\Vb_w\in \mathcal{V}^{**}) + \P(\Vb_w\notin \mathcal{V}^{**}) \\
&=  \exp(-\omega(\log^2 d))+d^{-\tau},
\end{align*}
where $(T-u)p_{w,w'}^{(u)} = N_{w,w'}^{(u)}$ and $\tau>0$ is a large constant as in Lemma~\ref{lem:pw_pww'_to_cov}.

\end{proof}

\subsection{Concentration of Empirical Occurrences}\label{subsec:X_w,w'_to_S_w,w'}

The main result in this part is the concentration of empirical occurrences $X_{w}$ and $X_{w,w'}^{(u)}$, $X_{w,w'}^{[q]}$ defined in Section~\ref{subsec:def_word_occur} to their conditional expected counterparts. This is based on the fact that conditional on $\{\bc_t\}_{t>0}$, $X_{w}(t)$ and $X_{w,w'}(t,t+u)$ can be viewed as independent Bernoulli random variables (or at least independent within some carefully filtered subsequence). Then we use a Chernoff bound for sum of independent Bernoulli r.v.s in Lemma~\ref{lem:Chernoff_Ber} to guarantee their concentrations.

\begin{lemma} \label{lem:Chernoff_Ber} [Chernoff bound for sum of Bernoulli r.v.'s.]
Let $X_1, X_2, \cdots, X_n$ be independent $\{0, 1\}$-valued random variables. Let $S_n = \sum_{i = 1}^n X_i$ whose expectation is $E_n = \EE (S_n)$. Then
$$\PP\left[ S_n \geq (1+\delta)E_n\right] \leq \exp\left(\frac{- \delta^2 E_n}{2+ \delta} \right), \mbox{  $\delta > 0$};$$
$$\PP\left[ S_n \leq (1-\delta)E_n\right] \leq \exp\left(\frac{- \delta^2 E_n}{2} \right) , \mbox{  $0<\delta < 1$}. $$
\end{lemma}

\begin{proof}[Proof of Lemma~\ref{lem:Chernoff_Ber}.]

Let $\mu_i := \EE[X_i]$ for $i = 1, 2, \cdots, n$.
For $ t > 0$ and $\delta > 0$, by Chernoff bound we have
\begin{equation}
\begin{split}
    \PP\left[S_n \geq (1+\delta)E_n \right] 
    &\leq \exp{(- t (1+\delta)E_n)} \prod_{i = 1}^n \EE\left[ \exp{(t X_i)} \right] \\ 
    &= \exp{(- t (1+\delta)E_n)} \prod_{i = 1}^n \left[1+ \mu_i(e^t - 1)\right] \\ 
    & \leq \exp{(- t (1+\delta)E_n)} \prod_{i = 1}^n \exp [\mu_i(e^t - 1)].
\end{split}
\end{equation}

Let $ t = \ln(1+\delta) > 0$, then we have 
\begin{equation}
\begin{split}
     \exp{(-t (1+\delta)E_n)} \prod_{i = 1}^n \exp [\mu_i(e^t - 1)] 
     =& \left( \frac{e^\delta}{(1+\delta)^{(1+\delta)}} \right)^{E_n}
     \\ 
     \leq& \exp\left[E_n\left(\delta - (1+\delta)\frac{2\delta}{2+\delta} \right) \right]\\
     =&\exp\left(\frac{- \delta^2 E_n}{2+ \delta}\right).
\end{split}
\end{equation}
So 
$$\PP\left[S_n\geq (1+\delta)E_n\right] \leq \exp\left(\frac{- \delta^2 E_n}{2+ \delta} \right), \mbox{  $\delta > 0$}. $$
And a similar proof shows that 
$$\PP\left[S_n \leq (1-\delta)E_n\right] \leq \left( \frac{e^{-\delta}}{(1-\delta)^{(1-\delta)}} \right)^{E_n} \leq \exp\left(\frac{- \delta^2 E_n}{2} \right) , \mbox{  $0<\delta < 1$}. $$
\end{proof}

In particular, conditional on $\{\bc_t\}_{t\geq 1}$ and $\Vb_w$, $\{X_w(t)\}_{t\geq 1}$ are independent with
$$
X_w(t)|\bc_t,\Vb_w \sim \text{Bernoulli}(p_w(t)).
$$
Recall that $X_w=\sum_{t=1}^TX_w(t)$ is the total occurrence of word $w$ and $S_w = \sum_{t=1}^T p_w(t)$ is its conditional expectation (conditional on discourse variables $\{\bc_t\}_{t>0}$ and naturally $\Vb_w$). We have the following lemma for this setting.

\begin{lemma}
Assume Assumptions~\ref{assump:parameter_space}, \ref{assump:block_bound} and \ref{assump:kdT} hold. Let $\mathcal{V}^*=\{V_w: \underline{p}/d\leq \min_w p_w \leq \max_w p_w\leq \overline{p}/d\}$.  Then for all $\Vb_w\in \mathcal{V}^*$, it holds that
\begin{equation}
\P\Big( \max_{w}\Big| \frac{X_w}{S_w}-1\Big|\geq \frac{1}{\sqrt{p}}  \Big| \Vb_w \Big) = \exp(-\omega(\log^2 d))+ d^{-\tau'},
\label{eq:Xw_to_Sw|Vw}
\end{equation}
for a large constant $\tau'>0$. If we remove the condition that $\Vb_w\in \mathcal{V}^*$, then we have
\begin{equation}
\P\Big( \max_{w}\Big| \frac{X_w}{S_w}-1\Big|\geq \frac{1}{\sqrt{p}}   \Big) = \exp(-\omega(\log^2 d)) + O(d^{-\tau'}),
\label{eq:Xw_to_Sw_whole}
\end{equation}
for the large constant $\tau'>0$.
\label{lem:X_w_to_cond_exp}
\end{lemma}

\begin{proof}[Proof of Lemma~\ref{lem:X_w_to_cond_exp}]
Apply the Chernoff bound in Lemma~\ref{lem:Chernoff_Ber} to $X_w(t)|\{\bc_t,\Vb_w\}$ which are independent $\mbox{Ber}(p_w(t))$ variables conditional on $\{\bc_t\}_{t>0}$, and recall $S_w = \sum_{t=1}^T p_w(t) = \sum_{t=1}^T \E[X_w(t)|\{\bc_t,\Vb_w\}] = \E[\sum_{t=1}^T X_w(t)|\{\bc_t,\Vb_w\}]$ is a function of $\{\bc_t\}_{t>0}$ and $\Vb_w$, we have that
\begin{equation}
\begin{split}
\P& \Big( \sum_{t=1}^T X_w(t) \geq (1+\delta) S_w  \Big|\{\bc_t\}_{t>0} ,\Vb_w\Big) \leq \exp\Big( -\frac{\delta^2 S_w}{2+\delta}  \Big),\quad \delta>0;\\
\P& \Big( \sum_{t=1}^T X_w(t) \leq (1-\delta) S_w  \Big|\{\bc_t\}_{t>0},\Vb_w \Big) \leq \exp\Big( -\frac{\delta^2 S_w}{2}  \Big),\quad 0<\delta<1.
\end{split} 
\label{eq:X_t(w)_chernoff}
\end{equation}
Note that $S_w$ is a function of $\{\bc_t\}_{t>0}$ and $\Vb_w$. Combining the two bounds in Eq.(\ref{eq:X_t(w)_chernoff}) and using union bound, for $\delta\in (0,1)$ we have 
\begin{equation}
\P\Big( \Big| \frac{X_w}{S_w}-1\Big|\geq \delta \big| \{\bc_t\}_{t>0},\Vb_w  \Big) \leq 2\exp\Big(-\frac{\delta^2 S_w}{2+\delta}\Big).
\label{eq:X_w,S_w_chernoff_whole}
\end{equation}

Define the event 
$$
F=\Big\{|\frac{S_w}{N_w}-1|<\frac{1}{\sqrt{p}},\mbox{ and } \underline{p}/d\leq \min_w p_w\leq \max_w p_w\leq \overline{p}/d\Big\},
$$
where $\underline{p},\overline{p}$ are constants as in Lemma~\ref{lem:pw_pww'_to_cov}. Note that 
$$
F = \Big\{|\frac{S_w}{N_w}-1|<\frac{1}{\sqrt{p}}\Big\} \cap \big\{ \Vb_w \in \mathcal{V}^* \big\}.
$$
therefore by Lemma \ref{lemma:cond_sum_pw_stat}, we know for all $\Vb_w\in \mathcal{V}^*$ as in Lemma~\ref{lemma:cond_sum_pw_stat},
$$
\P(F|\Vb_w) =\E\big[ I\{F\}I{\{\Vb_w \in \mathcal{V}^*\}}  \big|\Vb_w\big] =I{\{\Vb_w \in \mathcal{V}^*\}} \E\big[ I\{F\}  \big|\Vb_w\big] = \E\big[ I\{F\}  \big|\Vb_w\big] \geq 1-\exp(-\omega(\log^2(d))).
$$
Also $\Vb_w\in \mathcal{V}^*$ with high probability, specifically, $\P(\Vb_w\notin \mathcal{V}^*)\leq \exp(-\omega(\log^2 d))+d^{-\tau}$ for some (large) constant $\tau>0$.

Also recall $p_w$ is a function of $\Vb_w$, $N_w$ is a function of $\Vb_w$ and $S_w$ is a function of $\{\bc_t\}_{t>0}$ and $\Vb_w$, so $I\{F\}$ is $\sigma(\Vb_w,\{\bc_t\}_{t>0})$-measurable. And on $F$,  we have $N_w = Tp_w \geq \underline{p} T/d$. Therefore,
\begin{equation}
\begin{split}
\P\Big( \Big| \frac{X_w}{S_w}-1\Big|\geq \delta\Big| \Vb_w \Big) \leq&  \P\Big( \Big\{\Big| \frac{X_w}{S_w}-1\Big|\geq \delta\Big\}\cap F\Big| \Vb_w\Big) + \P(F^c|\Vb_w)\\
=& \E\Big[ \E\big[ I{\{|X_w-S_w|\geq \delta S_w\}} I\{F\}  \big| \{\bc_t\}_{t>0} , \Vb_w \big]\Big| \Vb_w \Big] +\P(F^c|\Vb_w)\\
= & \E\Big[ \E\big[ I{\{|X_w-S_w|\geq \delta S_w\}}  \big| \{\bc_t\}_{t>0} ,\Vb_w \big]I\{F\} \Big| \Vb_w\Big] +\P(F^c|\Vb_w)\\
\leq& 2\E\Big[ \exp\big(-\frac{\delta^2 S_w}{2+\delta}\big) I\{F\}\Big| \Vb_w\Big] +\P(F^c|\Vb_w).
\end{split}
\label{eq:proof_Xw_bd1}
\end{equation}
Here the first inequality is just union bound. The two equalities that follow are due to tower property and the fact that $I\{F\}$ is $\sigma(\{\bc_t\}_{t>0},\Vb_w)$-measurable. The second inequality is due to Eq.(\ref{eq:X_w,S_w_chernoff_whole}). 

Furthermore, according to the definition of $F$, on $F$ it holds that $S_w \geq (1-\frac{1}{\sqrt{p}})N_w$, and $N_w\geq T\underline{p}/d$. Continuing the bound in Eq.(\ref{eq:proof_Xw_bd1}) we have
\begin{equation}
\begin{split}
\P\Big( \Big| \frac{X_w}{S_w}-1\Big|\geq \delta\Big| \Vb_w \Big) \leq&2\E\Big[ \exp\big(-\frac{\delta^2}{2+\delta} (1-\frac{1}{\sqrt{p}})^2 N_w^2\big)I\{F\} \Big| \Vb_w\Big] +\P(F^c|\Vb_w)\\
\leq&2\E\Big[ \exp\big(-\frac{\delta^2}{2+\delta} (1-\frac{1}{\sqrt{p}})^2 \frac{T^2 \underline{p}^2}{d^2}\big)I\{F\} \Big| \Vb_w\Big] +\P(F^c|\Vb_w)\\
\leq&2\E\Big[ \exp\big(-\frac{\delta^2}{2+\delta} (1-\frac{1}{\sqrt{p}})^2 \frac{T^2\underline{p}^2}{d^2}\big) \Big| \Vb_w\Big] +\P(F^c|\Vb_w)\\
\leq& 2\exp\Big(-\frac{\delta^2}{2+\delta} (1-\frac{1}{\sqrt{p}})^2 \frac{T^2\underline{p}^2}{d^2}\Big) + \P(F^c|\Vb_w).
\end{split}
\label{eq:proof_Xw_bd2}
\end{equation}
Here the first two inequalities are due to the definition of $F$. The third inequality is due to monotonicity of expectation, and the last is because the quantity inside is a constant thus we remove the conditioning.

Letting $\delta = \frac{1}{\sqrt{p}}$, then for all $\Vb_w\in \mathcal{V}^*$, according to Eq.(\ref{eq:proof_Xw_bd1}) and Assumption~\ref{assump:parameter_space}, we have 
\begin{equation}
\begin{split}
\P\Big( \Big| \frac{X_w}{S_w}-1\Big|\geq \frac{1}{\sqrt{p}}\Big| \Vb_w \Big) \leq& 2\exp\Big(-\frac{1}{2k+\sqrt{p}} (1-\frac{1}{\sqrt{p}})^2 \frac{T^2p_*^2}{d^2}\Big) + \exp(-\omega(\log^2 d)) + d^{-\tau}\\
=& \exp(-\omega(\log^2 d)) + d^{-\tau}.
\end{split}
\label{eq:proof_Xw_bd3}
\end{equation}
Since the quantity does not depend on $\Vb_w$, combined with the fact that $\P(\Vb_w\notin \mathcal{V}^*)\leq \exp(-\omega(\log^2 d))+d^{-\tau}$, according to Eq.(\ref{eq:proof_Xw_bd2}) we have
\begin{equation*}
\begin{split}
 \P_{\Vb_w,\bc_t}\Big( \Big| \frac{X_w}{S_w}-1\Big|\geq \frac{1}{\sqrt{p}}\Big) =& \E\Big[ I{\{\Vb_w\in \mathcal{V}^*\}} \P\Big( \Big| \frac{X_w}{S_w}-1\Big|\geq \frac{1}{\sqrt{p}}\Big| \Vb_w \Big)    \Big] + \P(\Vb_w\notin \mathcal{V}^*)\\
 \leq&   \exp(-\omega^2(\log^2 d)) +  O(d^{-\tau}),
 \end{split}
\end{equation*}
for large constant $\tau>0$. Finally, applying union bounds for all $w$ and note that $d\exp(-\omega(\log^2 d)) +d\cdot d^{-\tau} = \exp(-\omega(\log^2 d)) +d^{-\tau'}$ for another large constant $\tau'>0$, we have the desired results in Eq.(\ref{eq:Xw_to_Sw|Vw}) and (\ref{eq:Xw_to_Sw_whole}).
\end{proof}

Similar to Lemma \ref{lem:X_w_to_cond_exp}, as $S_{w,w'}^{(u)}$ is also centered around $N_{w,w'}^{(u)}$, the empirical co-occurrences also concentrate to their conditional expectations. Recall that $X_{w,w'}^{(u)}=\sum_{t=1}^{T-u}X_{w,w'}(t,t+u)$ is the total co-occurrence of words $w,w'$, and $S_{w,w'}^{(u)} = \sum_{t=1}^T p_{w,w'}^{(u)}(t,t+u)$ is its conditional expectation. We have the following result for concentration of empirical co-occurrences.

\begin{lemma}
Consider a fixed windows size $q$. Assume Assumptions \ref{assump:parameter_space}, \ref{assump:block_bound} and \ref{assump:kdT} hold. Recall the set of word vectors $\mathcal{V}^{**}=\{V_w: \underline{p}/d^2\leq \min_{w,w'} p^{(u)}_{w,w'} \leq \max_{w,w'} p^{(u)}_{w,w'}\leq \overline{p}/d^2\}$ with constants $\underline{p},\overline{p}$ as in Lemma~\ref{lem:pw_pww'_to_cov}. Then for $u=1,\dots,q$, for $\Vb_w\in \mathcal{V}^{**}$, it holds that
\begin{equation}
\begin{split}
\P&\Big( \max_{w,w'}\Big| \frac{X_{w,w'}^{(u)}}{S_{w,w'}^{(u)}}-1\Big|\geq \frac{1}{\sqrt{p}} \Big| \Vb_w \Big) = \exp(-\omega(\log^2 d))+d^{-\tau'},\\
\P&\Big( \max_{w,w'}\Big| \frac{X_{w,w'}^{[q]}}{S_{w,w'}^{[q]}}-1\Big|\geq \frac{1}{\sqrt{p}} \Big| \Vb_w \Big) = \exp(-\omega(\log^2 d))+d^{-\tau'},
\end{split}
\end{equation}
for some large constant $\tau'>0$. If we remove the conditions on $\Vb_w$, then
\begin{equation}
\begin{split}
\P&\Big( \max_{w,w'}\Big| \frac{X_{w,w'}^{(u)}}{S_{w,w'}^{(u)}}-1\Big|\geq \frac{1}{\sqrt{p}}  \Big) = \exp(-\omega(\log^2 d))+O(d^{-\tau'}),\\
\P&\Big( \max_{w,w'}\Big| \frac{X_{w,w'}^{[q]}}{S_{w,w'}^{[q]}}-1\Big|\geq \frac{1}{\sqrt{p}}  \Big) = \exp(-\omega(\log^2 d))+O(d^{-\tau'}).
\end{split}
\label{eq:X_ww_to_S_ww_full}
\end{equation}
\label{lem:X_ww'_to_cond_exp}
\end{lemma}

Note that, different from the occurrence of one single word at one single step, co-occurrence indicators $X_{w,w'}(t,t+u)$ are not independent conditional on $\{\bc_t\}_{t>0},\Vb_w$. To circumvent such difficulty, in our proof, we carefully choose a subsequence of $\{1,\dots,T\}$ so that along the sequence, these Bernoulli random variables are conditionally independent. Apart from this modification, other arguments are essentially the same as the proof of Lemma~\ref{lem:X_w_to_cond_exp}.

\begin{proof}[Proof of Lemma~\ref{lem:X_ww'_to_cond_exp}]
We first analyze for a fixed $u\leq q$. Consider $2u$ subsequences of $\{X_{t,t+u}(w,w')\}_{t=1}^{T-u}$ denoted as 
\begin{equation*}
\mathbf{X}^{(u,i)} = \{X_{2ku + i,(2k+1)u+i}(w,w')\}_{k=0}^{\lfloor \frac{T-u-i}{2u} \rfloor},\quad i=1,\dots,2u.
\end{equation*}
Then $\mathbf{X}^{(u,i)}$ are disjoint, and $\cup_{i}\mathbf{X}^{(u,i)}  = \{X_{t,t+u}(w,w')\}_{t=1}^{T-u}$. Moreover, for any fixed $1\leq i\leq 2u$, the elements within the subsequence $\mathbf{X}^{(u,i)} $ are independent Bernoulli ($p_{w,w'}(t,t+u)$) random variables conditional on $\{\bc_t\}_{t>0}$ and $\Vb_w$, where by definition $p_{w,w'}(t,t+u)=\E[X_{w,w'}(t,t+u)|\{\bc_t\}_{t>0},\Vb_w]$.

Denote total co-occurrence, conditional expectations and stationary versions for these subsequences as
\begin{equation*}
\begin{split}
X^{(u,i)}_{w,w'} =& \sum_{k=0}^{\lfloor \frac{T-u-i}{2u}\rfloor} X_{2ku + i,(2k+1)u+i}(w,w'),\\
S^{(u,i)}_{w,w'} =& \sum_{k=0}^{\lfloor \frac{T-u-i}{2u}\rfloor} \E[X_{2ku + i,(2k+1)u+i}(w,w')|\{\bc_t\}_{t>0},\Vb_w] = \sum_{k=0}^{\lfloor \frac{T-u-i}{2u}\rfloor }p_{w,w'}(2ku + i,(2k+1)u+i),\\
N^{(u,i)}_{w,w'} =&  \sum_{k=0}^{\lfloor \frac{T-u-i}{2u}\rfloor }p_{w,w'}^{(u)} = \lfloor \frac{T-u-i}{2u}\rfloor p_{w,w'}^{(u)}.
\end{split}
\end{equation*}

For each fixed $u,i$, define the event (for simplicity drop the superscript $(u,i)$)
$$
F = \Big\{  \Big| \frac{S_{w,w'}^{(u,i)}}{N_{w,w'}^{(u,i)} } - 1  \Big| <\frac{1}{\sqrt{p}} \Big\} \cap \big\{  \Vb_w \in \mathcal{V}^{**} \big\}.
$$
Since $u,i$ are all bounded by fixed window size $q$, similar to the proof of Lemma~\ref{lem:Sww'_to_Nww'} modified for subsequence $\mathbf{X}^{(u,i)}$, we have for all $\Vb_w\in \mathcal{V}^{**}$ it holds that
$$
\P(F|\Vb_w) = I{\{\Vb_w \in \mathcal{V}^{**} \}} \P\Big( \Big|\frac{S_{w,w'}^{(u,i)}}{N_{w,w'}^{(u,i)} } - 1  \Big| <\frac{1}{\sqrt{p}} \Big| \Vb_w \Big) = \exp(-\omega(\log^2 d)) ,
$$
with $\P(\Vb_w\notin \mathcal{V}^{**})\leq \exp(-\omega(\log^2 d)) + d^{-\tau}$ for a large constant $\tau>0$ as in Lemma~\ref{lem:pw_pww'_to_cov}. Also note that $F\in \sigma(\Vb_w,\{\bc_t\}_{t>0})$. Applying Lemma~\ref{lem:Chernoff_Ber} to conditionally independent Bernoulli random variables inside $\mathbf{X}^{(u,i)}$, we know that for $\delta\in (0,1)$,
\begin{equation}
\begin{split}
\P\Big( \Big| \frac{X^{(u,i)}_{w,w'}}{S^{(u,i)}_{w,w'}}-1\Big|\geq \delta\Big| \Vb_w \Big) \leq&  \P\Big( \Big\{\Big| \frac{X^{(u,i)}_{w,w'}}{S^{(u,i)}_{w,w'}}-1\Big|\geq \delta\Big\}\cap F\Big| \Vb_w\Big) + \P(F^c|\Vb_w)\\
=& \E\Big[ \E\big[ I{\{|X^{(u,i)}_{w,w'}-S^{(u,i)}_{w,w'}|\geq \delta S^{(u,i)}_{w,w'}\}} I\{F\}  \big| \{\bc_t\}_{t>0} , \Vb_w \big]\Big| \Vb_w \Big] +\P(F^c|\Vb_w)\\
= & \E\Big[ \E\big[ I{\{|X^{(u,i)}_{w,w'}-S^{(u,i)}_{w,w'}|\geq \delta S^{(u,i)}_{w,w'}\}}  \big| \{\bc_t\}_{t>0} ,\Vb_w \big]I\{F\} \Big| \Vb_w\Big] +\P(F^c|\Vb_w)\\
\leq& 2\E\Big[ \exp\big(-\frac{\delta^2 S^{(u,i)}_{w,w'}}{2+\delta}\big) I\{F\}\Big| \Vb_w\Big] +\P(F^c|\Vb_w).
\end{split}
\label{eq:proof_Xww'_bd1}
\end{equation}
In Eq.(\ref{eq:proof_Xww'_bd1}), the first inequality is union bound. The second line is due to tower property of conditional expectations. The third line is due to the fact that $F\in \Sigma(\{\bc_t\}_{t>0},\Vb_w)$, and the last line is due to the Chernoff bound in Lemma~\ref{lem:Chernoff_Ber}. Note that on $F$, $S^{(u,i)}_{w,w'} \geq (1-\frac{1}{\sqrt{p}})N^{(u,i)}_{w,w'}$ and $N^{(u,i)}_{w,w'}\geq \lfloor (T-u-i)/(2u)\rfloor \underline{p}/d^2\geq T\underline{p}/(u^2d^2)$ for appropriately large $d,T$. Continuing Eq.(\ref{eq:proof_Xww'_bd1}),
\begin{equation}
\begin{split}
\P\Big( \Big| \frac{X^{(u,i)}_{w,w'}}{S^{(u,i)}_{w,w'}}-1\Big|\geq \delta\Big| \Vb_w \Big) \leq&2\E\Big[ \exp\big(-\frac{\delta^2}{2+\delta} (1-\frac{1}{\sqrt{p}})^2 (N^{(u,i)}_{w,w'})^2\big)I\{F\} \Big| \Vb_w\Big] +\P(F^c|\Vb_w)\\
\leq&2\E\Big[ \exp\big(-\frac{\delta^2}{2+\delta} (1-\frac{1}{\sqrt{p}})^2 \frac{T^2 \underline{p}^2}{5u^2d^2}\big)I\{F\} \Big| \Vb_w\Big] +\P(F^c|\Vb_w)\\
\leq&2\E\Big[ \exp\big(-\frac{\delta^2}{2+\delta} (1-\frac{1}{\sqrt{p}})^2 \frac{T^2\underline{p}^2}{5u^2d^2}\big) \Big| \Vb_w\Big] +\P(F^c|\Vb_w)\\
\leq& 2\exp\Big(-\frac{\delta^2}{2+\delta} (1-\frac{1}{\sqrt{p}})^2 \frac{T^2\underline{p}^2}{5u^2d^2}\Big) + \P(F^c|\Vb_w).
\end{split}
\label{eq:proof_Xww'_bd2}
\end{equation}
Then similar to reasoning in Eq.(\ref{eq:proof_Xw_bd3}) in the proof of Lemma~\ref{lem:X_w_to_cond_exp}, for $\delta = \frac{1}{\sqrt{p}}$ we have 
\begin{equation}
\begin{split}
\P\Big( \Big| \frac{X^{(u,i)}_{w,w'}}{S^{(u,i)}_{w,w'}}-1\Big|\geq \frac{1}{\sqrt{p}}\Big| \Vb_w \Big) \leq& 2\exp\Big(-\frac{1}{2k+\sqrt{p}} (1-\frac{1}{\sqrt{p}})^2 \frac{T^2p_*^2}{5u^2 d^2}\Big) + \exp(-\omega(\log^2 d)) + d^{-\tau}\\
=& \exp(-\omega(\log^2 d)) + d^{-\tau}.
\end{split}
\label{eq:proof_Xww'_bd3}
\end{equation}
Taking union bound of the bound in Eq.(\ref{eq:proof_Xww'_bd3}) for all $u,i$ with $1\leq i\leq 2u$, $1\leq u\leq q$, for all $\Vb_w\in \mathcal{V}^{**}$ we have
\begin{equation}
\P\Big( \Big| \frac{X^{(u,i)}_{w,w'}}{S^{(u,i)}_{w,w'}}-1\Big|< \frac{1}{\sqrt{p}},\forall u,i\Big| \Vb_w \Big) \geq 1- \exp(-\omega(\log^2 d)) - O(d^{-\tau}).
\label{eq:X_ww^ui_bound}
\end{equation}
Further with union bound for all pairs $w,w'$ we have
\begin{equation}
\P\Big( \max_{w,w',u,i}\Big| \frac{X^{(u,i)}_{w,w'}}{S^{(u,i)}_{w,w'}}-1\Big|\geq \frac{1}{\sqrt{p}}\Big| \Vb_w \Big) \leq d^2\exp(-\omega(\log^2 d)) - O(d^{-\tau+2}) = \exp(-\omega(\log^2 d)) + d^{-\tau'}
\label{eq:X_ww^ui_bound2}
\end{equation}
for some large constant $\tau'>0$ since the $\tau>0$ can be sufficiently large.

Also note that $X_{w,w'}^{(u)}=\sum_{i=1}^{2u}X_{w,w'}^{(u,i)}$, $S_{w,w'}^{(u)}=\sum_{i=1}^{2u} S_{w,w'}^{(u,i)}$, according to Eq.(\ref{eq:proof_Xww'_bd3}) with union bound for all $i=1,\dots,2u$ with fixed $u$ and all pairs $(w,w')$, for all $\Vb_w\in \mathcal{V}^{**}$ we have
\begin{equation}
\P\Big( \max_{w,w'}\Big| \frac{X^{(u)}_{w,w'}}{S^{(u)}_{w,w'}}-1\Big|\geq \frac{1}{\sqrt{p}}\Big| \Vb_w \Big) \leq \exp(-\omega(\log^2 d)) + d^{-\tau'}
\label{eq:X_ww^u_bound}
\end{equation}
for some large constant $\tau'$ and appropriately large $d,T$. And similarly, since $X_{w,w}^{[q]}=\sum_{u=1}^q [X_{w,w'}^{(u)}+X_{w',w}^{(u)}$, we have
\begin{equation}
\P\Big( \max_{w,w'}\Big| \frac{X^{[q]}_{w,w'}}{S^{(u)}_{w,w'}}-1\Big|\geq \frac{1}{\sqrt{p}}\Big| \Vb_w \Big) \leq \exp(-\omega(\log^2 d)) +d^{-\tau'}.
\label{eq:X_ww^[q]_bound}
\end{equation}
for some large constant $\tau'>0$ and appropriately large $d,T$. Removing the condition on $\Vb_w$, combined with the fact that $\P(\Vb_w\notin \mathcal{V}^{**}) \leq \exp(-\omega(\log^2 d)) + d^{-\tau}$, with same arguments as in Lemma\ref{lem:X_w_to_cond_exp} we get the desired results in Eq.(\ref{eq:X_ww_to_S_ww_full}).
\end{proof}


\section{Concentration of Stationary PMI}\label{sec:stat_PMI_to_cov}

\subsection{Concentration of Stationary Occurrence Probabilities}\label{subsec:concen_pw_pww'_stat}

We provide a concentration property for (stationary) occurrence probabilities $p_w$ and $p_{w,w'}^{(u)}$. The proof is generally based on Theorem 2.2 in \cite{arora2016latent} but with some modifications for our model. We first cite Lemma A.5 from \cite{arora2016latent} here.

\begin{lemma}[Lemma A.5 in \cite{arora2016latent}]
Let $\bv\in \real^p$ be a fixed vector with norm $\|\bv\|_2\leq \kappa \sqrt{p}$ for some constant $\kappa$. Then for random vector $\bc\sim\mathcal{C}$, we have that 
$$
\log \E[\exp(\langle \bv,\bc\rangle)] = \frac{\|\bv\|_2^2}{2p} + O(1/p).
$$
\label{lem:cite_expect_exp}
\end{lemma}

The next two lemmas provide key results for the concentration of stationary PMI, which is the 

\begin{lemma}
Suppose Assumptions~\ref{assump:parameter_space}, \ref{assump:block_bound} and \ref{assump:kdT} holds. Given that all code vectors are bounded with $\sigma_{\min}\sqrt{p/2}\leq \|\bv_w\|_2\leq \sigma_{\max}\sqrt{2p}$ for all $w$ and the condition in  (\ref{eq:bound_partition_fct}) holds, the stationary occurrence probabilities $p_w = \E_{\bc\sim \mathcal{C}} [ \frac{\exp(\langle \bv_w,\bc\rangle)}{Z(V,\bc)} |V] $ satisfy that 
\begin{equation}
\max_w\Big| \log (p_w) - \big(\frac{\|\bv_w\|_2^2}{2p} - \log Z\big) \Big| \leq \frac{2}{\sqrt{p}},
\end{equation}
for appropriately large $p$. Furthermore, the aforementioned conditions hold with probability at least $1-\exp(-\omega(\log^2 d))$.
\label{lem:pw_stationary}
\end{lemma}

\begin{proof}[Proof of Lemma~\ref{lem:pw_stationary}.]
Given all code vectors $V$ where $\frac{\sigma_{\min}}{\sqrt{2}}\sqrt{p}\leq \|\bv_w\|_2\leq \sigma_{\max}\sqrt{2p}$ for all $w$ and satisfies the condition in (\ref{eq:bound_partition_fct}), by previous results we know $\P(F|V)\geq 1-\exp(-\omega(\log^2 d))$ for
\begin{equation}
F = \Big\{  \Big|\frac{Z(V,\bc)}{Z}-1  \Big|\leq \frac{1}{\sqrt{p}} \Big\},
\end{equation}
where $Z(V,\bc)= \sum_w \exp(\langle \bv_w,\bc\rangle)$ and the probability measure corresponds to generating $\bc\sim \mathcal{C}$ the uniform distribution over unit sphere. Since $Z(V,\bc)\geq \exp(\langle \bv_w,\bc\rangle)$,
\begin{equation*}
p_{w} = \E_{\bc\sim \mathcal{C}} \Big[ \frac{\exp(\langle \bv_w,\bc\rangle)}{Z(V,\bc)}\Big |V\Big]= \E_{\bc\sim \mathcal{C}} \Big[ \frac{\exp(\langle \bv_w,\bc\rangle)}{Z(V,\bc)}I\{F\} \Big |V\Big] + \E_{\bc\sim \mathcal{C}} \Big[ \frac{\exp(\langle \bv_w,\bc\rangle)}{Z(V,\bc)}I\{F^c\} \Big |V\Big],
\end{equation*}
where 
\begin{equation}
\E_{\bc\sim \mathcal{C}} \Big[ \frac{\exp(\langle \bv_w,\bc\rangle)}{Z(V,\bc)}I\{F^c\} \Big |V\Big]\leq \E[I\{F^c\} |V]  = \exp(-\omega(\log^2 d)),
\label{eq:bound_stat_pw_part2}
\end{equation}
and 
\begin{equation}
\frac{\E_{\bc\sim \mathcal{C}}[\exp(\langle \bv_w,\bc\rangle)I\{F\}|V]}{(1+\epsilon_z)Z} \leq \E_{\bc\sim \mathcal{C}} \Big[ \frac{\exp(\langle \bv_w,\bc\rangle)}{Z(V,\bc)}I\{F\} \Big |V\Big] \leq \frac{\E_{\bc\sim \mathcal{C}}[\exp(\langle \bv_w,\bc\rangle)I\{F\}|V]}{(1-\epsilon_z)Z}.
\label{eq:bound_stat_pw_part1}
\end{equation}
Hereafter we omit the footscript indicating $\bc\sim \mathcal{C}$. Note that 
\begin{equation*}
\E[\exp(\langle \bv_w,\bc\rangle)I\{F\}|V]=\E[\exp(\langle \bv_w,\bc\rangle)|V] - \E[\exp(\langle \bv_w,\bc\rangle)I{F^C}|V] \leq \E[\exp(\langle \bv_w,\bc\rangle)|V],
\end{equation*}
where by Cauchy-Schwarz inequality,
\begin{equation}
\E[\exp(\langle \bv_w,\bc\rangle)I{F^C}|V] \leq \big(\E[\exp(\langle 2\bv_w,\bc\rangle)|V] \cdot  \E[I{F^C}|V] \big)^{1/2}.
\end{equation}
Here $\E[I{F^C}|V]=\exp(-\omega^2(\log^2 d))$ and by Lemma \ref{lem:cite_expect_exp}, since $2\|\bv_w\|_2\leq 2\sigma_{\max}\sqrt{2p}$, we have $\E[\exp(\langle 2\bv_w,\bc\rangle)|V] = \exp(O(\|v\|_2^2/(2k)+1/k)) = O(1)$, thus 
\begin{equation}
\E[\exp(\langle \bv_w,\bc\rangle)|V] - \exp(-\omega(\log^2 d))\leq \E[\exp(\langle \bv_w,\bc\rangle)I\{F\}|V] \leq \E[\exp(\langle \bv_w,\bc\rangle)|V].
\label{eq:bound_stat_pw_part1_continue}
\end{equation}
Combining results in Equation (\ref{eq:bound_stat_pw_part2}) and (\ref{eq:bound_stat_pw_part1_continue}), we have that
\begin{align*}
\frac{\E[\exp(\langle \bv_w,\bc\rangle)|V]-\exp(-\omega(\log^2 d))}{(1+\epsilon_z) Z} \leq &
\E_{\bc\sim \mathcal{C}} \Big[ \frac{\exp(\langle \bv_w,\bc\rangle)}{Z(V,\bc)}\Big |V\Big]\\
\leq& \frac{\E[\exp(\langle \bv_w,\bc\rangle)|V]+\exp(-\omega(\log^2 d))}{(1-\epsilon_z) Z}+\exp(-\omega(\log^2 d)),
\end{align*}
where $\E[\exp(\langle \bv_w,\bc\rangle)|V] = \exp(\frac{\|\bv_w\|_2^2}{2p} +O(1/p))$ and $\|\bv_w\|_2^2/2k\geq \sigma_{\min}^2/2=\Omega(1)$. Thus
\begin{equation}
\log (p_w) \leq \frac{\|\bv_w\|_2^2}{2p} +O(1/p) - \log Z + \epsilon_z + \exp(-\omega(\log^2 d)) \leq \frac{\|\bv_w\|_2^2}{2p} - \log Z + \frac{2}{\sqrt{p}},
\end{equation}
and 
\begin{equation}
\log(p_w) \geq \frac{\|\bv_w\|_2^2}{2p} - \log Z - \epsilon_z - \exp(-\omega(\log^2 d)) = \frac{\|\bv_w\|_2^2}{2p} - \log Z - \frac{2}{\sqrt{p}}
\end{equation}
for appropriately large $p$. The last assertion in the lemma follows union bound as well as the fact that $\P(\frac{\sigma}{\sqrt{2}}\sqrt{p}\leq \|\bv_w\|_2\leq 2\sqrt{p},\forall w)=1-\exp(-\omega(\log^2 d))$ and probability bound in Lemma \ref{lem:part_fct}.

\end{proof}

A similar result involving more techniques holds for stationary co-occurrence probabilities $p_{w,w'}$, stated as follows.

\begin{lemma}
Given the code vectors with $\sigma_{\min}\sqrt{p/2}\leq \|\bv_w\|_2\leq \sigma_{\max}\sqrt{2p}$ for all $w$ and condition of (\ref{eq:bound_partition_fct}) holds. And suppose Assumptions~\ref{assump:parameter_space}, \ref{assump:block_bound} and \ref{assump:kdT} hold. Then the stationary co-occurrence probabilities $p_{w,w'} =\E_{(\bc,\bc')\sim \mathcal{D}_1} [ \frac{\exp(\langle \bv_w,\bc\rangle)}{Z(V,\bc)}\frac{\exp(\langle \bv_{w'},\bc'\rangle )}{Z(V,\bc')}|V] $ satisfy that 
\begin{equation}
\Big| \log(p_{w,w'}) - \Big(\frac{\|\bv_w+\bv_{w'}\|_2^2}{2p}- 2\log Z  \Big) \Big|\leq 6\sqrt{\frac{\log d}{p}}
\label{eq:p_ww'_stationary}
\end{equation}
for appropriately large $p$. Furthermore, it holds with probability at least $1-\exp(-\omega(\log^2 d))$.
\label{lem:p_ww'_stationary}
\end{lemma}

\begin{proof}[Proof of Lemma~\ref{lem:p_ww'_stationary}]
Suppose Assumptions~\ref{assump:parameter_space}, \ref{assump:block_bound} and \ref{assump:kdT} hold. Given all code vectors $V$ where $\sigma_{\min}\sqrt{p/2}\leq \|\bv_w\|_2\leq \sigma_{\max}\sqrt{2p}$ for all $w$ and satisfies the condition in (\ref{eq:bound_partition_fct}), by previous results we know $\P(F|V)=1-\exp(-\omega(\log^2 d))$ for constant $C\leq 5$ and
\begin{equation}
F = \Big\{  \Big|\frac{z(V,\bc)}{Z}-1  \Big|<\frac{1}{\sqrt{p}}, \Big|\frac{z(V,\bc')}{Z}-1  \Big|<\frac{1}{\sqrt{p}}, \|\bc-\bc'\|_2\leq \frac{C\sqrt{\log d}}{p}\Big\},
\end{equation}
where $Z(V,\bc)= \sum_w \exp(\langle \bv_w,\bc\rangle)$ and the probability measure corresponds to generating $(\bc,\bc')\sim \mathcal{D}_1$. By definition, given $\bv_w$ and $\bv_{w'}$,
$$
p_{w,w'} = \E_{(\bc,\bc')\sim \mathcal{D}_1} \Big[ \frac{\exp(\langle \bv_w,\bc\rangle)}{Z(V,\bc)}\frac{\exp(\langle \bv_{w'},\bc'\rangle )}{Z(V,\bc')}\Big| V \Big].
$$
Hereafter we omit the subscript $(\bc,\bc')\sim \mathcal{D}_1$ for simplicity. Also note that 
\begin{equation}
p_{w,w'} = \E \Big[ \frac{\exp(\langle \bv_w,\bc\rangle)}{Z(V,\bc)}\frac{\exp(\langle \bv_{w'},\bc'\rangle )}{Z(V,\bc')} I\{F\} \big|V \Big] + \E \Big[ \frac{\exp(\langle \bv_w,\bc\rangle)}{Z(V,\bc)}\frac{\exp(\langle \bv_{w'},\bc'\rangle )}{Z(V,\bc')} I\{F^c\} \big|V \Big],
\label{eq:stat_pww'_twoparts}
\end{equation}
where since both two ratios are less than $1$, we have 
\begin{equation}
 \E \Big[ \frac{\exp(\langle \bv_w,\bc\rangle)}{Z(V,\bc)}\frac{\exp(\langle \bv_{w'},\bc'\rangle )}{Z(V,\bc')} I\{F^c\} \big|V \Big] \leq \P(F^c|V) = \exp(-\omega(\log^2 d)).
\label{eq:bound_pww'_part1}
\end{equation}
When $F$ happens we have $(1-\epsilon_z)Z\leq Z(V,\bc),Z(V,\bc')\leq (1+\epsilon_z)Z$, so the first term can be bounded as
\begin{align}
\nonumber\frac{\E[\exp(\langle \bv_w,\bc\rangle)\exp(\langle \bv_{w'},\bc'\rangle)I\{F\}|V]}{(1+\epsilon_z)^2Z^2}&\leq \E \Big[ \frac{\exp(\langle \bv_w,\bc\rangle)}{Z(V,\bc)}\frac{\exp(\langle \bv_{w'},\bc'\rangle )}{Z(V,\bc')} I\{F\} \big|V \Big]\\
& \leq \frac{\E[\exp(\langle \bv_w,\bc\rangle)\exp(\langle \bv_{w'},\bc'\rangle)I\{F\}|V]}{(1-\epsilon_z)^2Z^2}.
\label{eq:stat_pww'_part2_1}
\end{align}
We now focus on the term $\E[\exp(\langle \bv_w,\bc\rangle)\exp(\langle \bv_{w'},\bc'\rangle)I\{F\}|V]$. Note that when $F$ happens, we have $\|\bv_{w'}\|_2\leq \sigma_{\max}\sqrt{2p}$ and $\|\bc-\bc'\|_2\leq \frac{5\sqrt{\log d}}{p}$, hence
\begin{align}
&\E[\exp(\langle \bv_w,\bc\rangle)\exp(\langle \bv_{w'},\bc'\rangle)I\{F\}|V]\notag \\
=& \E[\exp(\langle \bv_w+\bv_{w'},\bc\rangle) \exp(\langle \bv_{w'},\bc'-\bc\rangle) I\{F\}|V]\notag \\
\leq & \E[\exp(\langle \bv_w+\bv_{w'},\bc\rangle) \exp(\|_2\bv_{w'}\|\cdot \|\bc'-\bc\|_2) I\{F\}|V]\notag \\
\leq & \exp(5\sqrt{\frac{\log d}{p}}) \E[\exp(\langle \bv_w+\bv_{w'},\bc\rangle) I\{F\}|V].
\label{eq:bound_pww'_part2_2pos}
\end{align}
And
\begin{align}
&\E[\exp(\langle \bv_w,\bc\rangle)\exp(\langle \bv_{w'},\bc'\rangle)I\{F\}|V]\notag \\
=& \E[\exp(\langle \bv_w+\bv_{w'},\bc\rangle) \exp(\langle \bv_{w'},\bc'-\bc\rangle) I\{F\}|V]\notag \\
\geq & \E[\exp(\langle \bv_w+\bv_{w'},\bc\rangle) \exp(-\|_2\bv_{w'}\|\cdot \|\bc'-\bc\|_2) I\{F\}|V]\notag \\
\geq & \exp(-5\sqrt{\frac{\log d}{p}}) \E[\exp(\langle \bv_w+\bv_{w'},\bc\rangle) I\{F\}|V].
\label{eq:bound_pww'_part2_2neg}
\end{align}
Also, the second term in Equation (\ref{eq:bound_pww'_part2_2pos}) and (\ref{eq:bound_pww'_part2_2neg}) satisfies that 
\begin{equation*}
\E[\exp(\langle \bv_w+\bv_{w'},\bc\rangle) I\{F\}|V]= \E[\exp(\langle \bv_w+\bv_{w'},\bc\rangle)|V] - \E[\exp(\langle \bv_w+\bv_{w'},\bc\rangle) I\{F^c\}|V].
\end{equation*}
By the fact that $\|\bv_w+\bv_{w'}\|_2\leq 2\sigma_{\max}\sqrt{2p}$ combined with Lemma \ref{lem:cite_expect_exp}, we have
\begin{equation*}
\E[\exp(\langle \bv_w+\bv_{w'},\bc\rangle)|V] = \exp(\frac{\|\bv_w+\bv_{w'}\|_2^2}{2p} + O(1/p))
\end{equation*}
and $\exp(\frac{\|\bv_w+\bv_{w'}\|_2^2}{2p} + O(1/p)) \leq \exp(4\sigma_{\max}^2+1)$ for appropriately large $p$ such that the term $O(1/p)\leq 1$. Hence
\begin{align*}
\E[\exp(\langle \bv_w+\bv_{w'},\bc\rangle) I\{F^c\}|V] \leq& \big( \E[\exp(\langle \bv_w+\bv_{w'},\bc\rangle)|V]\cdot \P(F^c|V)  \big)^{1/2}\\
\leq& \big(\exp(\frac{\|\bv_w+\bv_{w'}\|_2^2}{2p} + O(1/p)) \big)^{1/2} \cdot \exp(-\omega(\log^2 d)) = \exp(-\omega(\log^2 d)).
\end{align*}
Therefore we have 
\begin{equation}
\exp(\frac{\|\bv_w+\bv_{w'}\|_2^2}{2p} + O(1/p)) - \exp(-\omega(\log^2 d)) \leq \E[\exp(\langle \bv_w+\bv_{w'},\bc\rangle) I\{F\}|V] \leq \exp(\frac{\|\bv_w+\bv_{w'}\|_2^2}{2p} + O(1/p)).
\label{eq:bound_stat_pww'_part2}
\end{equation}
Combining results in Equations (\ref{eq:stat_pww'_twoparts})-(\ref{eq:bound_stat_pww'_part2}) we have
\begin{align*}
\log( p_{w,w'}) \leq& \log(\E \Big[ \frac{\exp(\langle \bv_w,\bc\rangle)}{Z(V,\bc)}\frac{\exp(\langle \bv_{w'},\bc'\rangle )}{Z(V,\bc')} I\{F\} \big|V \Big] +\exp(-\omega(\log^2 d)))\\
\leq& \log(\E[\exp(\langle \bv_w+\bv_{w'},\bc\rangle) I{F}|V])- 2\log Z + 2\epsilon_z +5\sqrt{\frac{\log d}{p}} +\exp(-\omega(\log^2 d))\\
\leq& \frac{\|\bv_w+\bv_{w'}\|_2^2}{2p}- 2\log Z + O(1/p)+ \frac{2}{\sqrt{p}} +5\sqrt{\frac{\log d}{p}} +\exp(-\omega(\log^2 d)) \\
\leq& \frac{\|\bv_w+\bv_{w'}\|_2^2}{2p}- 2\log Z  + 6\sqrt{\frac{\log d}{p}},
\end{align*}
and
\begin{align*}
\log( p_{w,w'})\geq& \log(\E \Big[ \frac{\exp(\langle \bv_w,\bc\rangle)}{Z(V,\bc)}\frac{\exp(\langle \bv_{w'},\bc'\rangle )}{Z(V,\bc')} I\{F\} \big|V \Big])\\
\geq& \log(\E[\exp(\langle \bv_w,\bc\rangle)\exp(\langle \bv_{w'},\bc'\rangle)I\{F\}|V])- 2\log Z - 2\epsilon_z \\
\geq& \log(\E[\exp(\langle \bv_w+\bv_{w'},\bc\rangle) I\{F\}|V])- 2\log Z - 2\epsilon_z - 5\sqrt{\frac{\log d}{p}}\\
\geq& \log( \exp(\frac{\|\bv_w+\bv_{w'}\|_2^2}{2p} + O(1/p)) - \exp(-\omega(\log^2 d)) )- 2\log Z - 2\epsilon_z - \frac{5\log d}{\sqrt{p}}\\
\geq& \frac{\|\bv_w+\bv_{w'}\|_2^2}{2p}- 2\log Z - \exp(-\omega(\log^2 d))  - 2\epsilon_z -5\sqrt{\frac{\log d}{p}}\\
\geq& \frac{\|\bv_w+\bv_{w'}\|_2^2}{2p}- 2\log Z- 6\sqrt{\frac{\log d}{p}}
\end{align*}
for appropriately large $p$ (and also $d$). Therefore we have the bound
\begin{equation*}
\Big| \log(p_{w,w'}) - \Big(\frac{\|\bv_w+\bv_{w'}\|_2^2}{2p}- 2\log Z  \Big) \Big|\leq 6\sqrt{\frac{\log d}{p}}.
\end{equation*}
And the last assertion of the lemma follows union bound applied to the fact that $\frac{\sigma}{\sqrt{2}}\sqrt{p}\leq \|\bv_w\|_2\leq 2\sqrt{p}$ for all $w$ happens with probability at least $1-\exp(-\omega(\log^2 d))$ and condition in Equation (\ref{eq:bound_partition_fct}) is satisfied with probability at least $1-\exp(-\omega(\log^2 d))$.

\end{proof}

A more general result for all $u$ is as follows:
\begin{lemma}
Suppose Assumptions~\ref{assump:parameter_space}, \ref{assump:block_bound} and \ref{assump:kdT} hold. Given the code vectors with $\sigma_{\min}\sqrt{p/2}\leq \|\bv_w\|_2\leq \sigma_{\max}\sqrt{2p}$ for all $w$ and condition of  (\ref{eq:bound_partition_fct}) holds. Then the stationary co-occurrence probabilities $p_{w,w'}^{(u)} =\E_{(\bc,\bc')\sim \mathcal{D}_u} [ \frac{\exp(\langle \bv_w,\bc\rangle)}{Z(V,\bc)}\frac{\exp(\langle \bv_{w'},\bc'\rangle )}{Z(V,\bc')}|V] $ satisfy that 
\begin{equation}
\max_{w,w'}\Big| \log(p_{w,w'}^{(u)}) - \Big(\frac{\|\bv_w+\bv_{w'}\|_2^2}{2p}- 2\log Z  \Big) \Big|\leq 7\sqrt{2u}\cdot \sqrt{\frac{\log d}{p}}
\label{eq:p_ww'^u_stationary}
\end{equation}
for appropriately large $p$. Furthermore, the aforementioned conditions hold with probability at least $1-\exp(\omega(\log ^2 d))$.
\label{lem:p_ww'^u_stationary}
\end{lemma}
\begin{proof}
The proof for this lemma is exactly the same as that for Lemma \ref{lem:p_ww'_stationary}, except that the bound $\|\bc-\bc'\|_2\leq 5\sqrt{\frac{\log d}{p}}$ is replaced by $\|\bc-\bc'\|_2\leq \frac{20\sqrt{2u\log d}}{3p}$. Then  (\ref{eq:p_ww'^u_stationary}) holds for appropriately large $p$, and the whole event has the same probability bound.

\end{proof}

\subsection{Concentration of Stationary PMI to Covariance Matrix}\label{subsec:stat_PMI_to_cov}
The main result in this subsection is the convergence of stationary PMI to the covariance matrix, which is essentially because of the convergence of $\langle \bv_w,\bv_{w'}\rangle /p$ to covariance matrix $\bSigma$, stated in Lemma~\ref{lem:inner_prod_to_cov}, combined with results in Lemma~\ref{lem:pw_stationary},~\ref{lem:p_ww'_stationary} and~\ref{lem:p_ww'^u_stationary}.

The following result shows that logarithm of occurrence probabilities concentrates to covariance (minus a constant), thus are within some particular scale with high probability. This result is a combination of the convergence of logarithm of stationary probabilities to inner products of code vectors in Lemma~\ref{lem:pw_stationary},~\ref{lem:p_ww'_stationary} and~\ref{lem:p_ww'^u_stationary}, as well as the convergence of inner products to true covariances later in Lemmas~\ref{lem:inner_prod_to_cov}.

\begin{proof}[Proof of Lemma~\ref{lem:pw_pww'_to_cov}]
By Lemma \ref{lem:inner_prod_to_cov} and Equation (\ref{eq:inner_prod_to_cov}), using union bounds for all pairs $w,w'$, which at most consist of $d^2$ pairs, we know that with probability at least $1-d^{-\tau}$ which corresponds to the process of generating code vectors, we have that for constant $C_\tau = 12\sqrt{3(\tau+4)}$,
\begin{equation*}
\max_{w}\Big|\frac{\|\bv_w\|_2^2}{p} - \Sigma_{ww}\Big|\leq C_\tau\|\bSigma\|_{\max}\sqrt{\frac{\log d}{p}},\quad \max_{w,w'}\Big| \frac{\langle \bv_w,\bv_{w'}\rangle}{p}-\Sigma_{w,w'} \Big| \leq C_\tau\|\bSigma\|_{\max}\sqrt{\frac{\log d}{p}}.
\end{equation*}
Thus we have
\begin{equation*}
\begin{split}
&\max_{w,w'}\Big|\frac{\|\bv_w+\bv_{w'}\|_2^2}{2p} - \frac{\Sigma_{ww}+\Sigma_{w'w'}+2\Sigma_{ww'}}{2}\Big|\\
\leq&\frac{1}{2} \max_{w,w'}\Big|\frac{\|\bv_w\|_2^2}{p} - \Sigma_{ww}\Big| +\frac{1}{2} \max_{w,w'}\Big|\frac{\|\bv_{w'}\|_2^2}{p} - \Sigma_{w'w'}\Big| + \max_{w,w'}\Big| \frac{\langle \bv_w,\bv_{w'}\rangle}{p}-\Sigma_{w,w'} \Big| \leq 2C_\tau\|\bSigma\|_{\max}\sqrt{\frac{\log d}{p}}.
\end{split}
\end{equation*}
Moreover, from results in Lemma \ref{lem:pw_stationary}, \ref{lem:p_ww'_stationary} and \ref{lem:p_ww'^u_stationary}, with probability (incorporating the randomness in the underlying code discourse variables) at least $1-\exp(-\omega(\log^2 d))$,
\begin{equation*}
\max_{w,}\Big| \log (p_w) - \big(\frac{\|\bv_w\|_2^2}{2p} - \log Z\big) \Big| \leq \frac{2}{\sqrt{p}},
\end{equation*}
\begin{equation*}
\max_{w,w'}\Big| \log(p_{w,w'}) - \Big(\frac{\|\bv_w+\bv_{w'}\|_2^2}{2p}- 2\log Z  \Big) \Big|\leq 6\sqrt{\frac{\log d}{p}},
\end{equation*}
\begin{equation*}
\max_{w,w'}\Big| \log(p_{w,w'}^{(u)}) - \Big(\frac{\|\bv_w+\bv_{w'}\|_2^2}{2p}- 2\log Z  \Big) \Big|\leq 7\sqrt{2u}\cdot \sqrt{\frac{\log d}{p}}.
\end{equation*}
Combining the above results, we know that with probability at least $1-\exp(-\omega(\log^2 d))- d^{-\tau} $ where universal constant $\tilde{C}$ only depends on the $\rho$ in Assumption~\ref{assump:parameter_space},for constant $C_\tau = 12\sqrt{3(\tau+4)}$,
\begin{equation*}
\max_{w}\Big| \log(p_w) - \big( \frac{\Sigma_{ww}}{2} - \log Z \big)\Big| \leq \frac{2+C_\tau\|\bSigma\|_{\max}\sqrt{\log d}}{\sqrt{p}},
\end{equation*}
\begin{equation*}
\max_{w,w'}\Big| \log(p_{w,w'}) - \big( \frac{\Sigma_{ww}+\Sigma_{w'w'}+2\Sigma_{ww'}}{2} - 2\log Z \big)\Big| \leq (6+2C_\tau\|\bSigma\|_{\max})\sqrt{\frac{\log d}{p}},
\end{equation*}
\begin{equation*}
\max_{w,w'}\Big| \log(p_{w,w'}^{(u)}) - \big( \frac{\Sigma_{ww}+\Sigma_{w'w'}+2\Sigma_{ww'}}{2} - 2\log Z \big)\Big| \leq (7\sqrt{2u}+2C_\tau\|\bSigma\|_{\max})\sqrt{\frac{\log d}{p}}.
\end{equation*}
The last assertions follows the fact that all entries $\Sigma_{ww'}$ are bounded below and above, and $d\leq Z\leq \tilde{c}d$ for some universal constant $\tilde{c}$ under our model assumptions.
\end{proof}

Below is a standard result for the convergence of sample covariance to true covariance, which we include and prove here for completeness and for accurate specification of the constants in tail bounds. It serves as the intermediate step for the convergence of stationary PMI to covariance matrix.

\begin{lemma}[Concentration of $\langle \bv_w,\bv_{w'}\rangle/p$ to covariance]\label{lem:inner_prod_to_cov}
For fixed pair of words $w,w'$, denote the covariances in our Gaussian graphical model as $\Sigma_{ww},\Sigma_{w'w'}$ and $\Sigma_{ww'}$. Then for fixed $a>0$, with probability no less than $1-\exp(-\omega(\log ^2 d))$, we have 
\begin{equation}
\P\Big(\Big|\frac{\langle \bv_w,\bv_{w'}\rangle}{p} - \Sigma_{ww'}\Big| \geq a\sqrt{\frac{\log d}{p}}\Big) \leq \exp(-\frac{a^2\log d}{4\cdot 432\Sigma_{ww}\Sigma_{w'w'}}).
\label{eq:inner_prod_to_cov_tail_bound}
\end{equation}
Taking $a=12\sqrt{3(\tau+2)}\|\bSigma\|_{\max}$ for any constant $\tau>0$ yields that, with probability at least $1- d^{-\tau}$, we have
\begin{equation}
\max_{w,w'}\Big|\frac{\langle \bv_w,\bv_{w'}\rangle}{p} - \Sigma_{ww'}\Big| \leq 12\|\bSigma\|_{\max}\sqrt{\frac{3(\tau+2)\log d}{p}}.
\label{eq:inner_prod_to_cov}
\end{equation}
\end{lemma}

\begin{proof}[Proof of Lemma~\ref{lem:inner_prod_to_cov}]
In our setting, components in the code vector $\bv_w$'s are i.i.d. copies of Gaussian random variables with variance $\Sigma_{ww}$ and covariance $\Sigma_{w,w'}$. Applying Lemma \ref{lem:tail_bound_sample_cov} and note that $(1+|\rho|)^2\leq 4$ yields the desired tail probability bound. Moreover, since $\Sigma_{ww}$ are bounded in both directions, taking union bound over $d^2$ pairs of $(w,w')$ and note that $d^2\exp(-(\tau+2)\log d))=d^{-\tau}$, we have the uniform control of covariance deviation.
\end{proof}

Below we prove a result for the convergence of sample covariance of two Gaussian samples to the true covariance and provide precise constants for tails, which are of use in Lemma~\ref{lem:inner_prod_to_cov}. 

\begin{lemma}[Concentration of sample covariance]\label{lem:tail_bound_sample_cov}
Let $(X_i,Y_i)_{i=1}^p$ be i.i.d. samples of bivariate Gaussian random vector $(X,Y)$, with $\E[X]=\E[Y]=0$, $\Var(X)=\sigma_x^2$, $\Var(Y)=\sigma_y^2$ and $\Cov(X,Y)=\Sigma_{xy}=\rho \sigma_x \sigma_y$. Then we have the tail bound for sample covariance $\hat \Sigma_{xy}:=1/p\sum_{i=1}^p X_iY_i$ that
\begin{equation}
\P\Big( \Big| \hat \Sigma_{xy}-\Sigma_{xy}  \Big| >\delta \Big)\leq 4\exp(-\frac{p\delta^2}{432(1+|\rho|)^2\sigma_x^2\sigma_y^2}),
\label{eq:tail_bound_sample_cov]}
\end{equation}
for any $\delta\in (0,12\sigma_x\sigma_y)$.
\end{lemma}

\begin{proof}[Proof of Lemma~\ref{lem:tail_bound_sample_cov}]
Let $X_i^* = X_i/\sigma_x$ and $Y_i^*=Y_i/\sigma_y$, and let $U_i = X_i^* + Y_i^*$, $V_i = X_i^*-Y_i^*$. Then $U_i\sim N(0,2(1+\rho))$, $V_i \sim N(0,2(1-\rho))$. Note that $\sum_{i=1}^p X_i^*Y_i^* = \frac{1}{4}\sum_{i=1}^p (U_i^2 - V_i^2)$. By union bound and triangle inequality,
\begin{align*}
\P\Big( \Big| \hat \Sigma_{xy}-\Sigma_{xy}  \Big| >\delta \Big)=& \P\Big( \Big| \sum_{i=1}^p (X_iY_i-\Sigma_{xy} ) \Big| >k\delta \Big)\\
=& \P\Big( \Big| \sum_{i=1}^p (X_i^*Y_i^*-\rho ) \Big| >\frac{p\delta}{\sigma_x\sigma_y} \Big)\\
=&  \P\Big( \Big| \sum_{i=1}^p (U_i^2 - 2(1+\rho)) - \sum_{i=1}^p (V_i^2-2(1-\rho)) \Big| >\frac{4k\delta}{\sigma_x\sigma_y} \Big)\\
\leq& \P\Big( \Big| \sum_{i=1}^p (U_i^2 - 2(1+\rho))\Big| >\frac{2k\delta}{\sigma_x\sigma_y} \Big) +\P\Big( \Big| \sum_{i=1}^p (V_i^2 - 2(1-\rho))\Big| >\frac{2k\delta}{\sigma_x\sigma_y} \Big).
\end{align*}
Let $Z_i = U_i^2-2(1+\rho)$. By Jensen inequality we have $|x+y|^m \leq 2^{m-1}(|x|^m + |y|^m)$, hence
\begin{equation*}
\big|\E[(Z_i^m)]\big| = \big|\E[(U_i^2 - 2(1+\rho))^m]\big|\leq 2^{m-1}(\E[U_i^{2m}] + 2^m(1+\rho)^m),
\end{equation*}
where by Gaussianity, $\E[U_i^{2m}] = 2^m(1+\rho)^m (2m-1)!!$. Thus
\begin{equation*}
\big|\E[(Z_i^m)]\big|\leq 2^{m-1}[2(1+\rho)]^m (1+(2m-1)!!).
\end{equation*}
Note that for $m\geq 3$, we have $1+(2m-1)!!\leq \frac{(2m)!!}{3}$, thus for $m\geq 3$,
\begin{equation*}
\big|\E[(Z_i^m)]\big|\leq \frac{(2m)!!}{3} 2^{2m-1}(1+\rho)^m = \frac{2^{3m-1}}{3}(1+\rho)^m \cdot m!.
\end{equation*}
We apply Bernstein's inequality for $X=Z_i$ with $\sigma^2 = \E[Z_i^2]=8(1+\rho)^2$, and $b$ such that
\begin{equation*}
\frac{m!}{2}\sigma^2 b^{m-2}=4m!(1+\rho)^2b^{m-2}\geq \frac{2^{3m-1}(1+\rho)^m}{3}m!,
\end{equation*}
where we may choose $b=24(1+\rho)$. Therefore by Bernstein's inequality, with $\epsilon=\delta/(\sigma_x\sigma_y)$,
\begin{equation*}
\P\Big(\sum_{i=1}^p Z_i \geq 2k\epsilon\Big) \leq e^{-2\lambda k\epsilon} \E[e^{\lambda\sum_{i=1}^p Z_i}] \leq \exp(-\frac{p}{2}\big(4\lambda\epsilon - \frac{\lambda^2 b^2}{1-b\lambda}\big)),\quad \forall 0 \leq \lambda <\frac{1}{b}.
\end{equation*}
Now let $b\lambda = 1-\frac{1}{\sqrt{1+4\epsilon/b}}\in [0,1)$, the tail bound becomes
\begin{equation*}
\P\Big(\sum_{i=1}^p Z_i \geq 2k\epsilon \Big) \leq 
\exp(-\frac{2k\epsilon}{b}\frac{\sqrt{1+4\epsilon/b}-1}{1+\sqrt{1+4\epsilon/b}})= \exp(-\frac{8k\epsilon^2}{b^2(1+\sqrt{1+4\epsilon/b})^2}).
\end{equation*}
Therefore for $4\epsilon/b\leq 1$, i.e., $\delta \leq 12\sigma_x \sigma_y$, we have
\begin{equation*}
\P\Big(\sum_{i=1}^p Z_i \geq 2k\epsilon \Big) \leq 
\exp(-\frac{8k\epsilon^2}{6b^2}) = \exp(-\frac{p\epsilon^2}{432(1+\rho)^2}) \leq \exp(-\frac{p\epsilon^2}{432(1+|\rho|)^2})
\end{equation*}
Similar result applied to $-Z_i$ yields
\begin{equation}
\P\Big(\Big|\sum_{i=1}^p Z_i\Big| \geq 2k\epsilon \Big)\leq 2\exp(-\frac{p\epsilon^2}{432(1+|\rho|)^2}).
\label{eq:bernstein_zi}
\end{equation}
Moreover, let $W_i=V_i^2-2(1-\rho)$, exactly the same result applied to $-\rho$ yields
\begin{equation}
\P\Big(\Big|\sum_{i=1}^p W_i\Big| \geq 2k\epsilon \Big)\leq 2\exp(-\frac{p\epsilon^2}{432(1+|\rho|)^2})
\label{eq:bernstein_wi}
\end{equation}
Combining Equations (\ref{eq:bernstein_zi}) and (\ref{eq:bernstein_wi}) yields
\begin{equation*}
\P\Big( \Big| \hat \Sigma_{xy}-\Sigma_{xy}  \Big| >\delta \Big)\leq 4\exp(-\frac{p\epsilon^2}{432(1+|\rho|)^2})=4\exp(-\frac{p\delta^2}{432(1+|\rho|)^2\sigma_x^2\sigma_y^2}).
\end{equation*}
\end{proof}

\end{document}